\documentclass{article}
\usepackage[utf8]{inputenc}
\usepackage{float} 
\usepackage{amsmath}
\usepackage{amssymb}
\usepackage{booktabs}
\usepackage{array}
\usepackage{bbm}
\usepackage{makeidx}
\usepackage{hyperref}
\usepackage{subcaption}
\usepackage{amsfonts}
\usepackage{bbm}
\usepackage{dsfont}
\usepackage{graphicx}
\usepackage{lipsum}

\usepackage{xspace}
\usepackage{amsthm}
\usepackage{color}
\usepackage{mathtools}

\usepackage{mathrsfs}
\usepackage{scalerel}
\usepackage{forest}
\usepackage{algorithm,algpseudocode}
\usepackage[inline]{}
\usepackage{xcolor}
\usepackage[title]{appendix}
\hypersetup{
    colorlinks,
    linkcolor={blue!80!black},
    citecolor={green!50!black},
}
\usepackage{accents}
\usepackage{apptools}
\usepackage{tikz}
\usetikzlibrary{arrows, positioning, automata}
\usepackage{comment}
\usepackage[shortlabels]{enumitem}

\usepackage[mathscr]{euscript}
\DeclareSymbolFont{rsfs}{U}{rsfs}{m}{n}
\DeclareSymbolFontAlphabet{\mathscrsfs}{rsfs}

\AtAppendix{\counterwithin{lemma}{section}}

\usepackage{mathtools}

\newtheorem{theorem}{Theorem}
\newtheorem{lemma}{Lemma}
\newtheorem{corollary}{Corollary}
\newtheorem{definition}{Definition}
\newtheorem{remark}{Remark}

\allowdisplaybreaks
\usepackage[top=1in,bottom=1in,left=1in,right=1in]{geometry}
\usepackage[utf8]{inputenc}

\newcommand{\voteone}{\textsc{Vote1X}\xspace}
\newcommand{\votetwo}{\textsc{Vote2X}\xspace}
\newcommand{\votethree}{\textsc{Vote3X}\xspace}
\newcommand{\ouralgo}{\textsc{StreamBP}\xspace}
\newcommand{\boundedalgo}{\textsc{StreamBP$^*$}\xspace}
\newcommand{\offlinealgo}{\textsc{BP}\xspace}

\newcommand{\cA}{\mathcal{A}}

\newcommand{\bG}{\bold{G}}

\newcommand{\Bprop}{\mathsf{Bprop}}

\def\NN{\mathbb{N}}

\def\RR{\mathbb{R}}

\def\bG{\mathbf{G}}

\def\bm{\mathbf{R}}

\def\bb{\boldsymbol{b}}
\def\bc{\boldsymbol{c}}
\def\bd{\boldsymbol{d}}
\def\be{\boldsymbol{e}}

\def\bm{\boldsymbol{m}}

\def\bm{\boldsymbol{m}}

\def\bw{\boldsymbol{w}}

\def\bG{\boldsymbol{G}}

\newcommand\ta{\tilde a}
\newcommand\tb{\tilde b}

\renewcommand{\P}{\mathbb{P}}
\newcommand{\E}{\mathbb{E}}

\newcommand{\Var}{\operatorname{Var}}
\newcommand{\argmax}{\operatorname{argmax}}

\newcommand{\Unif}{\operatorname{Unif}}

\newcommand{\RN}[1]{%
  \textup{\uppercase\expandafter{\romannumeral#1}}%
}

\newcommand\iidsim{\stackrel{\mathclap{iid}}{\sim}}

\newcommand{\RNum}[1]{\uppercase\expandafter{\romannumeral #1\relax}}

\makeatletter
\newenvironment{breakablealgorithm}
  {
   \begin{center}
     \refstepcounter{algorithm}
     \hrule height.8pt depth0pt \kern2pt
     \renewcommand{\caption}[2][\relax]{
       {\raggedright\textbf{\ALG@name~\thealgorithm} ##2\par}%
       \ifx\relax##1\relax 
         \addcontentsline{loa}{algorithm}{\protect\numberline{\thealgorithm}##2}%
       \else 
         \addcontentsline{loa}{algorithm}{\protect\numberline{\thealgorithm}##1}%
       \fi
       \kern2pt\hrule\kern2pt
     }
  }{
     \kern2pt\hrule\relax
   \end{center}
  }
\makeatother

\makeatletter
\newcommand*{\rom}[1]{\expandafter\@slowromancap\romannumeral #1@}
\makeatother

\def\SSBM{{\rm SSBM}}
\def\ttau{\tilde{\tau}}
\def\htau{\hat{\tau}}
\def\cA{{\mathcal{A}}}
\def\fS{{\mathfrak S}}
\def\Ball{{\sf B}}

\def\ind{{\mathds{1}}}
\def\bfone{{\mathbf 1}}
\def\cG{{\mathcal G}}
\def\cE{{\mathcal E}}
\def\cV{{\mathcal V}}
\def\ocG{\bar{\mathcal G}}
\def\ocE{\bar{\mathcal E}}
\def\ocV{\bar{\mathcal V}}

\def\reals{\mathbb{R}}

\def\ttau{\tilde{\tau}}

\def\SBM{{\rm SBM}}
\def\STSSBM{{\rm StSSBM}}
\def\STSBM{{\rm StSBM}}
\def\Unif{{\sf Unif}}
\def\<{\langle}
\def\>{\rangle}
\def\Lip{\mbox{\tiny\rm Lip}}
\def\BP{\textsc{BP}\xspace}

\def\est_ac{Q_n}             
\def\opt_est_ac{Q_n^{\ast}}  
\def\rand_alg_quant{\xi}     
\def\localV{\mathsf{V}^{\mathsf{act}}}    
\def\localE{\mathsf{E}^{\mathsf{act}}}    

\pdfoutput=1

\title{Streaming Belief Propagation for Community Detection}
\author{ 
	Yuchen Wu\thanks{Department of Statistics, Stanford University}
	\and
	MohammadHossein Bateni\footnotemark[2]\thanks{Google Research}
	\and
	André Linhares\footnotemark[2]
	\and
	Filipe Miguel Gonçalves de Almeida\footnotemark[2]
	\and
	Andrea Montanari\footnotemark[1]
	\thanks{Department of Electrical Engineering, Stanford University}
	\and
	Ashkan Norouzi-Fard\footnotemark[2]
	\and
	Jakab Tardos\thanks{Ecole Polytechnique Fédérale de Lausanne}
}
\date{\today}
\begin{document}

\maketitle

\begin{abstract}
The community detection problem requires to cluster the nodes of a network
into a small number of well-connected `communities'. 
There has been substantial recent progress in characterizing the
fundamental statistical limits of community detection under simple stochastic block models.
However, in real-world applications, the network structure is typically dynamic, with nodes that join over time.
In this setting, we would like a detection algorithm to perform only a limited number of updates at each node arrival.
While standard voting approaches satisfy this constraint, it is unclear whether they 
exploit the network information optimally.
We introduce a simple model for networks growing over time which we refer to as 
\emph{streaming stochastic block model} (\STSBM). 
Within this model, we prove that voting algorithms have fundamental limitations. We also develop a streaming belief-propagation
(\ouralgo) approach, for which we prove optimality in certain regimes.
We validate our theoretical findings on synthetic
and real data.
\end{abstract}

\section{Introduction}
Given a single realization of a network $G=(V,E)$, the community detection problem requires to find a partition of its vertices into
a small number of clusters or `communities' \cite{girvan2002community, madeira2004biclustering,jiang2004cluster,fortunato2010community,papadopoulos2012community,javed2018community}.

Numerous methods have been developed for community detection on static networks \cite{girvan2002community, clauset2004finding, guimera2005functional,radicchi2004defining,lancichinetti2009community,von2008consistency,wang2018extreme}.  
However, the network structure evolves over time in most applications. 
For instance, in social networks new users can join the network; in online commerce new products can be listed; 
see
\cite{WGKM18,GLZ08, KZBHLLNSSMS14} for other examples.
In such dynamic settings, it is desirable to have algorithms that perform only a limited number of operations each time a node joins or
leaves, possibly revising the labels 
of nodes in a neighborhood of the new node. (These notions will be formalized below.) 
Several groups have developed algorithms of this type in the recent past \cite{holme2012temporal,cordeiro2016dynamic,zeng2019consensus,martinet2020robust}.
The present paper aims at characterizing the fundamental statistical limits of community detection in the dynamic setting,
and proposing algorithms that achieve those limits.

As usual, establishing fundamental statistical limits requires introducing a statistical model for the network $G=(V,E)$. 
In the case of static networks, precise characterizations have only been achieved recently, and almost uniquely for
a simple network model, namely the stochastic block model (\SBM) \cite{holland1983stochastic, airoldi2008mixed,karrer2011stochastic, rohe2011spectral,decelle2011asymptotic,massoulie2014community,abbe2015exact,abbe2015community,mossel2018proof}.
In this paper we build on these recent advances and study a dynamic generalization of the \SBM, which we refer to as
{\em streaming  \SBM} (\STSBM). 

The \SBM\, can be defined as follows. Each vertex $v$ is  given a label $\tau(v)$ drawn independently from a fixed distribution over
the set $[k]=\{1,2,\dots,k\}$. Edges are conditionally independent given vertex labels. Two vertices $u$, $v$ are connected by an edge with
probability $W_{\tau(u),\tau(v)}$. The analyst is given a realization of the graph and required to estimate the labels $\tau$.
We will assume in addition that the analyst has access to some  noisy version of the vertex labels, denoted by $\ttau = (\ttau(v))_{v\in V}$.
This is a mathematically convenient  generalization: the special case in which no noisy
observations $\ttau$ are available can be captured by letting $\ttau$ be independent of $\tau$.
Further, such a generalization is useful to model cases in which covariate information
is available at the nodes \cite{mossel2016local}.

Informally, the {\em streaming  \SBM} (\STSBM) is a version of SBM in which nodes are revealed one at a time in random order (see below for a formal definition).
In order to model the notion that only a limited number of updates is performed
each time a new node joins the network, we introduce a class of `local streaming algorithms.' These encompass several
algorithms in earlier literature, e.g, \cite{ZGL03,cordeiro2016dynamic}.
Our definition is inspired and motivated  by the more classical definition of local algorithms for static graphs.  Local algorithms output estimates
for one vertex based only on a small  neighborhood around it, and thus scale well to large graphs; see \cite{suomela2013survey} for a survey.
A substantial literature studies the behavior of local algorithms for sparse graphs \cite{gharan2012approximating,hatami2014limits,gamarnik2014limits,montanari2015finding,mossel2016local,fan2017well}.
The recent paper \cite{manoel2017streaming} studies streaming in conjunction with message passing algorithms for Bayesian inference on dense graphs. 

Our results focus on the sparse regime in which the graph's average degree is bounded.  This is the most challenging regime for the \SBM, and
it is also relevant for real-world applications where networks are usually sparse. We present the following contributions:

\vspace{0.2cm}

\noindent{\emph{Fundamental limitations of local streaming algorithms.}} We prove that, in the absence of side information, in streaming symmetric \SBM ~(introduced in  Section \ref{sec:model}), 
local streaming algorithms (introduced in Section \ref{sec:local-streaming-alg}) do not
achieve any non-trivial reconstruction: their accuracy is asymptotically the same as random guessing; see Corollary \ref{cor:streaming-without-sideinfo}. This holds despite the fact that there exist polynomial time  non-local algorithms that achieve significantly better accuracy.
From a practical viewpoint, this indicates that methods with a small `locality radius' (the range over which the algorithm updates its
estimates) are ineffective at aggregating information: they perform poorly unless strong local side information is available.

\vspace{0.2cm}

\noindent{\emph{Local streaming algorithms with summary statistics.}} Given this negative result, it is natural to ask what are the limits of
general streaming algorithms in the \STSBM \hspace{0.05cm} that maintain a bounded amount of global memory, on top of the local information.
While we do not solve this question, we study a subset  of these algorithms that we call  `local streaming algorithms with summary statistics.'
These algorithms maintain bounded-size vectors at vertices and edges, as well as the averages of these vectors that are stored as
global information.  Under suitable regularity conditions on the update functions, we prove that (for the symmetric model and in
absence of side information) these algorithms do not perform
better than random guessing; see Theorem \ref{thm:local-with-summary-stat-is-trivial}.

\vspace{0.2cm}

\noindent{\emph{Optimality of streaming belief propagation.}} 
On the positive side, in Section \ref{sec:sbp} we define a streaming version of belief propagation (\BP),
a local streaming algorithm that we call \ouralgo, parameterized by a locality radius $R$. We prove that, for any non-vanishing amount of
side information, \ouralgo achieves the same reconstruction accuracy as offline BP; see Theorem \ref{thm-streaming-local-bp-is-optimal}. 
The latter is in turn conjectured to be the optimal
offline polynomial-time algorithm \cite{decelle2011asymptotic,abbe2017community,hopkins2017efficient}. 
 Under this conjecture, there is no loss of performance in restricting to local streaming algorithms as long as
$(1)$~local side information is available; $(2)$~the locality radius is sufficiently large;
and $(3)$~information is aggregated optimally via \ouralgo.

\vspace{0.2cm}

Let us emphasize that we do not claim (nor do we expect) \ouralgo to outperform offline BP.
We use offline BP as an `oracle' benchmark (as it has the full graph information available, and is not constrained to act in a streaming fashion).

\vspace{0.2cm}

\noindent{\emph{Implementation and numerical experiments.}} In Section \ref{sec:empirical} and Appendix \ref{sec:synthetic}--\ref{sec:realworld}, we validate our results both on synthetic data, generated
according to the \STSBM, and on real datasets. Our empirical results are consistent with the theory; in particular,
\ouralgo substantially outperforms simple voting methods. However, we observe that it can behave poorly with large locality radius $R$.
In order to overcome this problem, we introduce a `bounded distance' version of \ouralgo, called \boundedalgo,
which appears to be more robust. (\boundedalgo can be shown to enjoy the same theoretical guarantees as \ouralgo.)

\section{Streaming stochastic block model}\label{sec:model}

In this section we present a formal definition of the proposed model. The streaming stochastic block model  
	is a probability distribution 
	$\STSBM(n,k,p,W,\alpha)$
	over triples
$(\tau,\ttau,\bG)$
where $\tau\in [k]^n$ is a  vector of labels (here $[k]\triangleq\{1,\dots, k\}$),
$\ttau\in [k]^n$ are noisy observations of the labels $\tau$, and $\bG =(G(0),G(1), \dots,G(n))$
is a sequence of undirected graphs.
Here $G(t) = (V(t),E(t))$ is a graph over $|V(t)|=t$ vertices and, for each
$0\le t\le n - 1$, $V(t)\subseteq V(t+1)$ and $E(t)\subseteq E(t+1)$.
We will assume, without loss of generality, that $V(n) = [n]$, and interpret $\tau(v)$
as the label associated to vertex $v\in [n]$. For each $0\le t\le n-1$, $G(t)$ is the subgraph induced in $G(t+1)$ by $V(t)$; equivalently, all edges in $E(t+1) \setminus E(t)$ are incident to the unique vertex in $V(t+1) \setminus V(t)$. 

The distribution  $\STSBM(n,k,p,W,\alpha)$ is parameterized by a scalar $\alpha\in [0,\frac{k-1}{k}]$, a probability vector
$p = (p_1, \ldots, p_k)\in \Delta_k\triangleq\{x\in [0,1]^k, \<x,1\>=1\}$, and a symmetric matrix $W\in [0,1]^{k\times k}$.
We draw the coordinates of $\tau$ independently with distribution  $p$, and set $\ttau(v) = \tau(v)$
with probability $1-\alpha$, and $\ttau(v) \sim \Unif([k]\setminus\{\tau(v)\})$ otherwise, independently across vertices:
%
\begin{align*}
  &\P\big(\tau(v) = s\big) = p_s,\qquad  \P\left( \ttau(v) = s_1 \mid \tau(v) = s_0 \right) =
  \begin{cases}
		1 - \alpha& \text{if } s_1 = s_0,\\
		\alpha / (k-1)& \text{if } s_1 \neq s_0.
	\end{cases}
\end{align*}
We then construct $G(n)$ by generating conditionally independent edges, given $(\tau,\ttau)$, with
\begin{align}
  \P\big((u,v)\in E(n)\mid \tau,\ttau\big) = W_{\tau(u),\tau(v)}. \label{eq:EdgeDistr}
\end{align}
Note that the labels $\ttau$ provide noisy `side information' about the true labels $\tau$. This information $\ttau$ is conditionally independent of
the graph $G(n)$ given $\tau$. 
Finally we generate the graph sequence $\bG$ by choosing a uniformly random permutation of
the vertices $(v(1),v(2),\dots,v(n))$ and setting $V(t) =\{v(1),\dots,v(t)\}$ and $G(t)$ to the graph induced by $V(t)$. If $v = v(t)$, then we define $t$ as the arrival order of vertex $v$. Note that, for each $t$, $G(t)$ is distributed according to a standard SBM  with $t$ vertices: $G(t)\sim \SBM(t,k,p,W)$.

An equivalent description is that (conditional on $\tau,\ttau$) $\STSBM(n,k,p,W,\alpha)$ defines a Markov chain over graphs.
The new graph $G(t+1)$ is generated from $G(t)$ by drawing the vertex $v(t+1)$ uniformly at random from $[n]\setminus V(t)$,
and then the edges $(u,v(t+1))$, $u\in V(t)$,  independently with probabilities given by Equation \eqref{eq:EdgeDistr}.

We are interested in the behavior of large graphs with bounded average degree. In order to focus on this regime,
we will consider $n\to\infty$ and $W=W^{(n)}\to 0$ with $W^{(n)} = W_0/n$ for a matrix $W_0\in\reals_{\ge0}^{k\times k}$ independent of $n$. 

A case of special interest is  the streaming symmetric \SBM, $\STSSBM(n, k, a, b, \alpha)$,
which corresponds to taking $p=(1/k,\dots,1/k)$ and $W_0$ having diagonal elements $a$ and non-diagonal elements $b$. Finally, the case $\alpha=(k-1)/k$
corresponds to pure noise $\ttau$: in this case we can drop $\ttau$ from the observations and we will drop
$\alpha$ from the distribution parameters.

\subsection{Definitions and notations}\label{sec:defs}
For two nodes $v, v' \in V(t)$, we denote by $d_t(v, v')$ their graph distance in $G(t)$, i.e., the length of the shortest path in $G(t)$ connecting $v$ and $v'$, with $d_t(v, v') = \infty$ if no such path exists. We also write $d(v, v') = d_n(v, v')$ for the graph distance in $G(n)$. For $v \in V(n)$ and $R \in \mathbb{N}^+$, let $\Ball_R^t(v) = (V_R^t(v), E_R^t(v))$ denote the ball of radius $R$ in $G(t)$ centered at $v$, i.e., the subgraph induced in $G(t)$ by nodes $V_R^t(v) \triangleq \{ v' \in V(t) : d_t(v, v') \leq R\}$ and edges $(v_1, v_2) \in E(t)$ with $v_1, v_2 \in V_R^t(v)$. Furthermore, let $D_R^t(v) \triangleq \{v' \in V(t): d(v, v') = R\}$. For $A \subseteq V$, let $\tau(A)$ be the vector containing all true labels of vertices in $A$, and $\tilde{\tau}(A)$ be the vector containing all noisy labels of vertices in $A$. Throughout the paper, unless otherwise stated, we assume $(v_1,v_2)$ is an undirected edge.  


We consider an algorithm $\cA$ that takes as input the graph $G(n)$ and side information $\ttau$, and for each $v\in V(n)$ outputs $\cA(v; G(n), \ttau)\in [k]$ as an estimate for $\tau(v)$. Note that we always assume the arrival orders of vertices are observed (i.e., by observing $v \in [n]$ we also observe the unique $t \in [n]$ such that $v = v(t)$), thus $G(n)$ contains the arrival order of its vertices. Define the estimation accuracy of algorithm $\cA$ as
\begin{align}
	\est_ac(\cA) \triangleq \E\hskip-1mm\left[ \max_{\pi\in \fS_k}\frac{1}{n} \hskip-1mm \sum_{v \in V(n)} \hskip-2mm \ind\left(\cA(v; G(n), \ttau) = \pi\circ\tau(v)\right) \right]\hskip-1mm.\label{eq:nodeAccuracy}
\end{align}
Here $\fS_k$ is the group of permutations over $[k]$. Denote the optimal estimation accuracy by
\begin{align}
	\opt_est_ac \triangleq \sup\limits_{\cA}\E\hskip-1mm\left[ \max_{\pi\in \fS_k}\frac{1}{n} \hskip-1mm\sum_{v \in V(n)} \hskip-2mm \ind\left(\cA(v; G(n), \ttau) = \pi\circ\tau(v)\right) \right]\hskip-1mm.\label{eq:OptimalAccuracy}
\end{align}
Here the supremum is taken over all algorithms, not necessarily local or online. In the above expressions, the expectation is with respect to $G(n),\tau,\ttau$, and the randomness of the algorithm (if $\cA$ is randomized).
%

\section{Local streaming algorithms}
\label{sec:local-streaming-alg}

In this section we introduce the local streaming algorithm, which is a generalization of local algorithm in the dynamic network setting. An \textit{$R$-local streaming  algorithm} is an algorithm that at each vertex keeps some information available
to that vertex. As a new vertex $v(t)$ joins, information within the $R$-neighborhood $\Ball_R^t(v(t))$ is  pulled. An estimate for $\tau(v)$ is constructed based on information available to $v$.
In order to accommodate randomized algorithms we assume that  random variables
$(\rand_alg_quant_v)_{v\in V(n)} \iidsim \Unif([0,1])$, independent of the graph, are part of the local information
available to the algorithm.

As an example, we can consider a simple voting algorithm. 
At each step $t$, this algorithm keeps in memory
the current estimates $\htau(v)\in[k]$ for all $v\in V(t)$. As a new vertex $v(t)$ joins, its estimated
label is determined according to
\begin{align}
  & \htau(v(t))= \argmax_{s\in[k]}\pi_{t}(s)\,  \qquad \pi_{t}(s) = \delta\, \ind(s=\ttau(v(t)))+
\hskip-4mm\sum_{(v(t),u)\in E(t)} \hskip-4mm \ind(s=\htau(u))\,. \label{eq:Voting}
\end{align}  
 %
In words, the estimated label at $v(t)$ is the winner of a voting procedure, where the neighbors of $v(t)$ contribute one vote
each, while the side information at $v(t)$ contributes $\delta$ votes. 

For $v \in V(n)$ and $t \in [n]$, we denote the subgraph accessible to $v$ at time $t$ by $\mathcal{G}_v^t = (\mathcal{V}_v^t, \mathcal{E}_v^t)$, with initialization $\mathcal{G}_v^0 = (\{v\}, \emptyset)$. At time $t$, we conduct the following updates:%
\begin{align*}
  \mathcal{V}_v^t \triangleq
  \begin{cases}
       \bigcup\limits_{v' \in V_R^t(v(t))} \mathcal{V}_{v'}^{t - 1}   &\mbox{ for }  v \in {V}_R^t(v(t)), \\
       \mathcal{V}_v^{t-1}  &\mbox{ for }   v \notin {V}_R^t(v(t)). 
    \end{cases}
\end{align*}
We let $\cG_v^t$  be the subgraph induced in $G(t)$ by $\cV_v^t$,  and
denote by $\ocG_v^t=(\ocV_v^t,\ocE_v^t)$ the  corresponding labeled
graph with vertex labels $\ttau$ and randomness $\xi$. Namely $\ocV_v^t \triangleq \{(v', \ttau(v'), \xi_{v'}): v' \in \cV_v^t\}$, $\ocE_v^t \triangleq \cE_v^t$.
Let us emphasize that the `neighborhoods' $\cG^t_v$ are not symmetric, in the sense that we can
have $v_1\in \cG^t_{v_2}$ but  $v_2\not\in \cG^t_{v_1}$.
\begin{definition}[$R$-local streaming algorithm]\label{def:streaming-R-local-algorithm}
An algorithm $\cA$ is an \emph{$R$-local streaming algorithm} if, at each time $t$ and for each vertex $v \in V(t)$, it outputs an estimate of $\tau(v)$  denoted by $\mathcal{A}^t(v; G(t), \ttau) \in [k]$, which is a function uniquely of $\bar{\cG}_v^t$.
\end{definition}
Note that this class includes as special cases voting algorithms (which correspond to $R=1$) but also a broad class of other approaches. We will compare $R$-local streaming algorithms with $R$-local algorithms (non-streaming).
In order to define the latter, given a neighborhood $\Ball_R^t(v)$, we define the corresponding labeled graph as
$\bar{\Ball}_R^t(v) \triangleq (\bar{V}_R^t(v), E_R^t(v))$, with $\bar{V}_R^t(v) \triangleq \{(v', \tilde{\tau}(v'), \rand_alg_quant_{v'}): v' \in V_R^t(v(t))\}$.

\begin{definition}[$R$-local algorithm]
An algorithm $\cA$ is an \emph{$R$-local algorithm} if, at each time $t$ and for each vertex $v \in V(t)$, it outputs an estimate of $\tau(v)$  denoted by $\mathcal{A}^t(v; G(t), \ttau) \in [k]$, which is a function uniquely of $\bar{\Ball}_R^t(v)$.
      \end{definition}
      For simplicity, define the final output of an algorithm $\cA$ by $\cA(v; G(n), \ttau) \triangleq \cA^n(v; G(n), \ttau)$. The next theorem states that, under \STSSBM, any local streaming algorithm with fixed radius behaves asymptotically as a local algorithm. Here we focus on \STSSBM---extension to asymmetric cases is straightforward. 
\begin{theorem}\label{thm:streaming-local-implies-local}
  Let $\bG$ be distributed according to \STSSBM$(n, k, a, b, \alpha)$, and $v_0 \sim \Unif( [n] )$ be a vertex chosen independently of $\bG$.
  Then, for any $\epsilon > 0$, there exist $n_{\epsilon}, r_{\epsilon} \in \mathbb{N}^+$, such that for every $n \geq n_{\epsilon}$ with probability at least $1 - \epsilon$, the following properties hold:
	\begin{enumerate}
		\item $\cG_{v_0}^n$ is a subgraph of $\Ball_{r_{\epsilon}}^n(v_0)$;
		\item $v_0$ does not belong to $\cG_v^n$ for any $v \in V(n) \backslash V_{r_{\epsilon}}^n(v_0)$.
	\end{enumerate}
\end{theorem}

Under the symmetric SBM, local algorithms without side information cannot achieve non-trivial estimation accuracy as defined in \eqref{eq:nodeAccuracy}
\cite{kanade2016global}. Therefore we have the following  corollary of the first part of Theorem \ref{thm:streaming-local-implies-local}.
\begin{corollary}\label{cor:streaming-without-sideinfo}
	Under $\STSSBM(n, k, a, b)$ with no side information,  no $R$-local streaming algorithm $\cA$ can achieve non-trivial estimation accuracy. Namely,  
	\begin{align*}
	\lim_{n\to\infty}\est_ac(\cA) =\frac{1}{k}\, .
	\end{align*}
  \end{corollary}

      \begin{remark}
       Corollary \ref{cor:streaming-without-sideinfo}  does not hold if side information is available.
        As we will see below, an arbitrarily
       small amount of side information (any $\alpha<(k-1)/k$) can be boosted to ideal accuracy \eqref{eq:OptimalAccuracy} using $R$-local streaming
       algorithms with sufficiently large $R$. 
       On the other hand, for a fixed small $R$, a small amount of side information
       only has limited impact on accuracy.
       Our numerical simulations illustrate this for voting algorithms,
       which are $R$-local for $R=1$: they do not provide substantial boost over the use of only side information (i.e., the estimated label $\htau(v) = \ttau(v)$
       at all vertices).
    \end{remark}

   \section{Local streaming algorithms with summary statistics}\label{sec:summary-stat}

    The class of local algorithms is somewhat restrictive. In practice we can imagine keeping a small memory containing global
    information and updating it each time a new vertex joins. We will not consider general streaming algorithms under a
    memory constraint; we instead consider a subclass that we name `local streaming algorithms with summary statistics'.

    Formally, the state of the algorithm at time $t$ is encoded in two vectors $\bw^t = (w^t_i)_{i\in V(t)}\in(\reals^m)^{V(t)}$,
    $\be^t= (e^t_{ij})_{(i, j)\in E(t)}\in (\reals^m)^{E(t)}$,    indexed respectively by the vertices and edges of $G(t)$.
    Here $m$ is a fixed integer independent of $n$.
    These are initialized to independent random variables $w^{t_*(i)-1}_i\iidsim P_w$,   $e^{t_*(i)\vee t_*(j)-1}_{ij}\iidsim P_e$,
    where $t_*(i)$ is the time at which vertex $i$ joins the graph ($v(t_*(i)) =i$), $t_*(i)\vee t_*(j) \triangleq \max\{t_*(i), t_*(j)\}$, and $P_w, P_e$ are probability distributions over $\RR^m$. 
    

    At each $ t \in [n]$, a new vertex $v(t)$ joins the graph, and a `range of action' $(\localV_t, \localE_t)$ is decided,
    with $\localV_t \subseteq V(t)$ a vertex set and $\localE_t \subseteq E(t)$ an edge set.
    We assume $(\localV_t, \localE_t)$ to depend uniquely on the $R$-neighborhood of $v(t)$,
    $\Ball_R^t(v(t))$, and to be such that:
    $(i)$~the range of action is a subset of the neighborhood: $\localV_t\subseteq V_R^t(v(t))$,  $\localE_t\subseteq E_R^t(v(t))$; and
    $(ii)$~the range of action has bounded size: $|\localV_t|+|\localE_t|\le C_{act} = C_{act}(R)$ which does not scale with $n$. 
    Notice that the second condition
    is only required because the maximum degree in $G(n)$ is $\log n$, and it is to avoid pathological behavior due to high-degree vertices; we believe it should be possible to avoid it at the cost of extra technical work.

    At each  time $t$, the algorithm updates the quantities $w^t_i$, $e^t_{ij}$ in the range of  action:
    \begin{align*}
      w_i^t &= F_w^t(\bw^{t - 1}(\localV_t), \be^{t - 1}(\localE_t), \bar{w}^{t - 1}, \bar{e}^{t - 1} | i)\, , & \forall i\in \localV_t,\\
      e_{ij}^t &= F_e^t(\bw^{t - 1}(\localV_t), \be^{t - 1}(\localE_t), \bar{w}^{t - 1}, \bar{e}^{t - 1} | i,j)\, , & \forall (i,j)\in \localE_t.
    \end{align*}
    Here  $\bw^{t - 1}(\localV_t)$, $\be^{t - 1}(\localE_t)$ are the restrictions of $\bw^{t-1}, \be^{t-1}$ to the range of action sets,
    and $\bar{w}^{t-1}$, $\bar{e}^{t-1}$ are summary statistics, updated according to:
\begin{align*}
      \bar{w}^t =\frac{1}{|V(t)|}\sum_{v \in V(t)} w_{v}^t\,,\;\;\;\;
      \bar{e}^t = \frac{1}{|E(t)|} \sum\limits_{(i, j) \in E(t)}e_{ij}^t\,.
\end{align*}
      Finally, vertex labels are estimated using a function $\htau:\reals^m\to [k]$. Namely, label at vertex $v$ is
      estimated at time $t$ as $\htau(w_v^t)$.

      We next establish that, under Lipschitz continuity assumptions on the update functions, local streaming algorithms with summary statistics cannot achieve non-trivial reconstruction in the
      symmetric model $\STSSBM(n, k, a, b)$. Notice that this claim cannot hold for a general algorithm in this class.
      Indeed, each node $i$ could encode the structure of (a bounded-size subgraph)  $\Ball_R^t(i)$ in the decimal expansion
      of $w_i^t$, in such a way that distinct vertices use non-overlapping sets of digits. Then the summary statistics $\bar{w}^t$
      would contain the structure of the whole graph. We avoid this by requiring the update functions to be bounded Lipschitz,
      and adding a small amount $\varepsilon$ of noise to $w_i^n$, before taking a decision. Informally, we are assuming that adding small perturbation will not have large effect on the final output. 
\begin{theorem}\label{thm:local-with-summary-stat-is-trivial}
  Assume that there exist numerical constant $L_F$, independent of $n$, such that for all $t \in [n]$, all $i \in[t]$ and all $1 \leq j < l \leq t$, we have
  $\|F_w^t(\, \cdot \, | i)\|_{\infty}, \|F_e^t(\,\cdot\, | j,l)\|_{\infty} \le L_F$ and
  $\|F_w^t(\, \cdot \, | i)\|_{\Lip}, \|F_e^t(\,\cdot\, | j,l)\|_{\Lip} \le L_F$. (Here $\|f\|_{\Lip}$ denotes the Lipschitz modulus of function $f$.) Let $\{w_i^t: t\le n, i\in V(t)\}$ be the vertex variables generated by the local streaming algorithm
  with summary statistics defined by functions $F_w, F_e$. Let $\htau:\reals^m\to [k]$ and $(U_{ij})_{i\le n,j\le m}\iidsim\Unif([-1,1])$ independent of the other randomness. Let $U_i = (U_{ij})_{j \le m} \in \RR^m$.
  Then under $\STSSBM(n, k, a, b)$, for any $\varepsilon>0$,
  \begin{align*}
    \limsup\limits_{n \rightarrow \infty}\E & \left[\max\limits_{\pi\in \fS_k} \frac{1}{n} \sum\limits_{i = 1}^n \ind{(\htau(w_i^n+
    \varepsilon U_i) = \pi \circ \tau(i))}\right] = \frac{1}{k}\, .
  \end{align*}
\end{theorem}

   \section{Streaming belief propagation}\label{sec:sbp}
   
In this section we focus on the symmetric model \STSSBM. Notice that this model makes community detection more difficult compared to the asymmetric model, as in the latter case average degrees for vertices are different across communities, and one can obtain non-trivial estimation accuracy by simply using the degree of each vertex. For the  symmetric model \STSSBM$(n,k,a,b,\alpha)$,  we proved that local streaming algorithms cannot provide any non-trivial
reconstruction of the true labels if no side information is provided, i.e., if $\alpha=(k-1)/k$. We also comment that, for a
fixed (small) $R$, accuracy achieved by any algorithm is continuous in $\alpha$, and hence a small amount of side information will have a small effect.

In contrast, for any non-vanishing side information $\alpha<(k-1)/k$,
we conjecture that information-theoretically optimal reconstruction is possible using a local streaming algorithm, under two conditions:
    $(i)$~the locality radius $R$ is large enough; and $(ii)$~the following Kesten-Stigum (KS) condition is met:
    \begin{align}
      \lambda \triangleq \frac{(a-b)^2}{a+(k-1)b}>1\, .\label{eq:KS}
    \end{align}
    We will refer to $\lambda$ as to the `signal-to-noise ratio' (SNR). 

    We provide evidence towards this conjecture by proposing a streaming belief propagation algorithm (\ouralgo)
    and showing that it achieves asymptotically the same accuracy as the standard offline BP algorithm. 
    The latter is believed to achieve
    information-theoretically optimal reconstruction above the KS threshold \cite{decelle2011asymptotic,mossel2014belief}.
    We will describe \ouralgo in the setting of the
    symmetric model \STSSBM$(n,k,a,b,\alpha)$, but its generalization to the asymmetric case is immediate.

    \begin{figure}[h]
    \centering
      \includegraphics[width=0.23\linewidth]{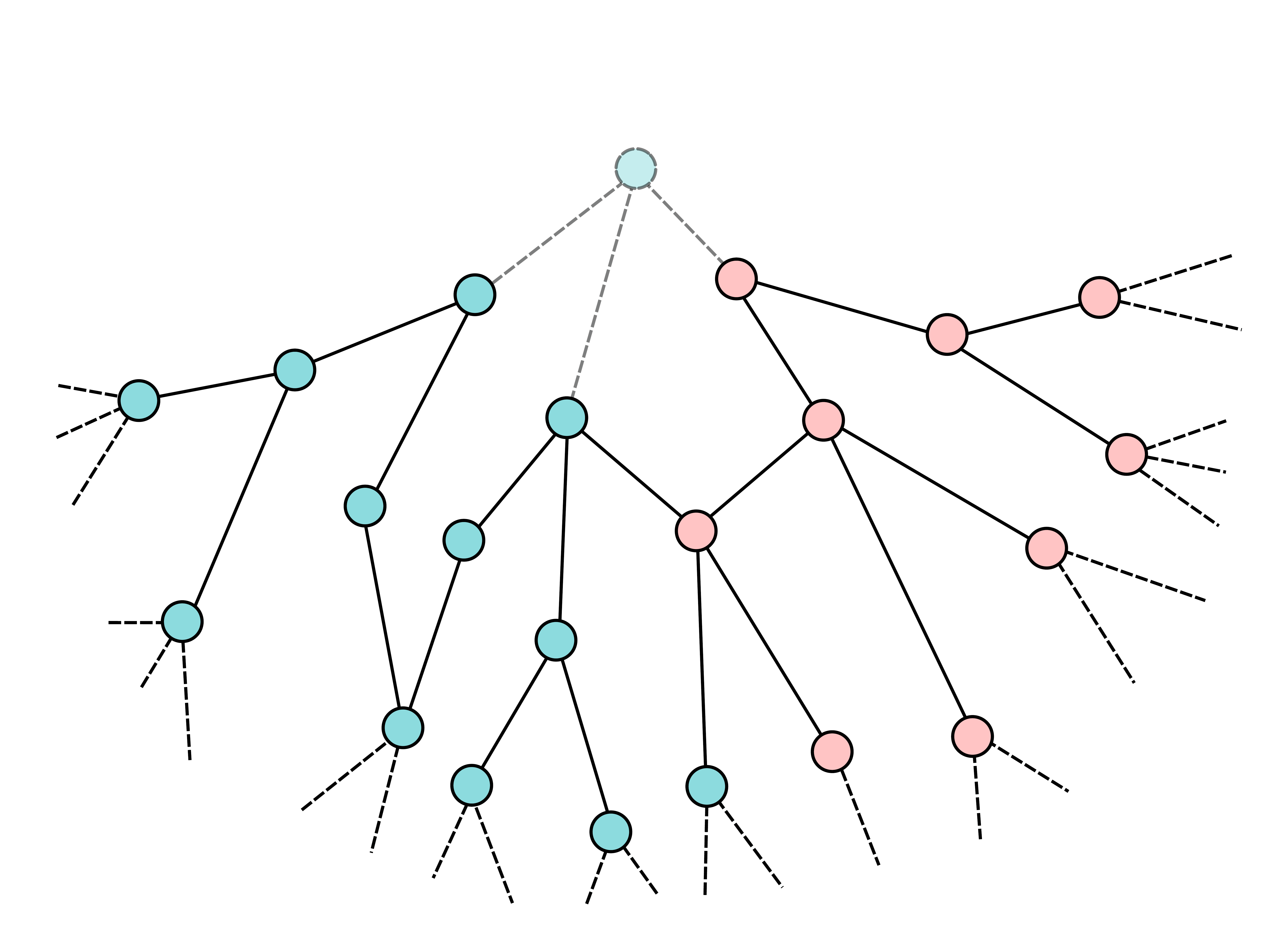}
      \includegraphics[width=0.23\linewidth]{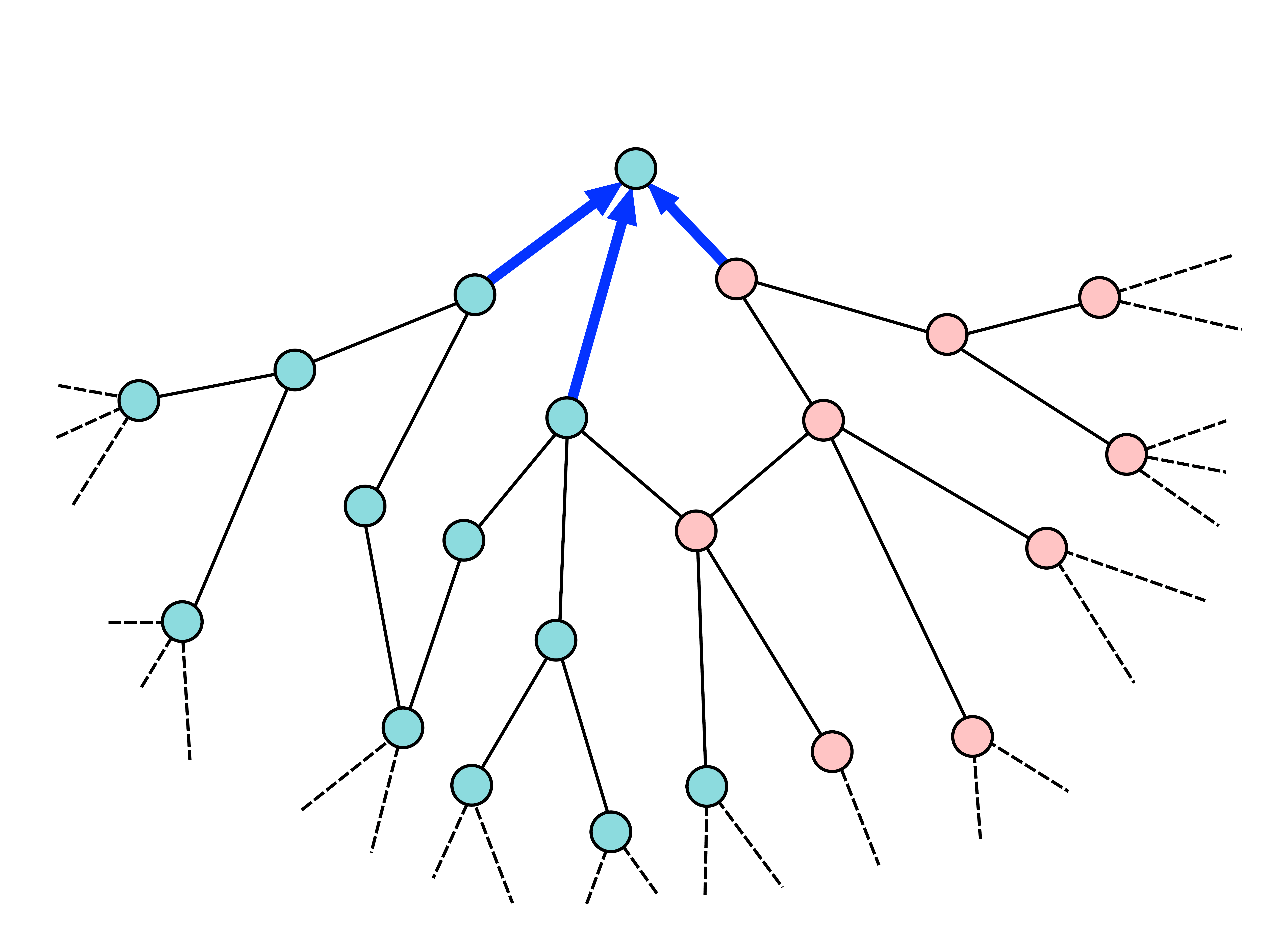}
      \includegraphics[width=0.23\linewidth]{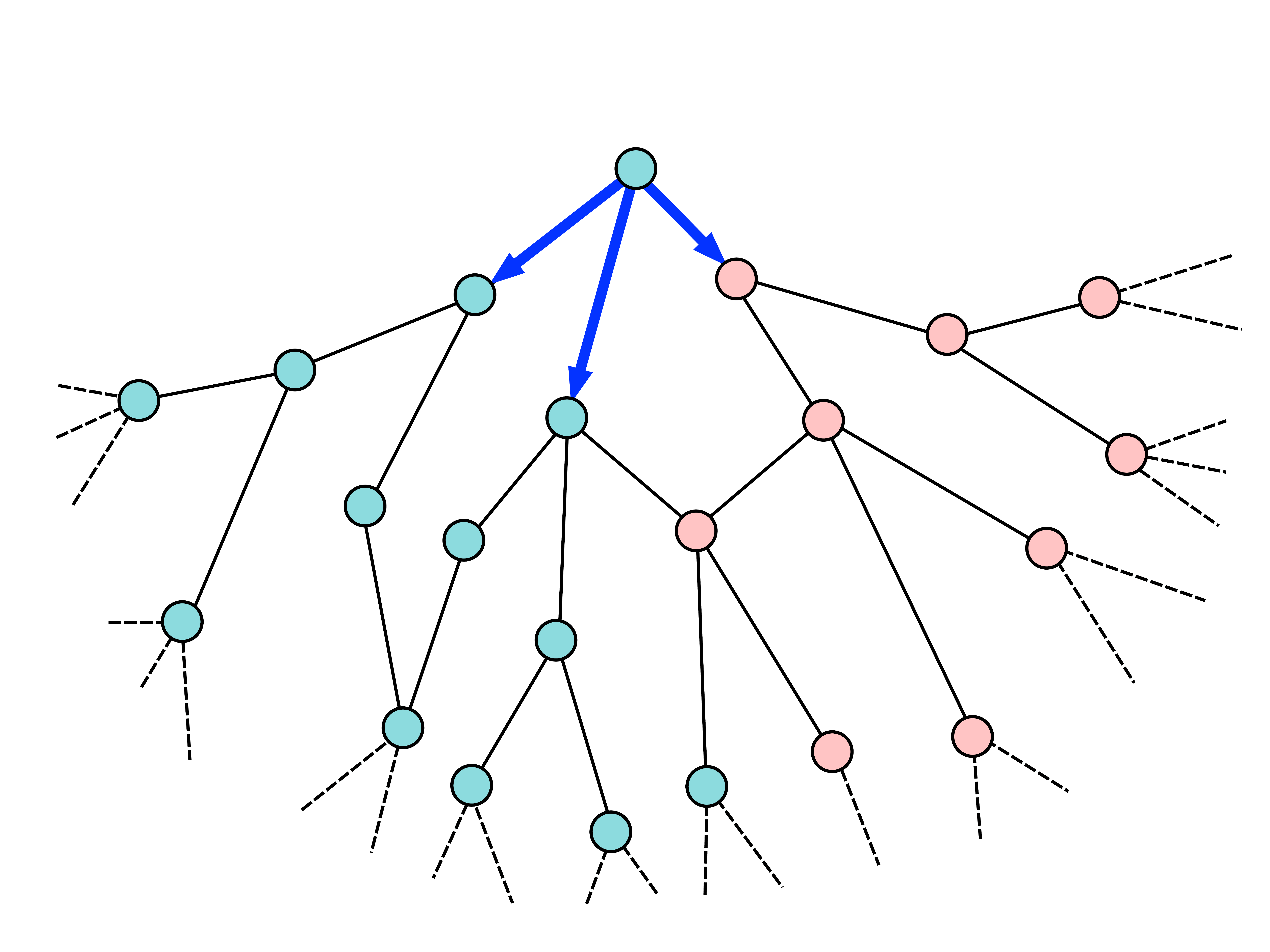}
      \includegraphics[width=0.23\linewidth]{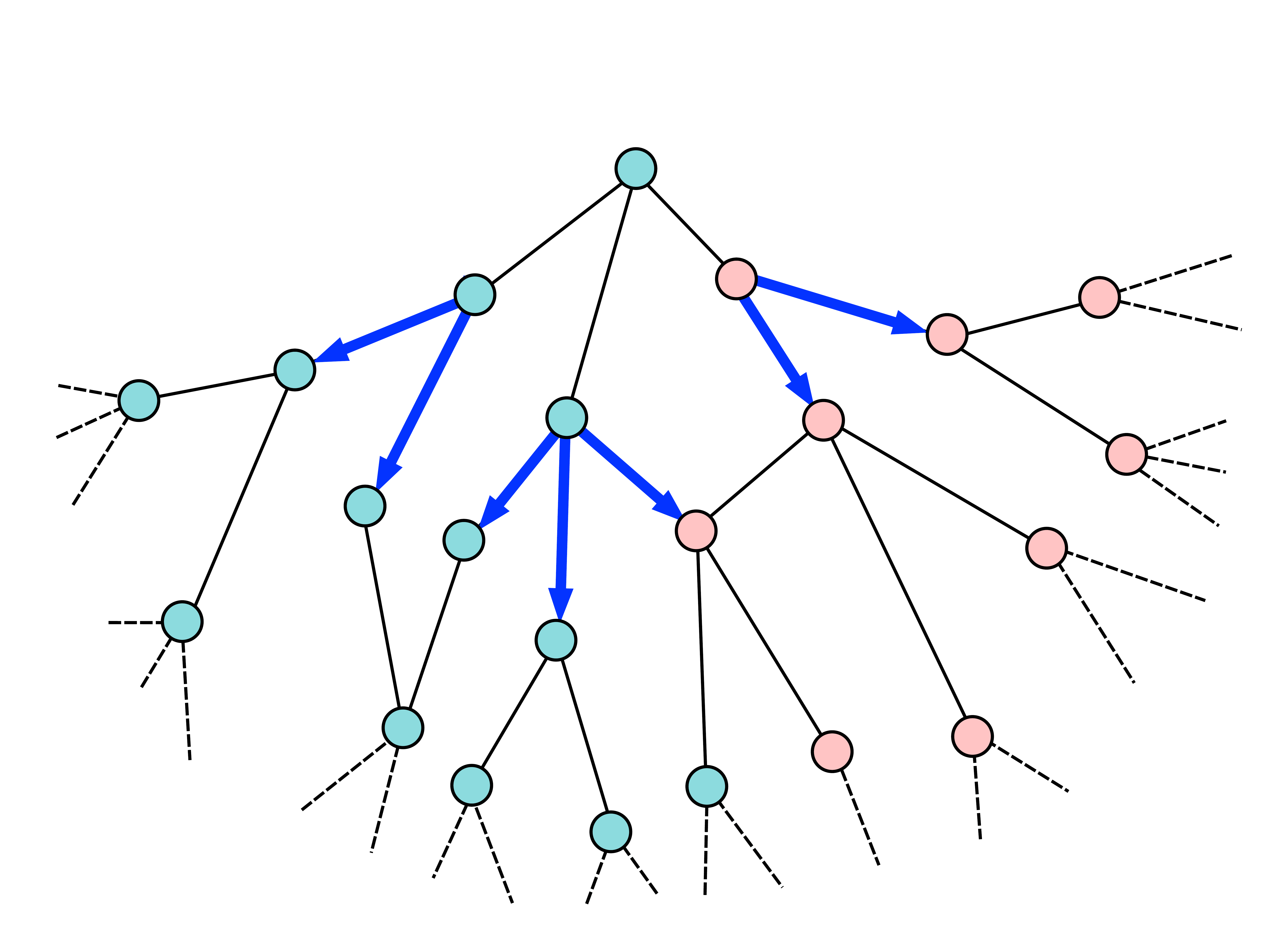}
    \caption{Update schedule of \ouralgo. Upon the arrival of a new vertex (shown in the leftmost figure), \ouralgo performs the belief propagation updates corresponding to the blue edges in the three other figures, in the order from left to right. }
    \label{fig:BP-Update}
\end{figure}
    The algorithm has a state which is given by a vector of messages indexed by directed edges in $G(t)$, 
    $\bm^t=\{\bm^t_{u\to v}, \bm^t_{v\to u}: (u,v)\in E(t)\}$. 
    Note that $G(t)$ is an undirected graph, and each edge $(u,v)$ corresponds to two messages indexed by $u \rightarrow v$ and $v \rightarrow u$. Each message is a probability distribution over $[k]$:
\begin{align*}
	\bm_{u \rightarrow v} = (m_{u \rightarrow v}(1), m_{u \rightarrow v}(2), \ldots, m_{u \rightarrow v}(k)) \in \Delta_k\, .
\end{align*}
The BP update is a map $\BP: (\Delta_k)^*\times[k]\to\Delta_k$, where  $(\Delta_k)^*$ denotes the  finite sequences of elements of $\Delta_k$:
\begin{align}
  \BP(\{\bm_i\}_{i\le \ell};\ttau)(s)  :=\frac{\BP_0(\ttau)(s)}{Z} \prod_{i=1}^{\ell}\big(b+(a-b)m_i(s)\big)\, .\label{eq:BPdef}
\end{align}
Here $\BP_0(\ttau)(s)\triangleq\big(\alpha+(k-1-k\alpha)
\ind_{\ttau=s}\big)/(k-1)$ and
the constant $Z=Z(\{\bm_i:\, i\le \ell\};\ttau)$ is defined implicitly by the normalization condition
$\sum_{s\in[k]}\BP(\{\bm_i:\, i\le \ell\};\ttau)(s) =1$. When a message $v\to u$ is updated, we compute its new value by applying the function \eqref{eq:BPdef} to the incoming messages into vertex $v$, with the exception of $u\to v$ (non-backtracking property):
\begin{align}
  \bm_{v\to u}\gets \BP(\{\bm_{w\to v}:\, w\in \partial v\setminus\{u\}\};\ttau(v))\, .\label{eq:BPupdate}
\end{align}
Here $\partial v$ denotes the set of neighbors of vertex $v$ in the current graph. When a new vertex $v(t)$ joins at time $t$,
we use the above rule to: $(1)$~update all the messages incoming into $v(t)$, i.e., $w\to v(t)$, for $w$ a neighbor of $v(t)$ in $G(t)$,
and $(2)$~update all messages at distance $1 \leq \ell \leq R$ from $v(t)$ in $G(t)$, along paths outgoing from $v(t)$, in order of increasing distance $\ell$.
The pseudocode for \ouralgo is given in Algorithm \ref{alg:streaming-bp-with-side-information}, and an illustration in Figure \ref{fig:BP-Update}.

    \begin{breakablealgorithm}\label{alg:streaming-bp-with-side-information}
	\caption{Streaming $R$-local belief propagation}
	\begin{algorithmic}[1]
	\State Initialization: $V(0) = E(0) = G(0) = \emptyset$.
		\For{$t = 1,2,\cdots, n$}
			\State \textit{// Update the graph:}
			\State $V(t) \gets V(t - 1) \cup \{v(t)\}$
			\State $E(t) \gets E(t - 1) \cup \{(v(t), v): v \in V(t - 1), (v(t), t) \in E(n)\}$
			\State $G(t) \gets (V(t), E(t))$
			\State \textit{// Update the incoming messages:}
			\For{$w \in {D}_1^t(v(t))$}
                        \State $\bm_{w \rightarrow v(t)} \gets \BP(\{\bm_{u\to w}:\, u\in \partial w\setminus\{v(t)\}\};\ttau(w))$
			\EndFor
			\State \textit{// Update the outgoing messages:}
			\For{$r = 1,2,\cdots, R$}
				\For{$v \in {D}_r^t(v(t))$}
					\State Let $v' \in {D}_1^t(v)$ on a shortest path connecting $v$ and $v(t)$.
					\State  $\bm_{v' \rightarrow v} \gets \BP(\{\bm_{u\to v'}:\, u\in \partial v'\setminus\{v\}\};\ttau(v'))$  
				\EndFor
			\EndFor
		\EndFor
			\For{$u \in V$}
		    \State $\bm_{u} \gets \BP(\{\bm_{v\to u}\}_{v\in \partial u};\ttau(u))$
			\State Output $\htau(u):=\arg\max_{s\in[k]} m_u(s)$ as an estimate for $\tau(u)$. 
		\EndFor
	\end{algorithmic}
\end{breakablealgorithm}
For the sake of simplicity, we analyze this algorithm in the two-group symmetric model \STSSBM$(n, 2, a, b, \alpha)$.
We believe that the extension of this analysis to other cases  is straightforward, but we leave it out of this presentation.

Our main results are that \ouralgo achieves asymptotically at least the same accuracy
as offline BP, as originally proposed in \cite{decelle2011asymptotic} and analyzed, e.g., in \cite{mossel2016local} (we refer to the latter for a formal definition of the algorithm). 
Offline BP performs the update from equation \eqref{eq:BPupdate}
in parallel on all the edges of $G(n)$ for $R-1$
iterations, and then computes vertex estimates using
$\bm_{u} \gets \BP(\{\bm_{v\to u}\}_{v\in \partial u};\ttau(u))$, $\htau(u):=\arg\max_{s\in[k]} m_u(s)$. 
Note that,  for each $R$, this defines an $R$-local algorithm, and hence we will refer to $R$ as its radius.
\begin{theorem}\label{thm-streaming-local-bp-better-than-local-bp}
  For $v \in V(n)$, let $\cA_R(v; G(n), \ttau) \in [k]$ be the estimate of $\tau(v)$ given by Algorithm \ref{alg:streaming-bp-with-side-information}
  (\ouralgo), and $\cA_R^{\mbox{\tiny\rm off} }(v; G(n), \ttau) \in [k] $ be the estimate given by offline BP with radius $R$ (equivalently, BP with parallel updates, stopped after $R$ iterations).
  Under the model \STSSBM$(n, 2, a, b, \alpha)$, \ouralgo
performs at least as well as  offline BP:
	\begin{align*}
	\liminf\limits_{n\rightarrow \infty}\left(Q_{n}(\cA_R) - Q_{n}(\cA_R^{\mbox{\tiny\rm off}})\right) \geq 0.
	\end{align*}
      \end{theorem}
      It is conjectured that, in the presence of side information, i.e., for $\alpha<(k-1)/k$, offline BP is optimal among
      all polynomial-time algorithms \cite{decelle2011asymptotic} (provided $R$ can be taken arbitrarily large).
      Whenever this is the case, the above theorem implies that \ouralgo is optimal as well.
      In the case of the symmetric model, it is also believed that under the KS condition (Equation \ref{eq:KS}),  and
      for $\alpha<(k-1)/k$, offline BP does indeed achieve the information-theoretically optimal accuracy.
      This claim has been proven in certain cases by \cite{mossel2016local}: we can use the results of \cite{mossel2016local}
      in conjunction with Theorem \ref{thm-streaming-local-bp-better-than-local-bp} to obtain conditions under
      which \ouralgo is information-theoretically optimal.

      \begin{corollary}\label{thm-streaming-local-bp-is-optimal}
        Suppose one of the following conditions holds (for a sufficiently large absolute constant $C$) in the two-group symmetric model \STSSBM$(n,k=2,a,b,\alpha)$:
        %
        
\begin{enumerate}
    \item $|a - b| < 2$ and $\alpha\in(0, 1 / 2)$.
    \item $(a - b)^2 > C(a + b)$ and $\alpha \in (0,1 / 2)$.
    \item $\alpha \in (0,1/C)$.
\end{enumerate}

Then Algorithm \ref{alg:streaming-bp-with-side-information} achieves optimal estimation accuracy:
\begin{align*}
    \limsup\limits_{R \rightarrow \infty} \limsup\limits_{n\rightarrow \infty}\left(Q_{n}^{\ast} - Q_{n}(\mathcal{A}_R)\right) = 0.
\end{align*}
\end{corollary}

\section{Empirical evaluation}
\label{sec:empirical}

In this section, we compare the empirical performance of two versions of our streaming belief propagation algorithm (Algorithm \ref{alg:streaming-bp-with-side-information}) with some baselines. We consider the following streaming algorithms:

\begin{itemize}
    \item  \ouralgo: proposed algorithm in this work with radius $R$, outlined in Algorithm \ref{alg:streaming-bp-with-side-information}.
    \item  \boundedalgo: proposed algorithm in this work, which is a `bounded-distance' version of \ouralgo. It is described in more detail in Section \ref{sec:streambp_star}, and its pseudocode is given in Algorithm \ref{alg:distance-r-sbp}.
    \item \voteone, \votetwo, \votethree: simple plurality voting algorithms
    that give a weight $\delta$ to the side information, as defined in Equation \eqref{eq:Voting} (the numbering corresponds to $\delta\in\{1,2,3\}$). Despite looking somewhat na\"\i ve, voting algorithms are common in industrial applications.
    
\end{itemize}
We also compare the streaming algorithms above with \offlinealgo, which is the {\em offline} belief propagation algorithm, parameterized by its radius (number of parallel iterations) $R$. Note that in contrast to the streaming algorithms, to which we reveal one vertex at a time (along with its side information and the edges connecting it to previously revealed vertices), \offlinealgo has access to {\em the entire graph} and the side information {\em of all the vertices} from the beginning.

Our experiments are based both on synthetic datasets (Section~\ref{synthetic_datasets_main_body}) and on real-world datasets (Section~\ref{real_world_datasets_main_body}). Because the model considered in this paper assumes \emph{undirected} graphs, in a preprocessing step we convert the input graphs of the real-world datasets into undirected graphs by simply ignoring edge directions. Table~\ref{table:datasets_stats} shows statistics of the datasets used (after making the graphs undirected); the values used for $a$ and $b$ for the real-world datasets are discussed in Section~\ref{real_world_datasets_main_body}.

\begin{table}[ht]
\begin{center}
\begin{tabular}{lccccc} \toprule
       \multicolumn{1}{c}{} & \multicolumn{1}{c}{$\pmb{|V|}$} & \multicolumn{1}{c}{$\pmb{|E|}$} & \multicolumn{1}{c}{$\pmb{k}$} & \multicolumn{1}{c}{$\pmb{a}$} & \multicolumn{1}{c}{$\pmb{b}$}\\ \midrule
    Citeseer & {3,264} & {4,536} & 6 & 11.47 & 0.89 \\
    Cora  & {2,708} & {5,278} & 7 & 17.62 & 0.90 \\
    Polblogs & {1,490} & {16,715} & 2 & 40.69 & 4.23 \\
    Synthetic & [10,000--50,000] & [20,000--700,000] & [2--5] & [2.5--18] & [0.05--1] \\ \bottomrule
\end{tabular}\vspace{2mm}
\caption{Statistics of the synthetic and real-world datasets. For the synthetic datasets, various experiments with a range of parameters are performed.}
\label{table:datasets_stats}
\end{center}
\end{table}

As the measure of accuracy, we use the empirical fraction of labels correctly recovered by each algorithm, that is
$$ 
{\rm Acc } =\hat{\E}\hskip-1mm\left[\max_{\pi\in \fS_k}\frac{1}{n} \hskip-1mm \sum_{v \in V(n)} \hskip-2mm \ind\left(\cA(v; G(n), \ttau) = \pi\circ\tau(v)\right) \right]\hskip-1mm.
$$
In synthetic data we observe that, as soon as $\alpha<(k-1)/k$,
the maximization over $\pi$ is not necessary (as expected from theory), and hence we drop it.
%

\subsection{Bounded-distance streaming BP}
\label{sec:streambp_star}

In our experimental results presented below, 
we observed that the simple implementation in Algorithm~\ref{alg:streaming-bp-with-side-information}
exhibits undesirable behaviors for certain graphs.
We believe this is caused by two factors. First, unlike \offlinealgo, we do not have an upper bound on the radius of influence of each vertex in \ouralgo; it may indeed use long paths. Second, the number of cycles in the graph increases as $a, b$ grow. Large values of $a$ and $b$ can result in many paths being cycles, which negatively affects the performance of the algorithm.

In order to overcome these problems, we use two modifications. The first one is 
standard: we constrain messages so that $m_{u\to v}(s)\in [\varepsilon,1-\varepsilon]$ for some fixed small $\varepsilon>0$ 
(we essentially constrain the log-likelihood ratios to be bounded).

The second modification 
defines a variant,
presented in Algorithm~\ref{alg:distance-r-sbp}, which we call \boundedalgo.
Here
the estimate at node $v$ is guaranteed to depend only on the graph structure and side information within
$\Ball_R^n(v)$, and not on the information outside this ball. This constraint can be implemented in a message-passing
fashion, by keeping, on each edge, $R + 1$ distinct messages
$\bm^0_{u\to v},\dots,\bm^{R}_{u\to v}$,  corresponding to different locality radii.

\begin{breakablealgorithm}
    \caption{\label{alg:distance-r-sbp} \boundedalgo: Bounded-distance streaming BP}
	\begin{algorithmic}[1]
		\State Initialization: ${V}(0) = {E}(0) = {G}(0) = \emptyset$.
		\For{$t = 1,2,\dots, n$}
   		   \State \textit{// Update the graph:}
   		   \State ${V}(t) \gets {V}(t - 1) \cup \{v(t)\}$
			\State 
			$E(t) \gets E(t - 1) \cup \{(v(t), v): v \in V(t - 1),  (v(t), v) \in E(n)\}$
    	   \State ${G}(t) \gets ({V}(t), {E}(t))$  
    	   \State \textit{// Update the incoming messages:}
    	   \For{$v \in D_1^t(v(t))$} 
    	   		\State $\bm_{v \rightarrow v(t)}^0 \gets \bfone/k$
    	   		
    	   		\For{$i = 1,2,\dots, R$}
    	   			\State $\bm^i_{v \rightarrow v(t)} \gets \BP(\{\bm_{v' \rightarrow v}^{i - 1}\}_{v'\in\partial v\setminus\{v(t)\}};\ttau(v))$
    	   			
    	   		\EndFor
    	   \EndFor
    	   \State \textit{// Update the outgoing messages:}
    	   \For{$v \in D_1^t(v(t))$} 
    	   		\State $\bm_{v(t) \rightarrow v}^0 \gets \bfone/k$
    	   		\For{$i = 1,2,\dots, R$}
    	   			\State  $\bm_{v(t) \rightarrow v}^i  \gets \BP(\{\bm_{v' \rightarrow v(t)}^{i - 1}\}_{v'\in\partial v(t)\setminus\{v\}};\ttau(v(t)))$
				
    	   		\EndFor
    	   \EndFor
    	   \For{$r = 2,3,\dots, R$}
     		   \For{$v \in {D}_r^t(v(t))$}
       			   \State Let $v' \in {D}_1^t(v)$ on a shortest path connecting $v$ and $v(t)$
       			   \For{$i = 1,2,\dots, R$}
       			   		\State $\bm_{v' \rightarrow v}^i \gets  \BP(\{\bm_{u \rightarrow v'}^{i - 1}\}_{u\in\partial v'\setminus\{v\}};\ttau(v'))$
       			   \EndFor
    		   \EndFor
    	   \EndFor
    	\EndFor
    	\State \textit{// Compute the estimates:}
    	\For{$u \in V$}
			\State $\bm_u \gets  \BP(\{\bm_{v \rightarrow u}^{ R}\}_{v\in\partial u};\ttau(u))$
			\State Output $\htau(u) := \arg\max_s m_u(s)$ 
		\EndFor
	\end{algorithmic}
\end{breakablealgorithm}

\subsection{Synthetic datasets}
\label{synthetic_datasets_main_body}
Figure \ref{fig:allmain} illustrates the effect of the radius on the performance of the algorithms.%
\footnote{Notice that the three voting algorithms do not depend on the radius, resulting in horizontal lines in the diagram. }
We use various settings for $k, a, b, \alpha$.  We observe that voting algorithms do not
perform significantly better than the baseline $1-\alpha$ (dashed line).
This is due both to the very small radius $R=1$ of these 
algorithms, and to the specific choice of the update rule. For 
$R=1$, \ouralgo and \boundedalgo perform significantly better than voting,
showing that their update rule is preferable. Their accuracy improves
with $R$, and is often close to the optimal accuracy (i.e. the accuracy 
of \offlinealgo for large $R$) already for $R\approx 5$. 

In Figure \ref{fig:constmain}, we study the effect of the SNR parameter $\lambda$, defined in Equation \eqref{eq:KS}, on the performance of the BP algorithms.
The accuracy of the algorithms improves as $\lambda$ increases.
It is close to the baseline $1-\alpha$ when the SNR is close to
the KS threshold at $\lambda=1$, and then it improves for large $\lambda$.
This is a trace of the phase transition at the KS threshold which is blurred because of side information.
 For large values of $R$, the accuracy of our streaming algorithms \ouralgo and \boundedalgo nearly matches that of the optimal {\em offline} algorithm \offlinealgo.

Further experiments on synthetic datasets are reported in Appendix~\ref{sec:synthetic}. We then present in Appendix~\ref{sec:nosideinfo} the result of experiments where no side information is provided to the algorithms. We observe that in the streaming setting, and for small radius ($R$), neither \ouralgo nor \boundedalgo achieves high accuracy (above $1/k$) in the absence of side information, as suggested by our theoretical results.

\subsection{Real-world datasets}
\label{real_world_datasets_main_body}

We further investigate the performance of our algorithms  on three real-world datasets: Cora~\cite{nr}, Citeseer~\cite{nr}, and Polblogs~\cite{DBLP:conf/kdd/AdamicG05}. Cora and Citeseer are academic citation networks: vertices represent scientific publications partitioned into $k = 7$ and $k = 6$ classes respectively; directed edges represent citations of a publication by another. The Polblogs dataset represents the structure of the political blogosphere in the United States around its 2004 presidential election: vertices correspond to political blogs, and directed edges represent the existence of hyperlinks from one blog to another. The blogs are partitioned into $k = 2$ classes based on their political orientation (liberal or conservative).

As mentioned in the beginning of Section~\ref{sec:empirical}, in a preprocessing step we convert the input graphs of the real-world datasets into undirected graphs by simply ignoring edge directions. Also, since the graphs do not stem from the models of Section \ref{sec:model}, 
we use oracle estimates of parameters $a$ and $b$, obtained by 
matching the density of intra- and inter-community edges
with those in the model $\STSSBM(n, k, \ta, \tb, \alpha)$.
Namely, for a graph $G = (V, E)$, letting $V_1, \dots, V_k \subseteq V$ be the ground-truth communities, we set
\begin{align*}
\ta = |V| \cdot \frac{ \sum_{i \in [k]} \sum_{(u, v) \in E} 
\ind {\{u, v \in V_i\}}}{\sum_{i \in [k]} \binom{|V_i|}{2}},
%
\quad
\tb = |V| \cdot \frac{\sum_{\{i, j\} \in \binom{[k]}{2}} \sum_{(u, v) \in E} 
\ind {\{u \in V_i, v \in V_j\}}}{\sum_{\{i, j\} \in \binom{[k]}{2}} |V_i||V_j|}.
\end{align*}
For each dataset, we run the streaming algorithms \ouralgo and \boundedalgo 
using the parameters $a = \ta$ and $b = \tb$. 
Notice that these estimates cannot be implemented in practice because we do not know the 
communities $V_j$ to start with. 
However, the performances appear not to be too sensitive to these estimates;
we provide empirical evidence of this in  Appendix~\ref{sec:synthetic}.


Figure \ref{fig:real_world_datasets} shows the accuracy 
of different algorithms
on these datasets, for selected values of $\alpha$.
Although the graphs in these datasets are not random graphs generated according to \STSBM,
the empirical results generally align with the theoretical results we proved for that model. Specifically, we see that our streaming algorithm \boundedalgo is approximately as accurate as the {\em offline} algorithm \offlinealgo, and significantly better than the voting algorithms. 
\ouralgo produces high-quality results for Cora and Citeseer datasets, but behaves erratically on the Polblogs dataset. That is likely due to the issues discussed in Section~\ref{sec:streambp_star}.

Appendix~\ref{sec:realworld} provides additional details on these experiments, as well as results for other values of $\alpha$.

\begin{figure*}[ht]
\begin{subfigure}{.4\textwidth}
  \centering
\includegraphics[scale=.17]{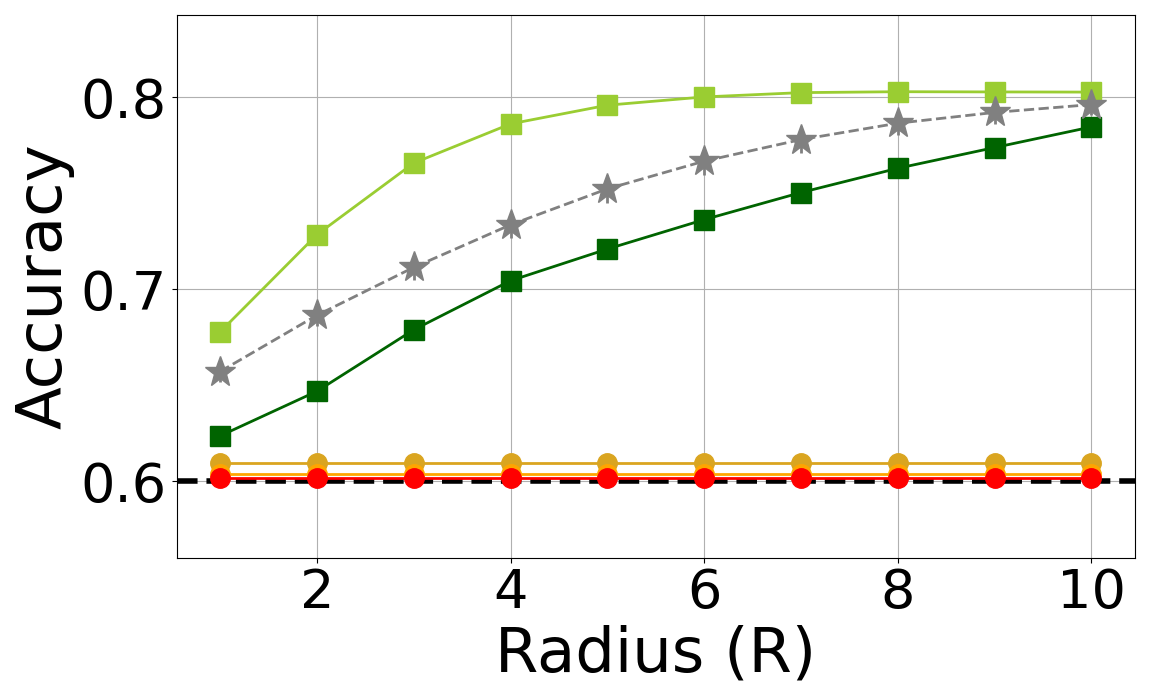} 
\caption{$k=2$, $a=3$, $b=.1$, $\alpha=.4$.}
\end{subfigure}
\begin{subfigure}{.6\textwidth}
  \centering
\includegraphics[scale=.17]{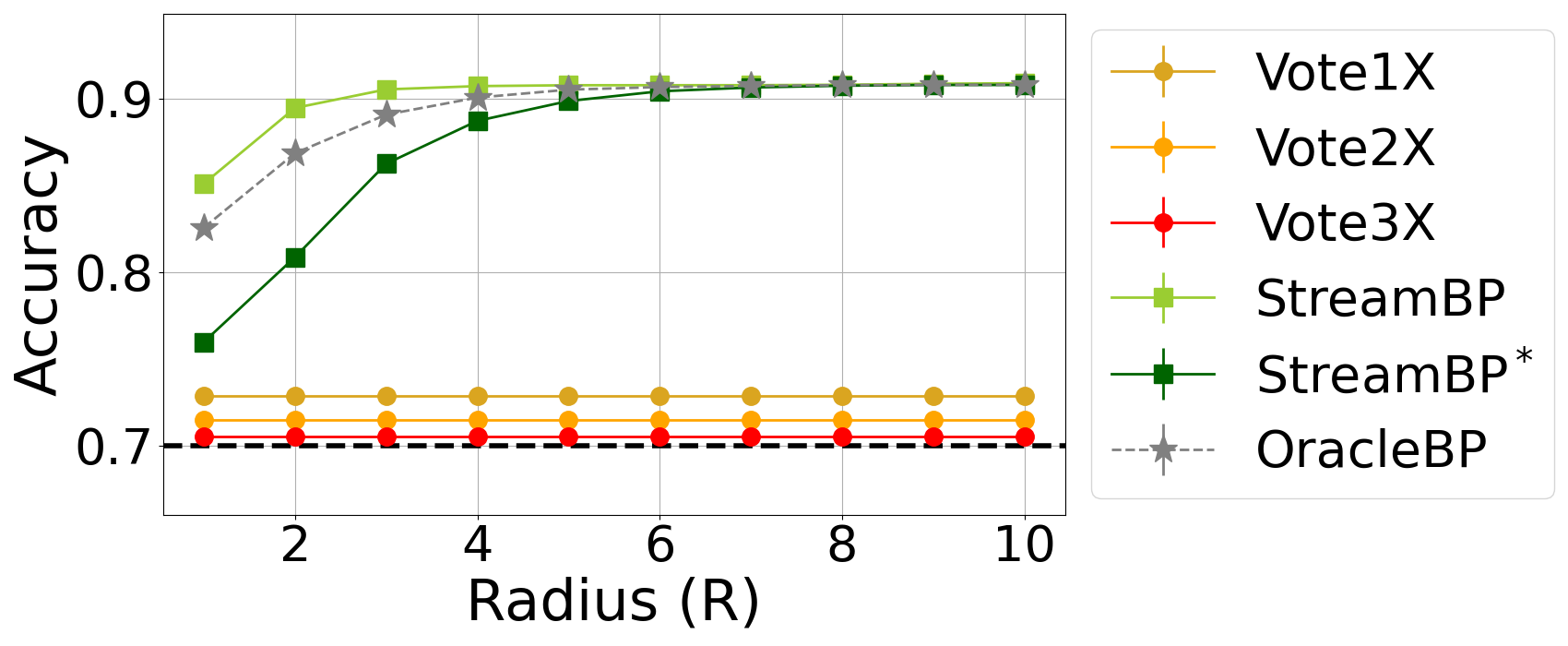} 
\caption{$k=2$, $a=5$, $b=.5$, $\alpha=.3$.}
\end{subfigure}\\

\begin{subfigure}{.33\textwidth}
  \centering
\includegraphics[scale=.155]{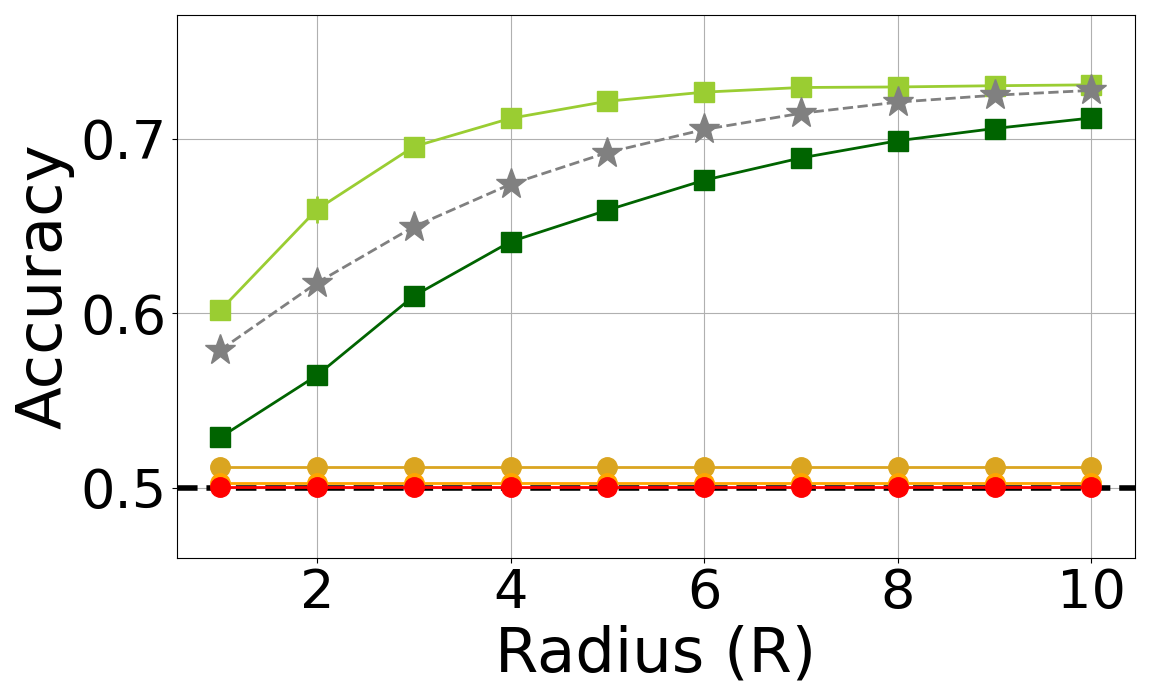} 
\caption{$k=3$, $a=4$, $b=.05$, $\alpha=.5$.}
\end{subfigure}
\begin{subfigure}{.33\textwidth}
  \centering
\includegraphics[scale=.155]{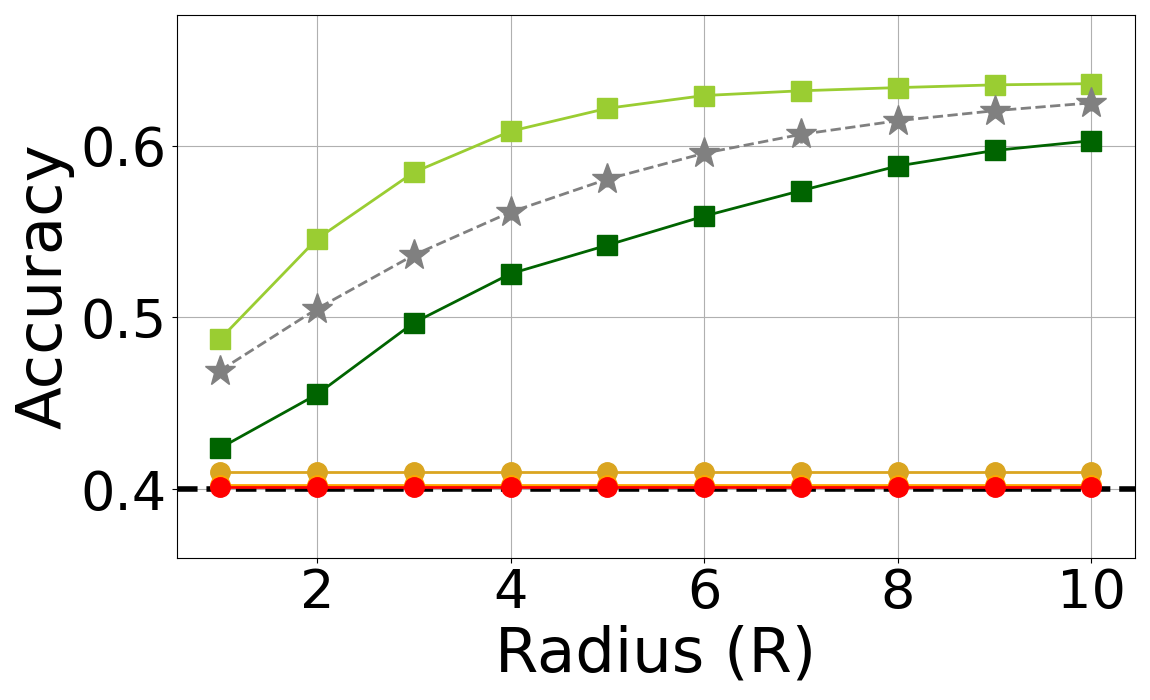} 
\caption{$k=4$, $a=3$, $b=.1$, $\alpha=.6$.}
\end{subfigure}
\begin{subfigure}{.33\textwidth}
  \centering
\includegraphics[scale=.155]{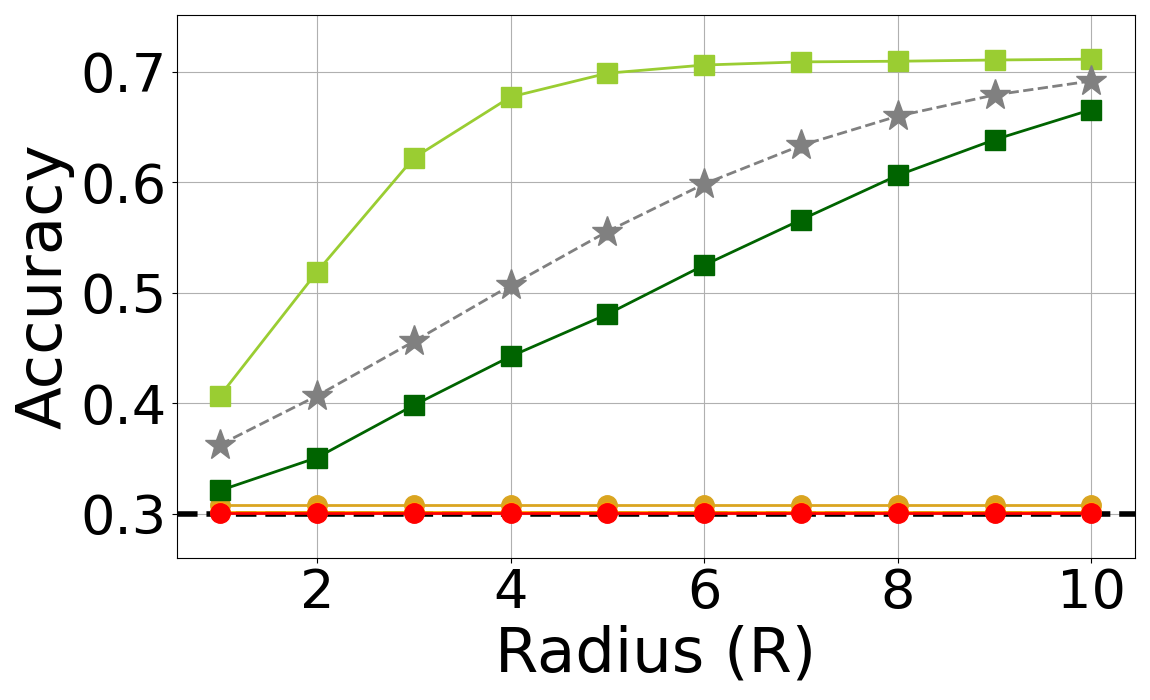} 
\caption{$k=5$, $a=8$, $b=.1$, $\alpha=.7$.}
\end{subfigure}\\
\caption{\label{fig:allmain}Results of the experiments on synthetic datasets with 50,000 vertices. The black dashed line represents the accuracy of the noisy side information (without using the graph at all), namely $1-\alpha$.}
\end{figure*}

\begin{figure*}[ht]
\begin{subfigure}{.33\textwidth}
  \centering
\includegraphics[scale=.15]{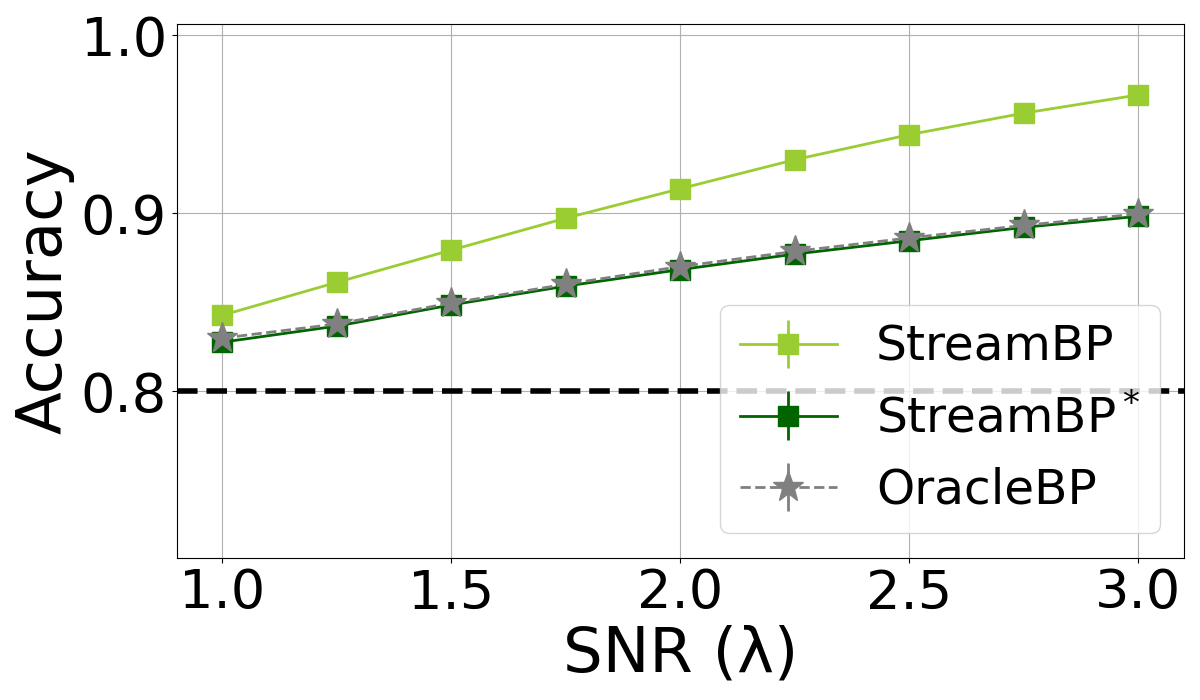}
\caption{$R = 1$}
\end{subfigure}
\begin{subfigure}{.33\textwidth}
  \centering
\includegraphics[scale=.15]{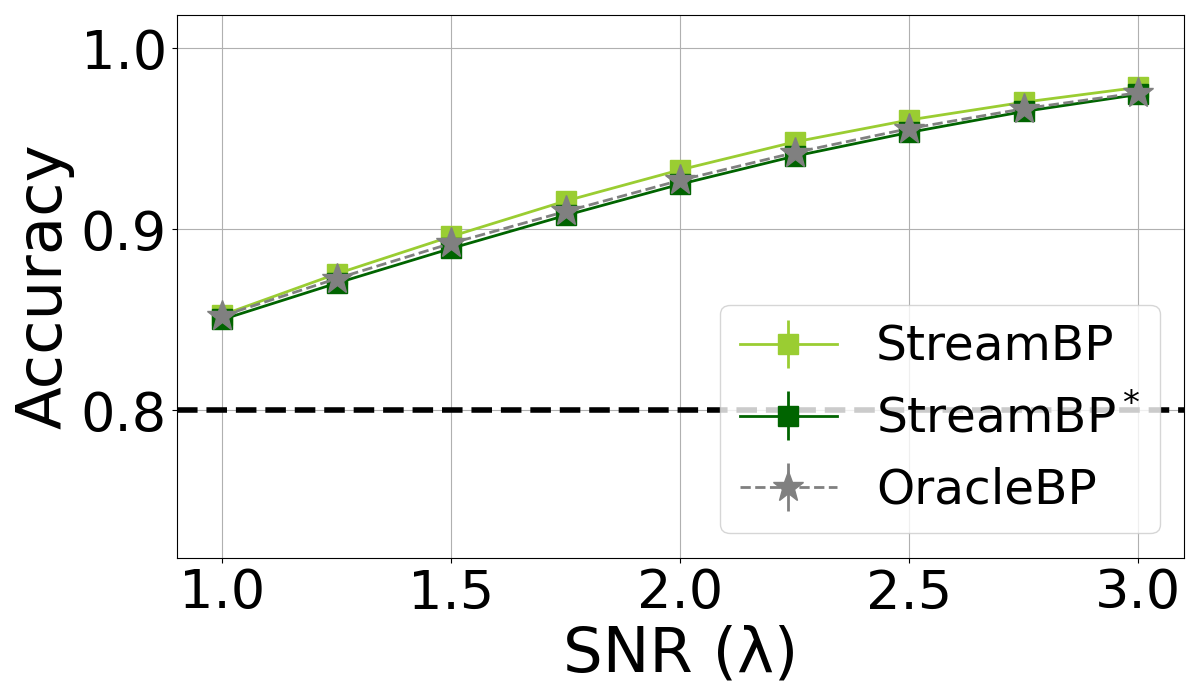} 
\caption{$R = 3$}
\end{subfigure}
\begin{subfigure}{.33\textwidth}
  \centering
\includegraphics[scale=.15]{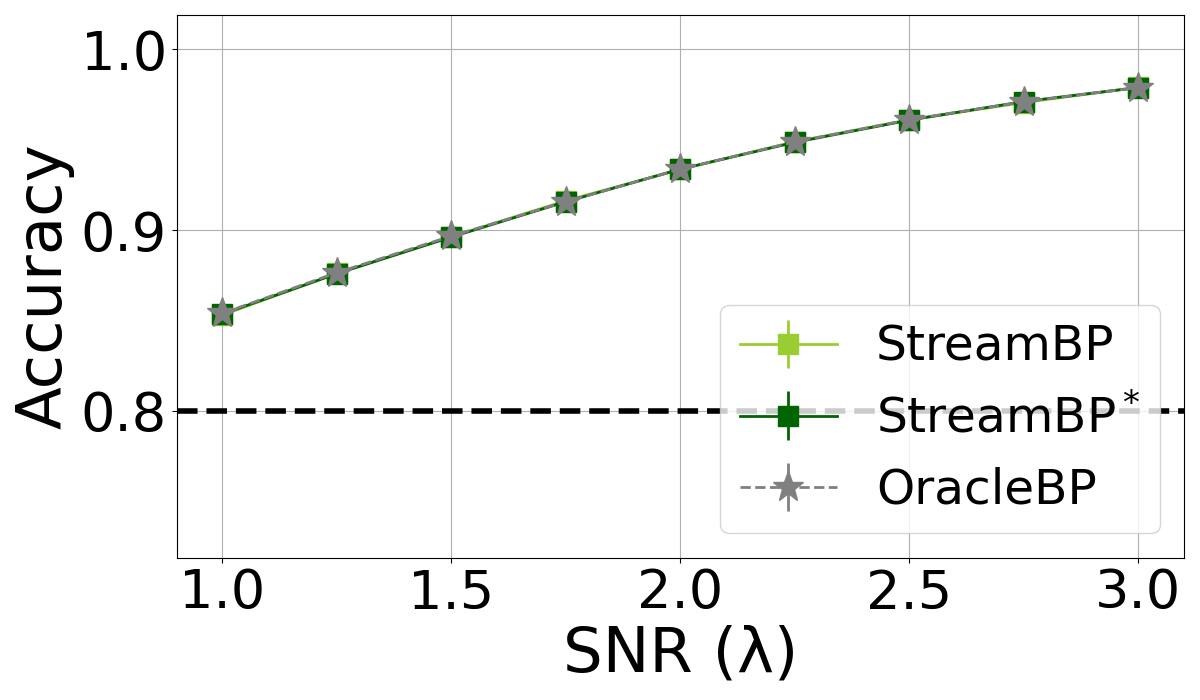} 
\caption{$R = 5$}
\end{subfigure}\\
\caption{\label{fig:constmain}Results of the experiments on synthetic datasets with 50,000 vertices, with $k = 2, a+b = 8, \alpha=0.2$,
as we increase the ratio $\lambda$ from 1.0 to 3.0. The black dashed line represents the accuracy of the noisy side information (without using the graph at all), namely $1-\alpha$.}
\end{figure*}
\begin{figure*}[ht]
\centering
\begin{subfigure}{.28\textwidth}
  \centering
   \includegraphics[scale=.14]{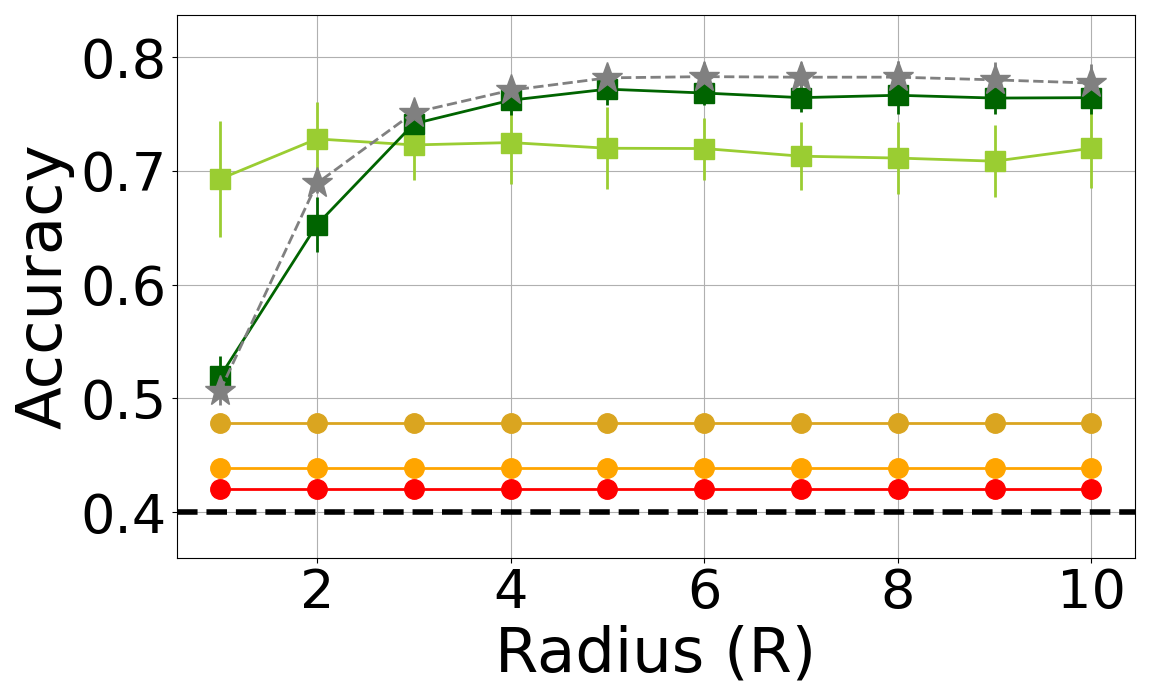} 
\caption{Cora ($k = 7$), $\alpha = 0.6$.}
\end{subfigure}
\begin{subfigure}{.28\textwidth}
  \centering
    \includegraphics[scale=.14]{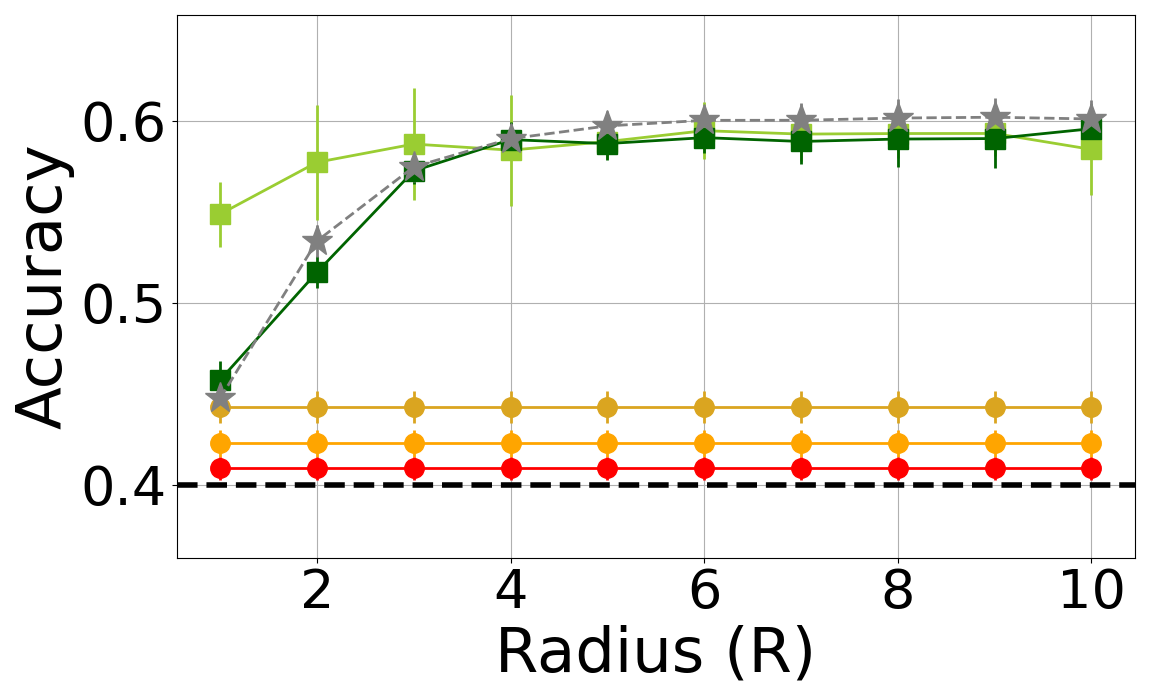} 
\caption{Citeseer ($k = 6$), $\alpha = 0.6$.}
\end{subfigure}
\begin{subfigure}{.42\textwidth}
  \centering
\includegraphics[scale=.14]{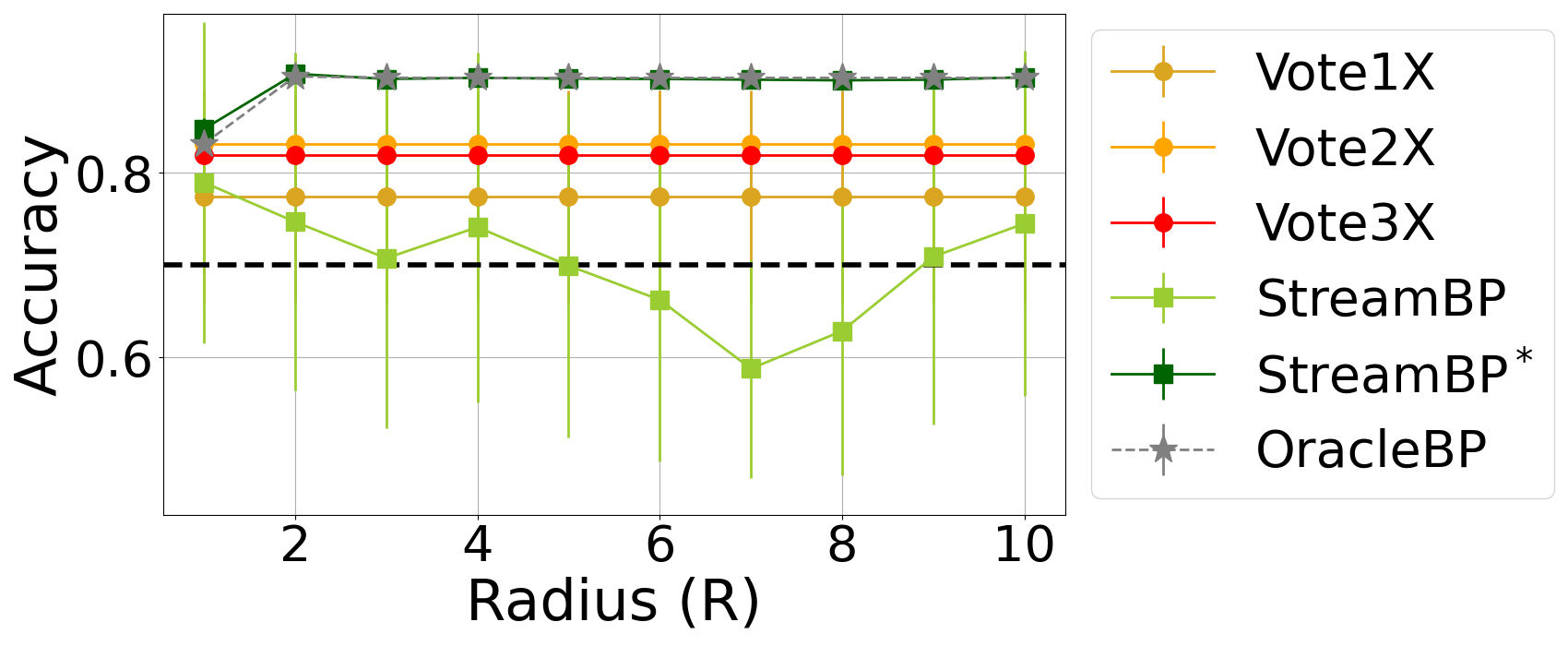} 
\caption{Polblogs ($k = 2$), $\alpha = 0.3$.}
\end{subfigure}\\
\caption{\label{fig:real_world_datasets}Results of  experiments on real-world datasets. The black dashed line represents the accuracy of the noisy side information (without using the graph at all), namely $1-\alpha$.}
\end{figure*}

\section*{Acknowledgements}

A.M. and Y.W. were partially supported by NSF grants CCF-2006489, IIS-1741162 and the ONR grant N00014-18-1-2729.

\bibliographystyle{alpha}

\newpage
\newcommand{\etalchar}[1]{$^{#1}$}


\newpage
\begin{appendices}
	\section{Further experiments with synthetic datasets}
\label{sec:synthetic}
\paragraph{Results with different sets of parameters.} In Figures~\ref{fig:allapp2}--\ref{fig:allapp5}, we report more comprehensively the dependence of our algorithms' performance on the radius, where each figure corresponds to one choice of $k = 2, 3, 4, 5$.
Similarly to Figure~\ref{fig:allmain}, we can see that as the radius increases, the performance of the three belief propagation algorithms (\ouralgo, \boundedalgo, and \offlinealgo) improves and converges to the same value. For some parameter settings, we observe that \ouralgo performs poorly when the radius is above a certain threshold; that is likely due to the issues discussed in Section~\ref{sec:streambp_star}.

\begin{figure*}[ht]
\begin{subfigure}{.49\textwidth}
  \centering
   \includegraphics[scale=.16]{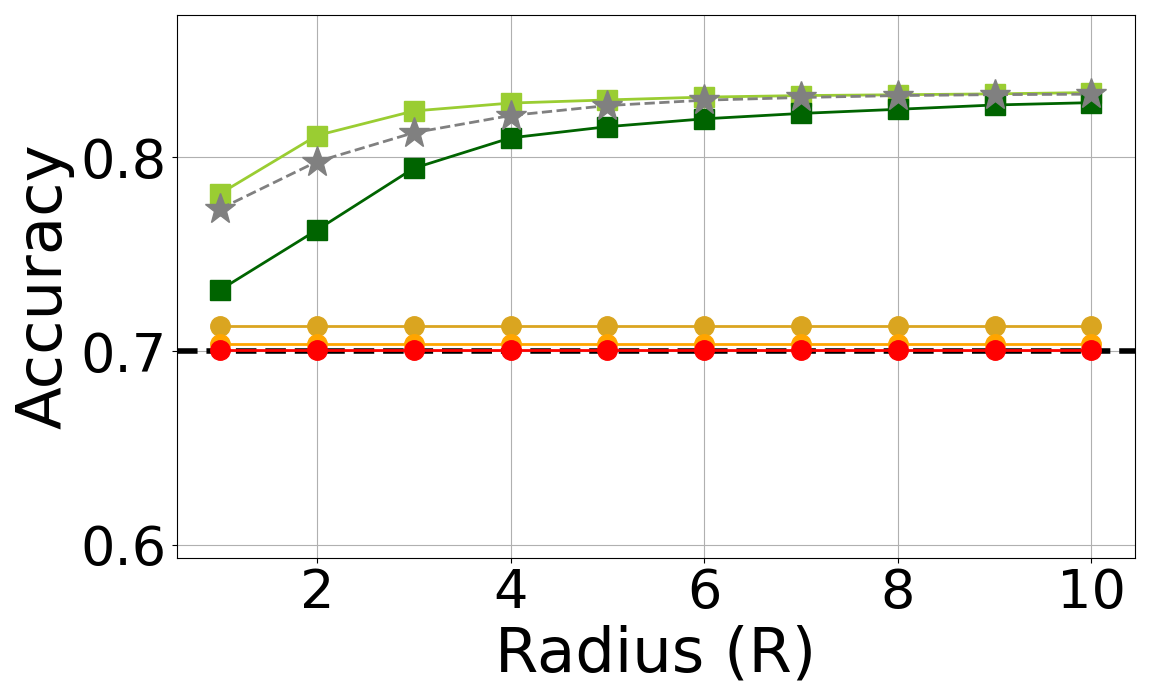} 
\caption{$a=2.5$, $b=0.05$, $\alpha=0.3$.}
\end{subfigure}
\begin{subfigure}{.49\textwidth}
  \centering
    \includegraphics[scale=.16]{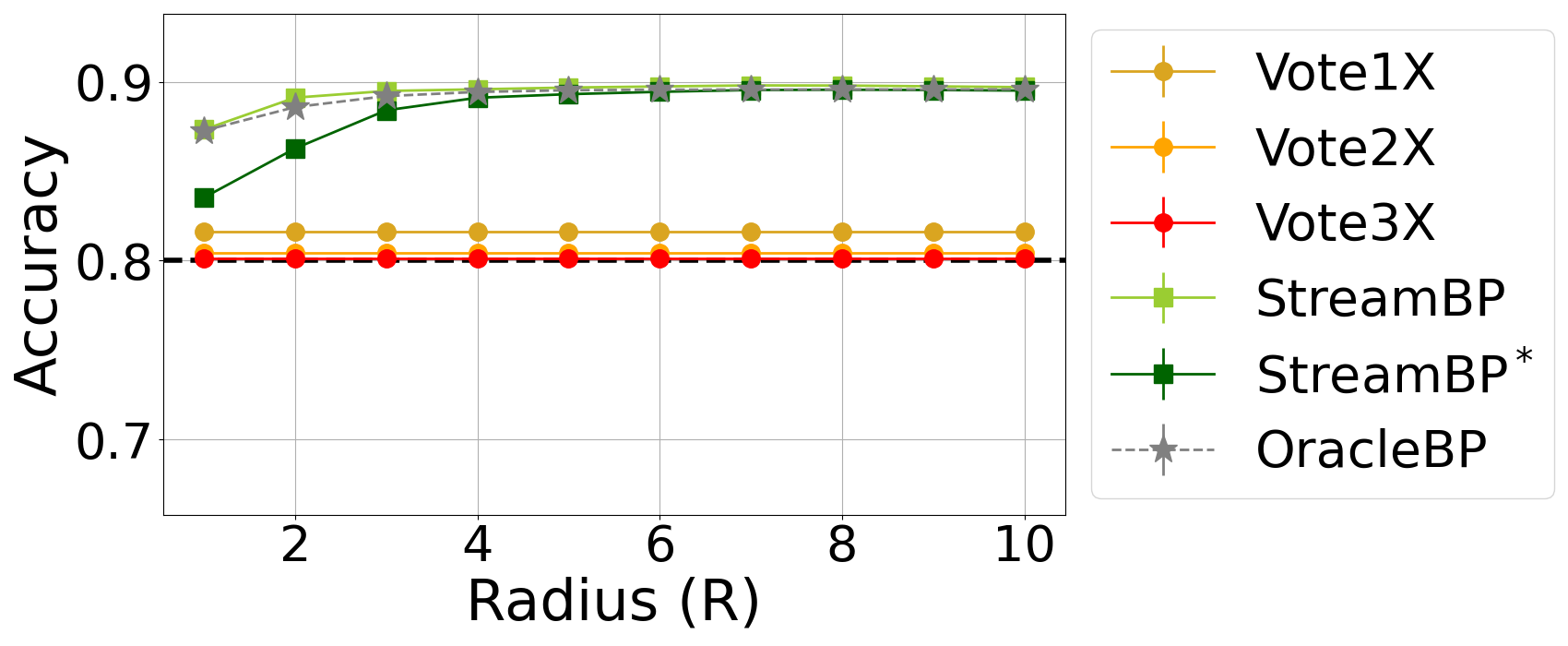} 
\caption{$a=2.5$, $b=0.05$, $\alpha=0.2$.}
\end{subfigure}\\

\begin{subfigure}{.33\textwidth}
  \centering
\includegraphics[scale=.16]{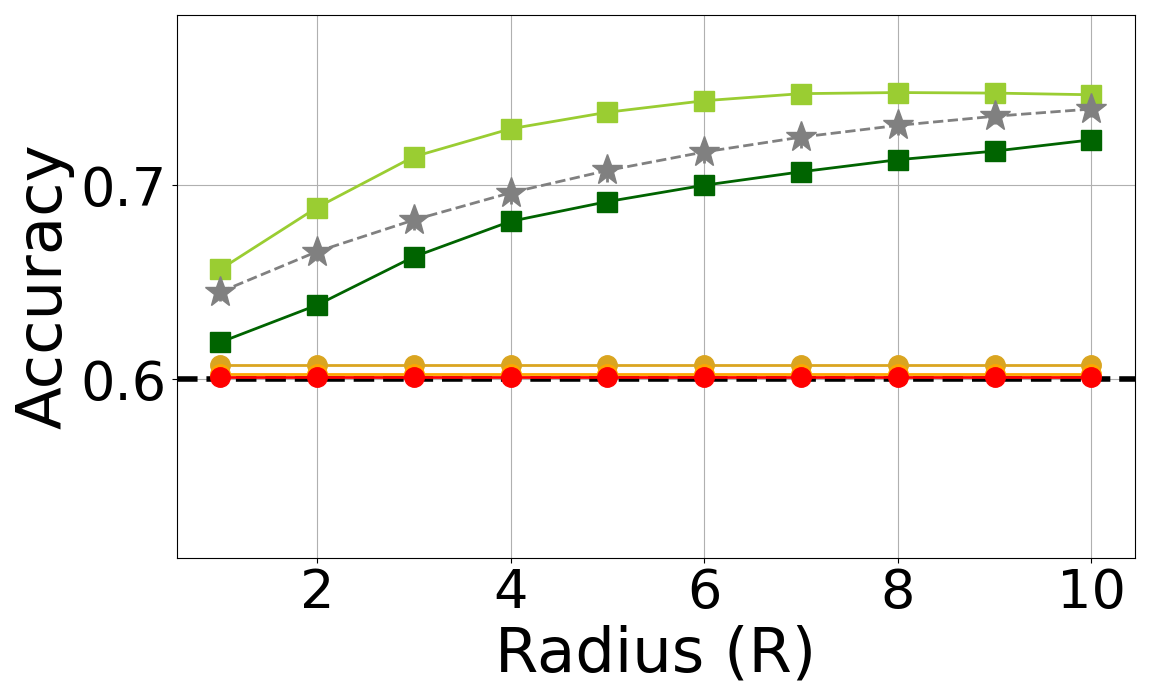} 
\caption{$a=2.5$, $b=0.05$, $\alpha=0.4$.}
\end{subfigure}
\begin{subfigure}{.33\textwidth}
  \centering
\includegraphics[scale=.16]{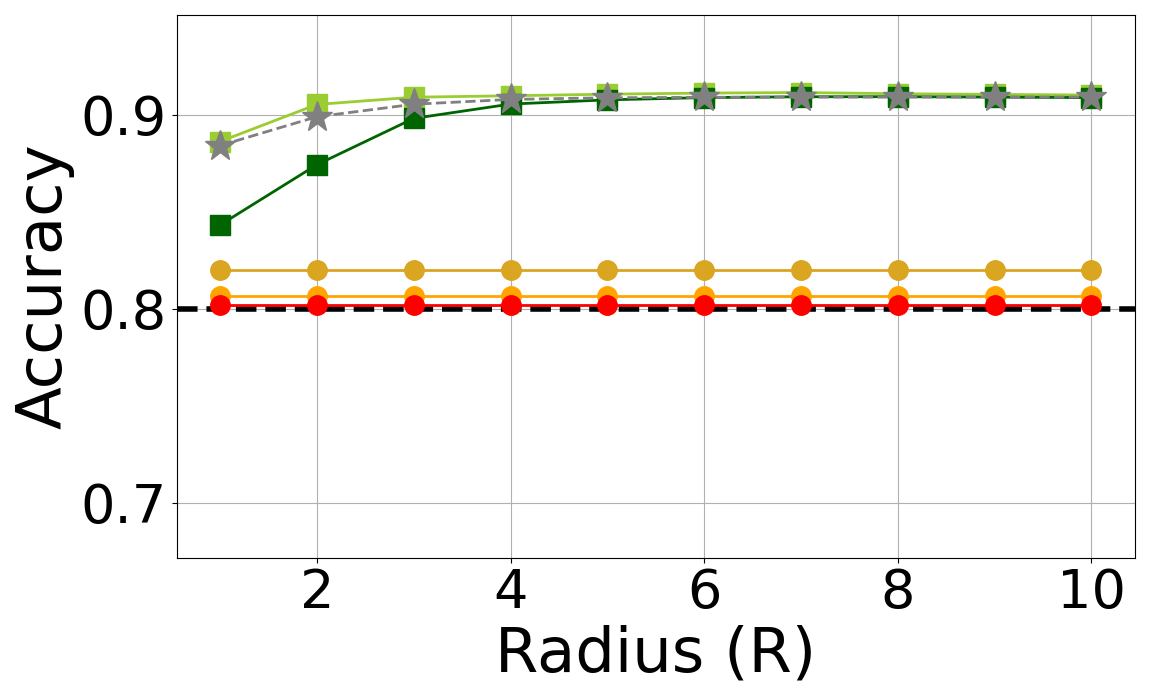} 
\caption{$a=3$, $b=0.1$, $\alpha=0.2$.}
\end{subfigure}
\begin{subfigure}{.33\textwidth}
  \centering
\includegraphics[scale=.16]{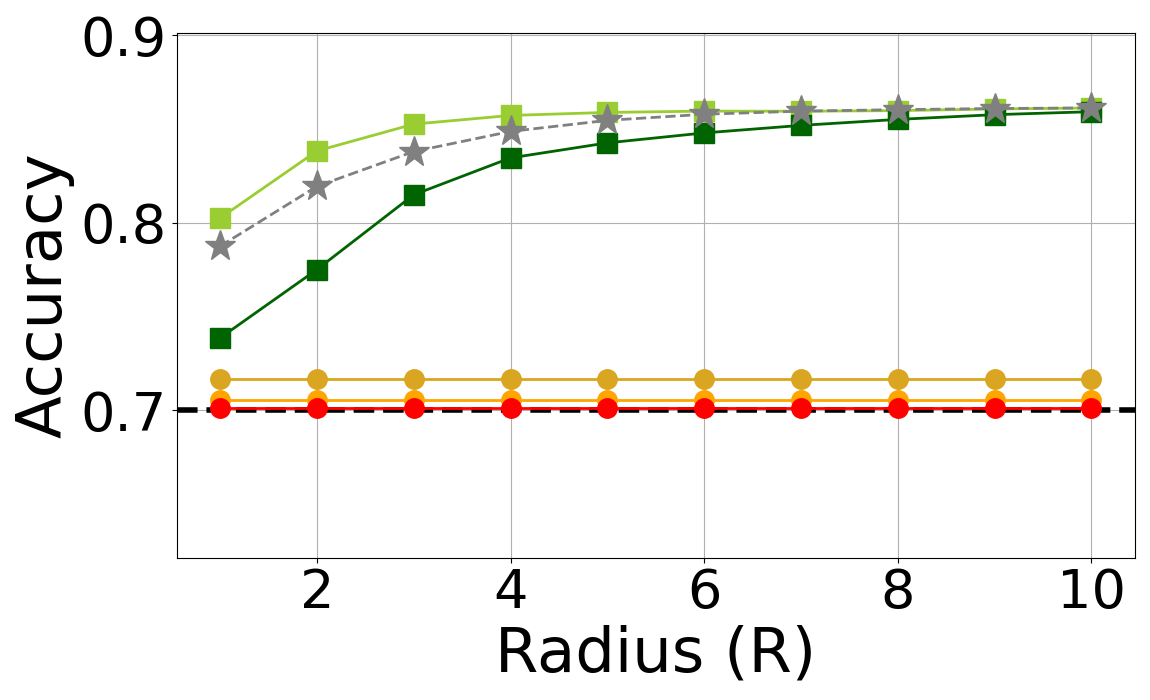} 
\caption{$a=3$, $b=0.1$, $\alpha=0.3$.}
\end{subfigure}\\
\begin{subfigure}{.33\textwidth}
  \centering
\includegraphics[scale=.16]{Figures/AllAlgos/without/allalgo_50000_2_3.000000_0.100000_0.400000.png} 
\caption{$a=3$, $b=0.1$, $\alpha=0.4$.}
\end{subfigure}
\begin{subfigure}{.33\textwidth}
  \centering
\includegraphics[scale=.16]{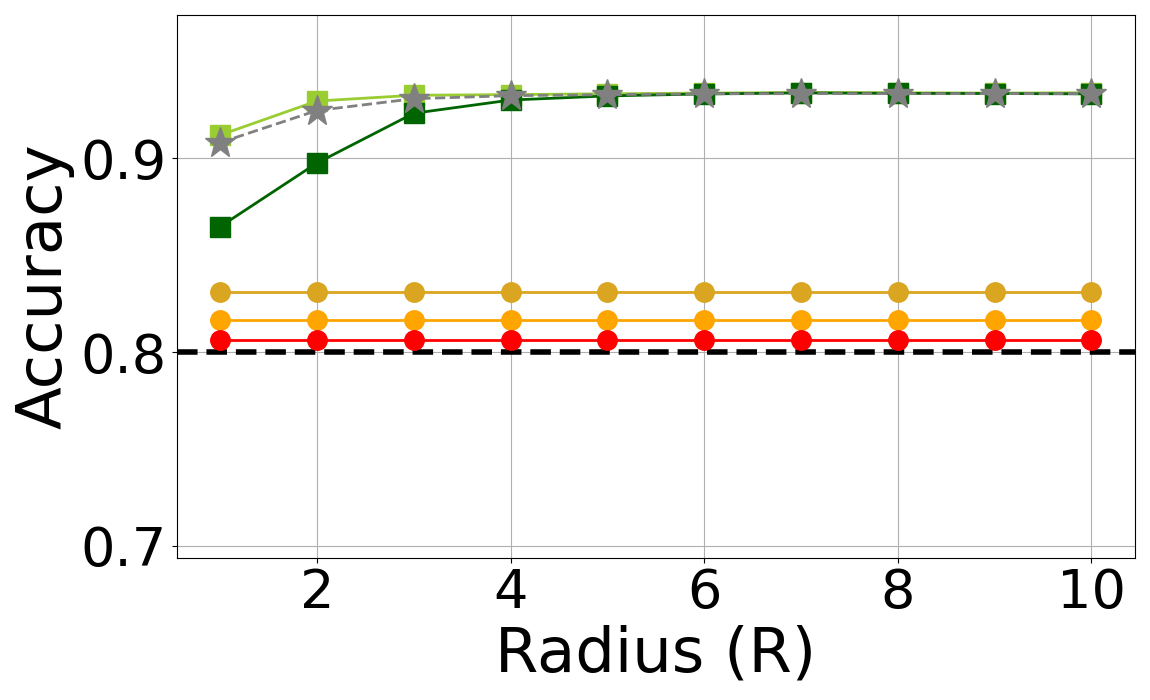} 
\caption{$a=5$, $b=0.5$, $\alpha=0.2$.}
\end{subfigure}
\begin{subfigure}{.33\columnwidth}
  \centering
\includegraphics[scale=.16]{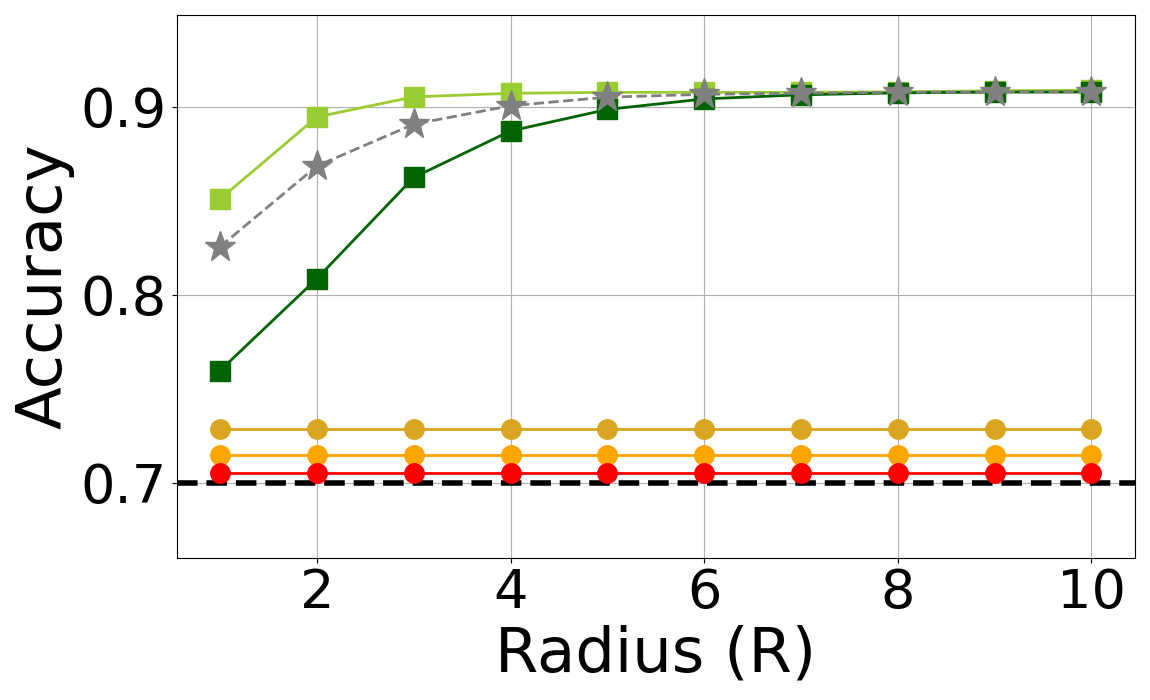} 
\caption{$a=5$, $b=0.5$, $\alpha=0.3$.}
\end{subfigure}\\
\caption{\label{fig:allapp2}Results of the experiments on synthetic datasets with 50,000 vertices, and $k = 2$. The black dashed line represents the accuracy of the noisy side information (without using the graph at all), namely $1-\alpha$.}
\end{figure*}

\begin{figure*}[ht]
\begin{subfigure}{.49\textwidth}
  \centering
   \includegraphics[scale=.16]{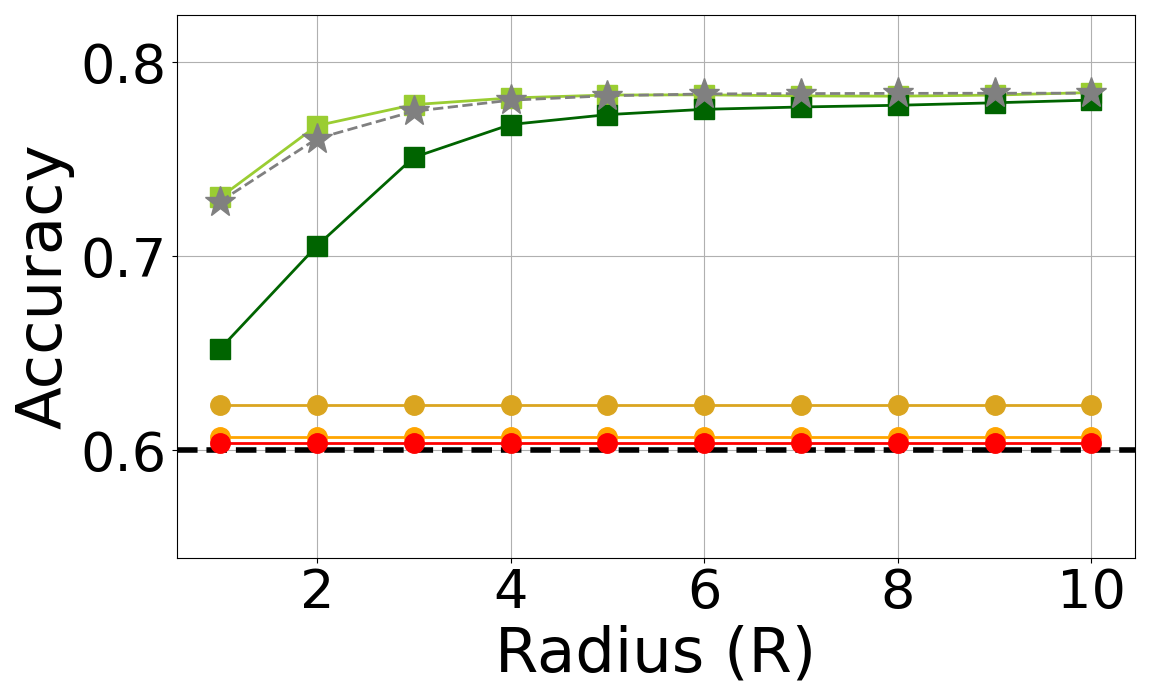} 
\caption{$a=6$, $b=0.05$, $\alpha=0.4$.}
\end{subfigure}
\begin{subfigure}{.49\textwidth}
  \centering
    \includegraphics[scale=.16]{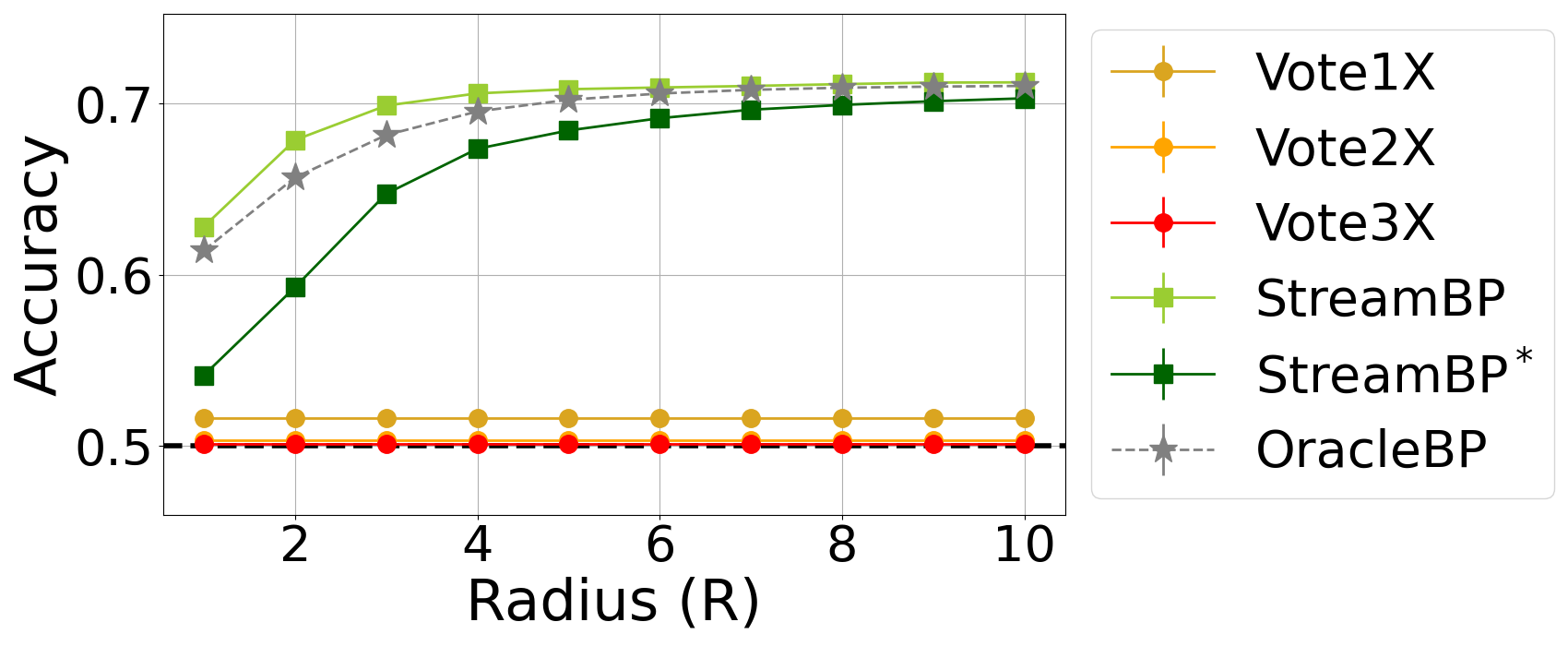} 
\caption{$a=6$, $b=0.05$, $\alpha=0.5$.}
\end{subfigure}\\

\begin{subfigure}{.33\textwidth}
  \centering
\includegraphics[scale=.16]{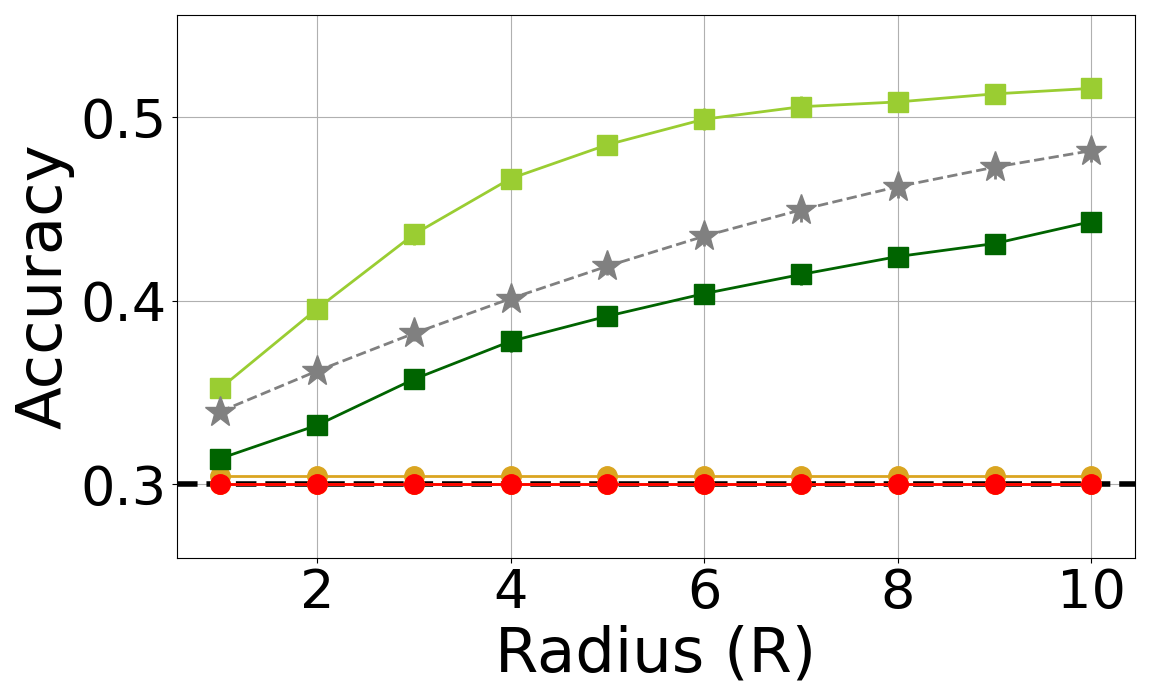} 
\caption{$a=6$, $b=0.05$, $\alpha=0.7$.}
\end{subfigure}
\begin{subfigure}{.33\textwidth}
  \centering
\includegraphics[scale=.16]{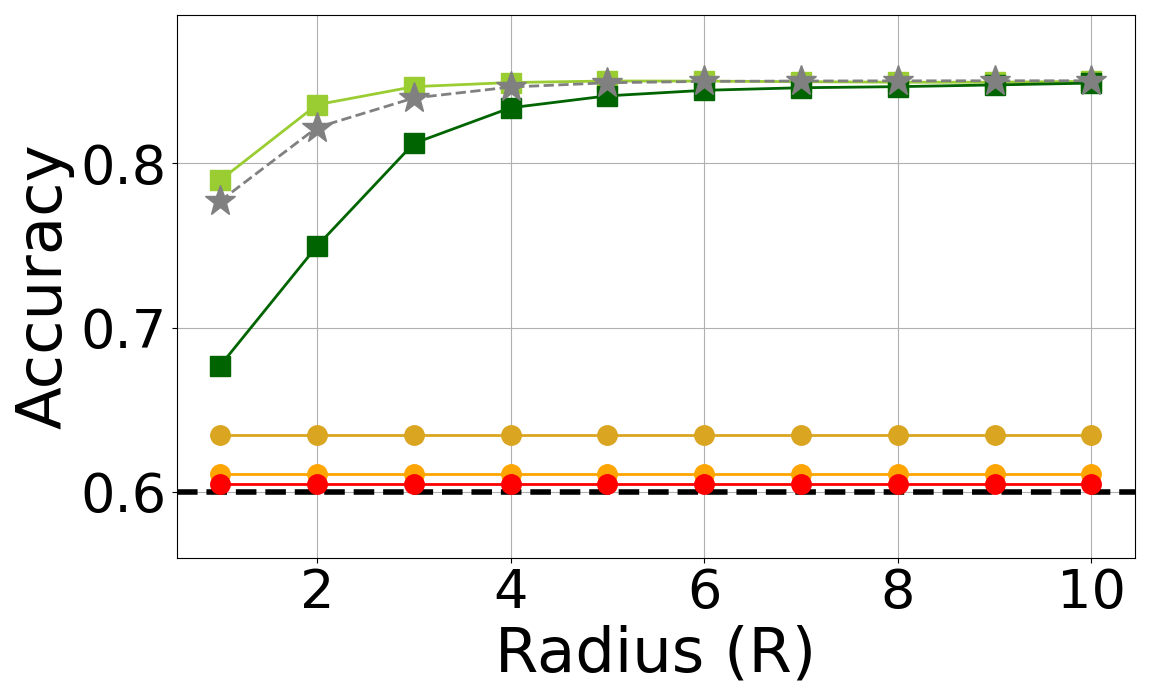} 
\caption{$a=8$, $b=0.1$, $\alpha=0.4$.}
\end{subfigure}
\begin{subfigure}{.33\textwidth}
  \centering
\includegraphics[scale=.16]{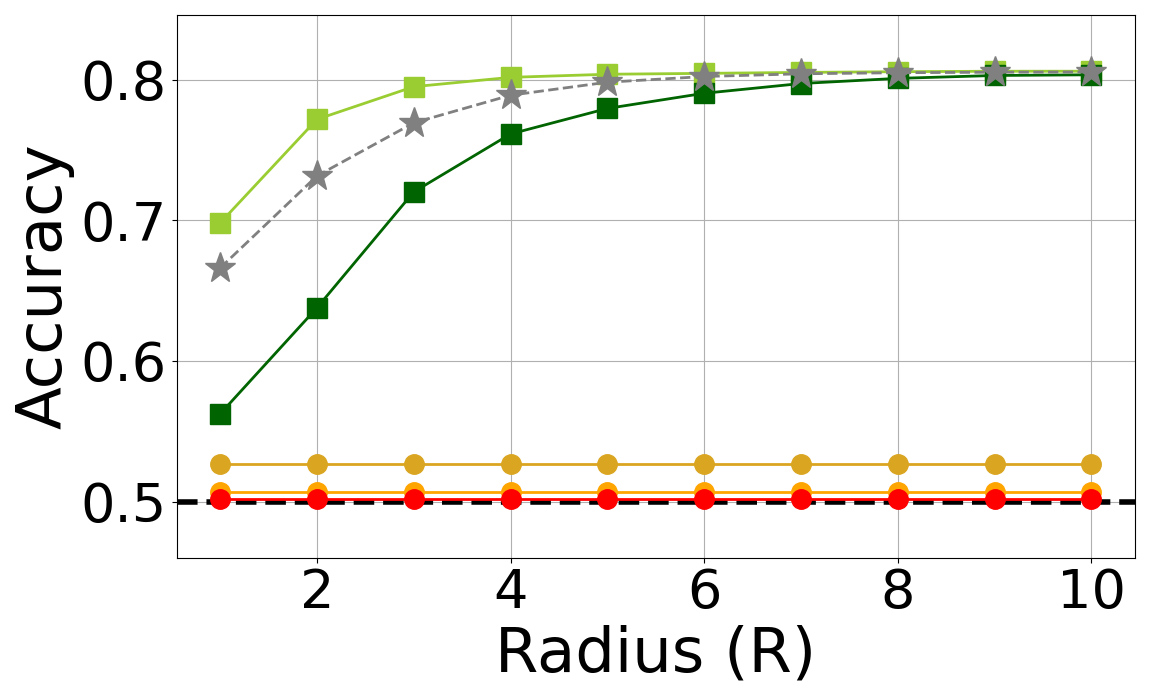} 
\caption{$a=8$, $b=0.1$, $\alpha=0.5$.}
\end{subfigure}\\
\begin{subfigure}{.33\textwidth}
  \centering
\includegraphics[scale=.16]{Figures/AllAlgos/without/allalgo_50000_5_8.000000_0.100000_0.700000.png} 
\caption{$a=8$, $b=0.1$, $\alpha=0.7$.}
\end{subfigure}
\begin{subfigure}{.33\textwidth}
  \centering
\includegraphics[scale=.16]{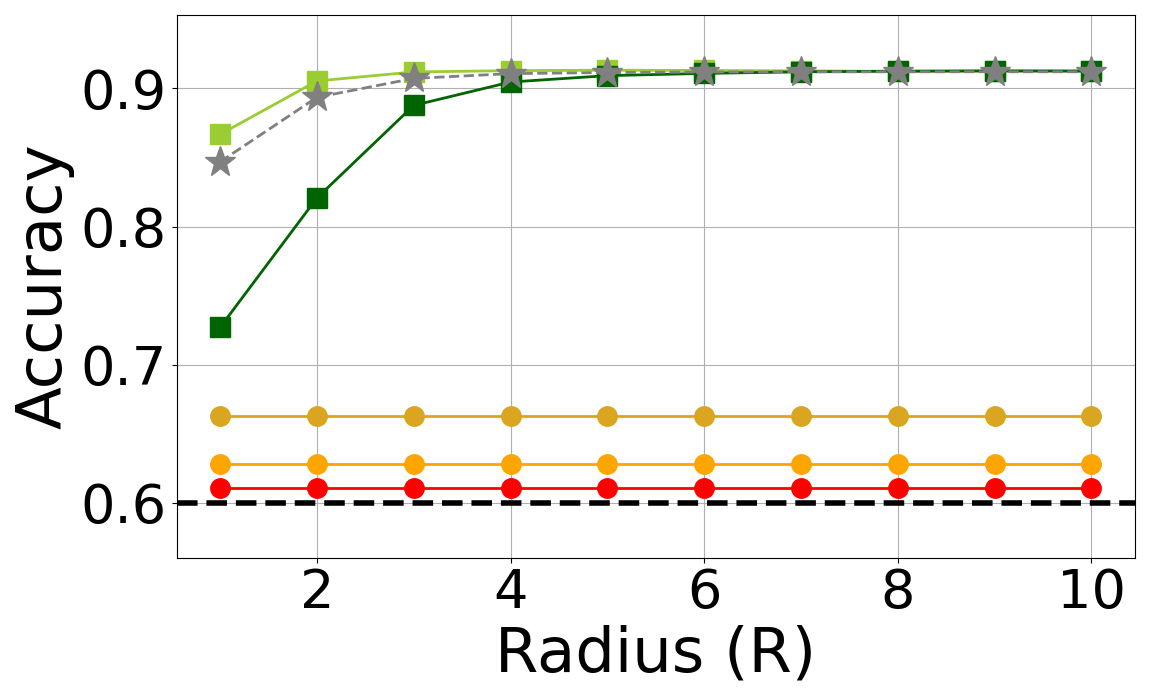} 
\caption{$a=13$, $b=0.5$, $\alpha=0.4$.}
\end{subfigure}
\begin{subfigure}{.33\textwidth}
  \centering
\includegraphics[scale=.16]{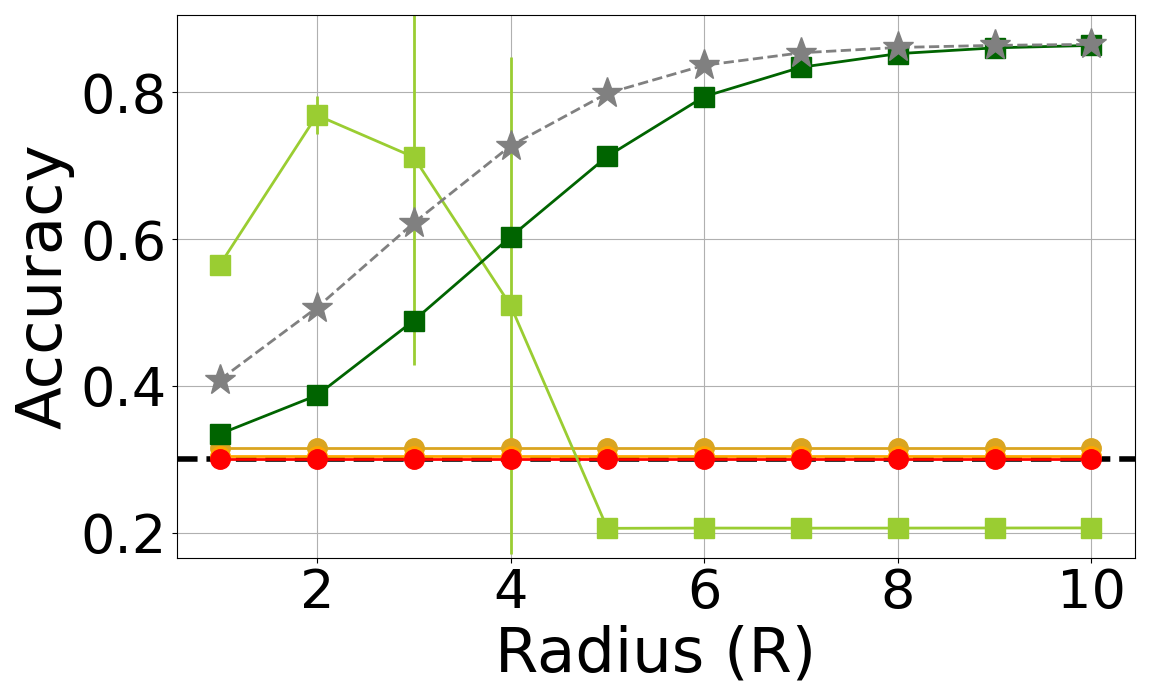} 
\caption{$a=13$, $b=0.5$, $\alpha=0.7$.}
\end{subfigure}\\
\caption{\label{fig:allapp5}Results of the experiments on synthetic datasets with 50,000 vertices, and $k = 5$. The black dashed line represents the accuracy of the noisy side information (without using the graph at all), namely $1-\alpha$.}
\end{figure*}

\begin{figure*}[ht]
\begin{subfigure}{.49\columnwidth}
  \centering
\includegraphics[scale=.16]{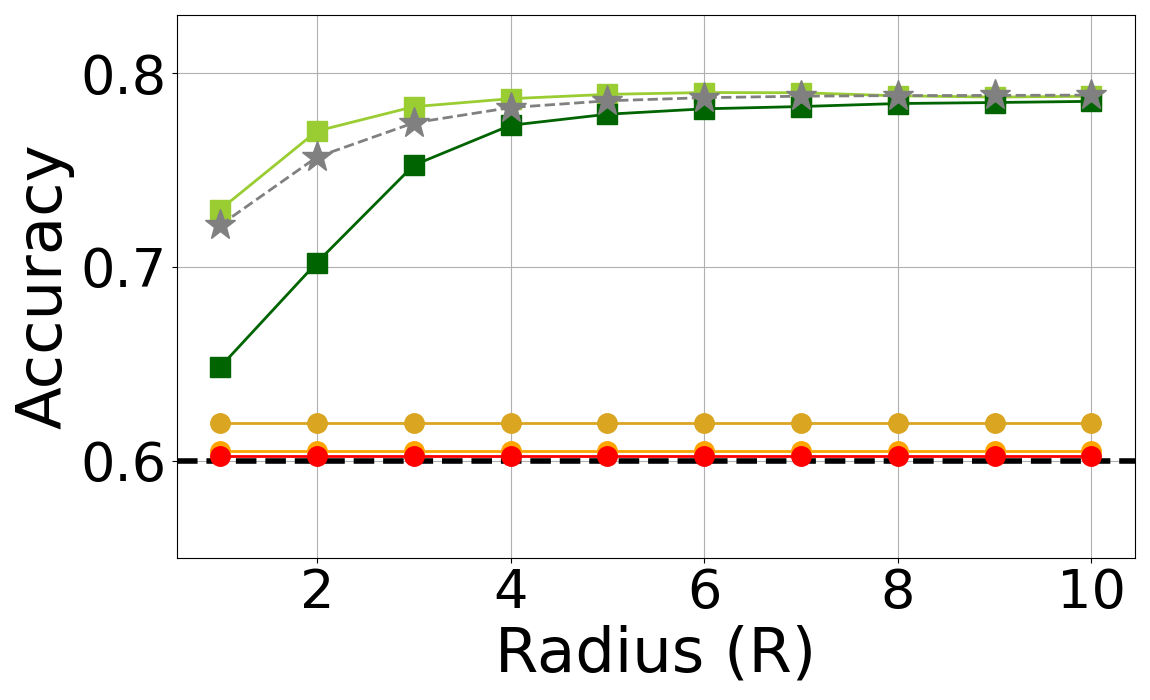} 
\caption{$a=5$, $b=0.05$, $\alpha=0.4$.}
\end{subfigure}
\begin{subfigure}{.49\columnwidth}
  \centering
   \includegraphics[scale=.16]{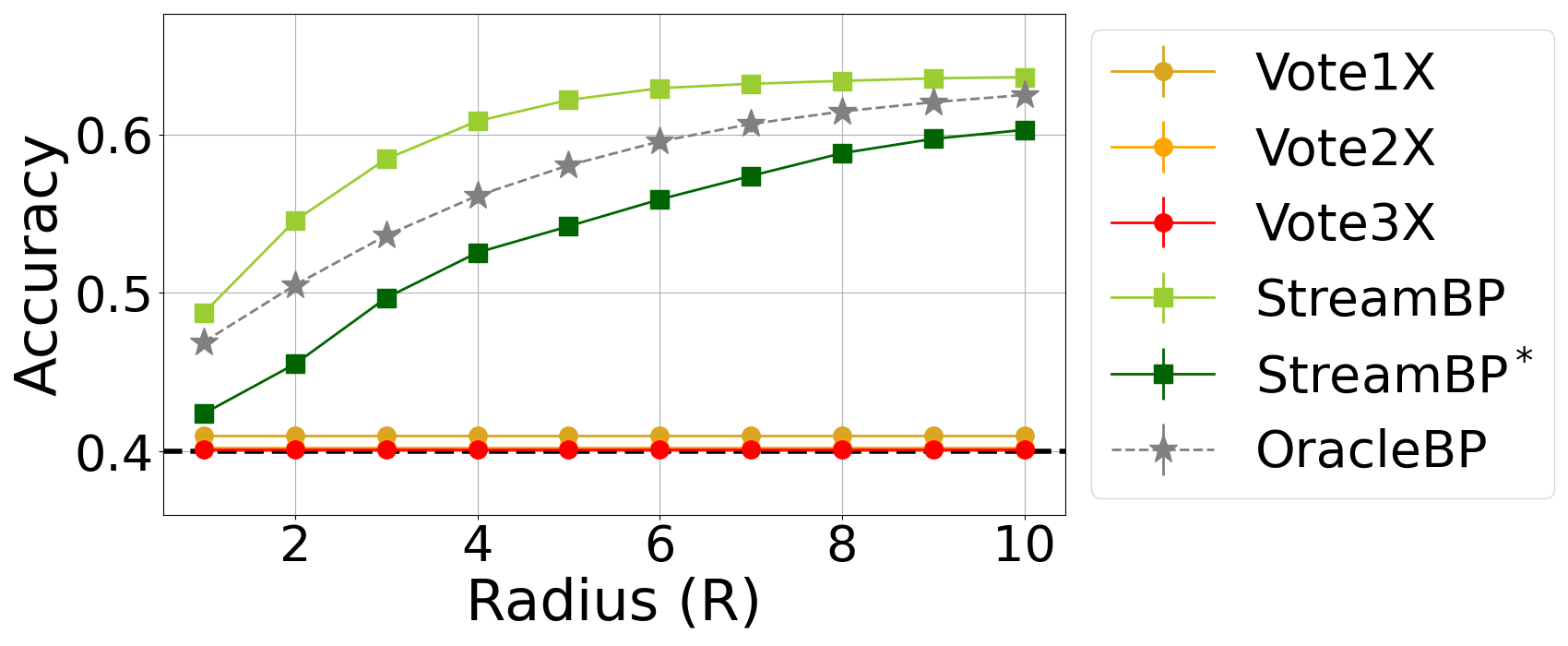} 
\caption{$a=5$, $b=0.05$, $\alpha=0.6$.}
\end{subfigure}\\

\begin{subfigure}{.33\columnwidth}
  \centering
\includegraphics[scale=.16]{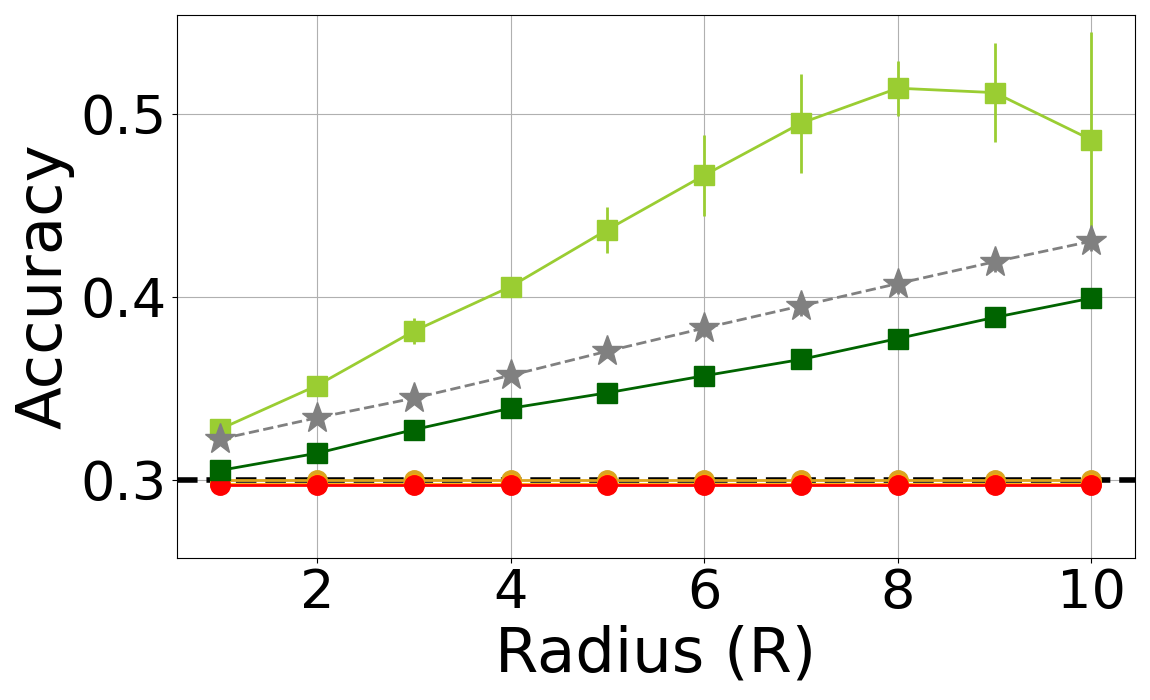} 
\caption{$a=5$, $b=0.05$, $\alpha=0.7$.}
\end{subfigure}
\begin{subfigure}{.33\columnwidth}
  \centering
\includegraphics[scale=.16]{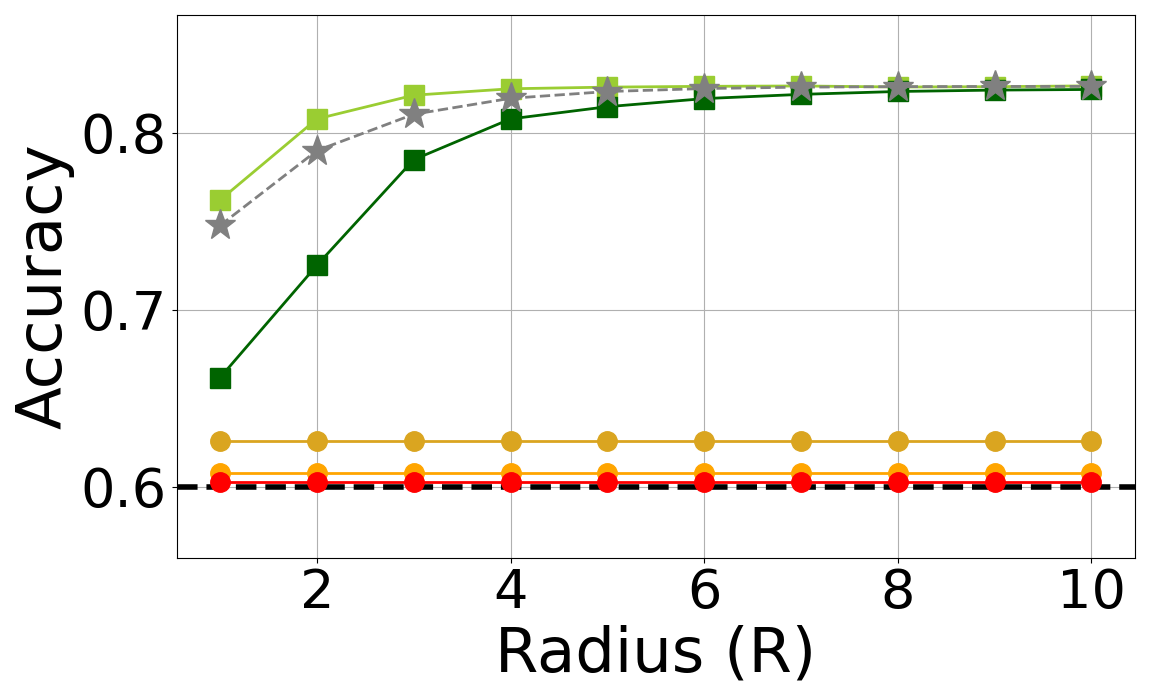} 
\caption{$a=6$, $b=0.1$, $\alpha=0.4$.}
\end{subfigure}
\begin{subfigure}{.33\columnwidth}
  \centering
\includegraphics[scale=.16]{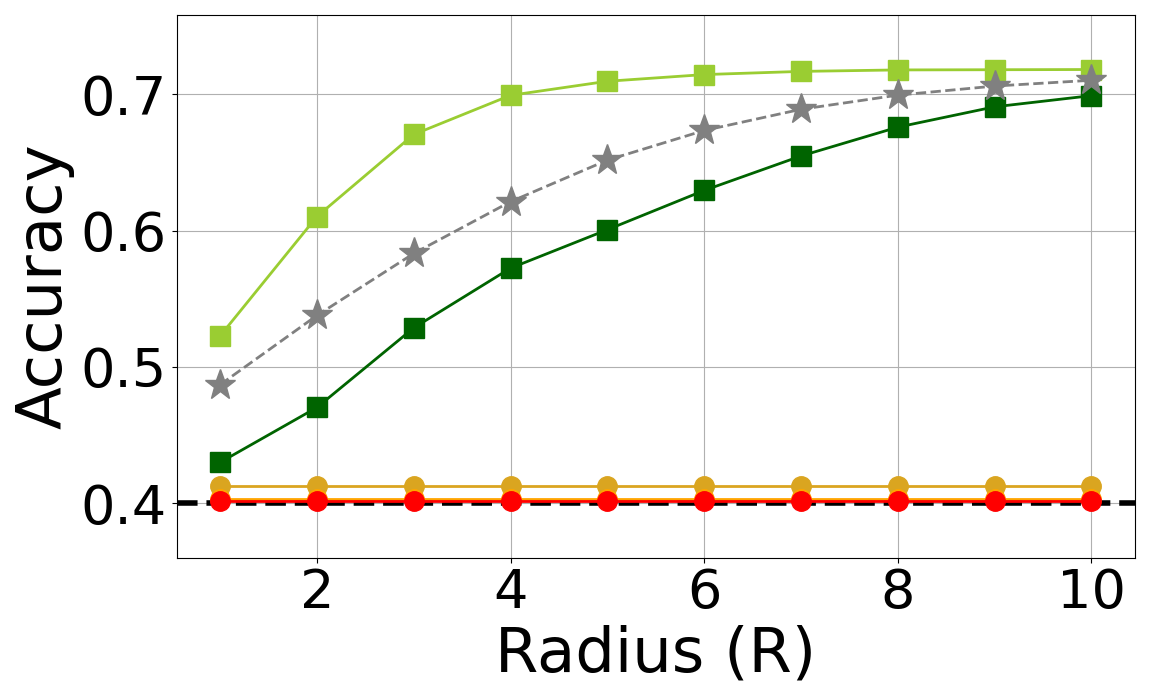} 
\caption{$a=6$, $b=0.1$, $\alpha=0.6$.}
\end{subfigure}\\
\begin{subfigure}{.33\columnwidth}
  \centering
\includegraphics[scale=.16]{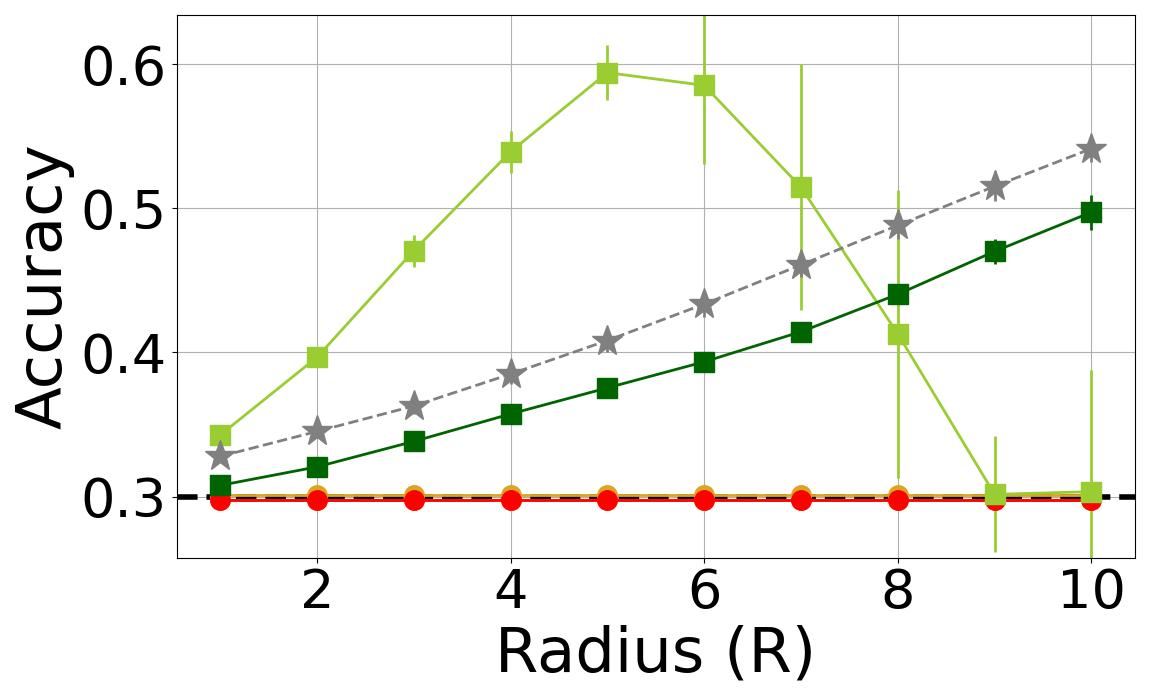} 
\caption{$a=6$, $b=0.1$, $\alpha=0.7$.}
\end{subfigure}
\begin{subfigure}{.33\columnwidth}
  \centering
\includegraphics[scale=.16]{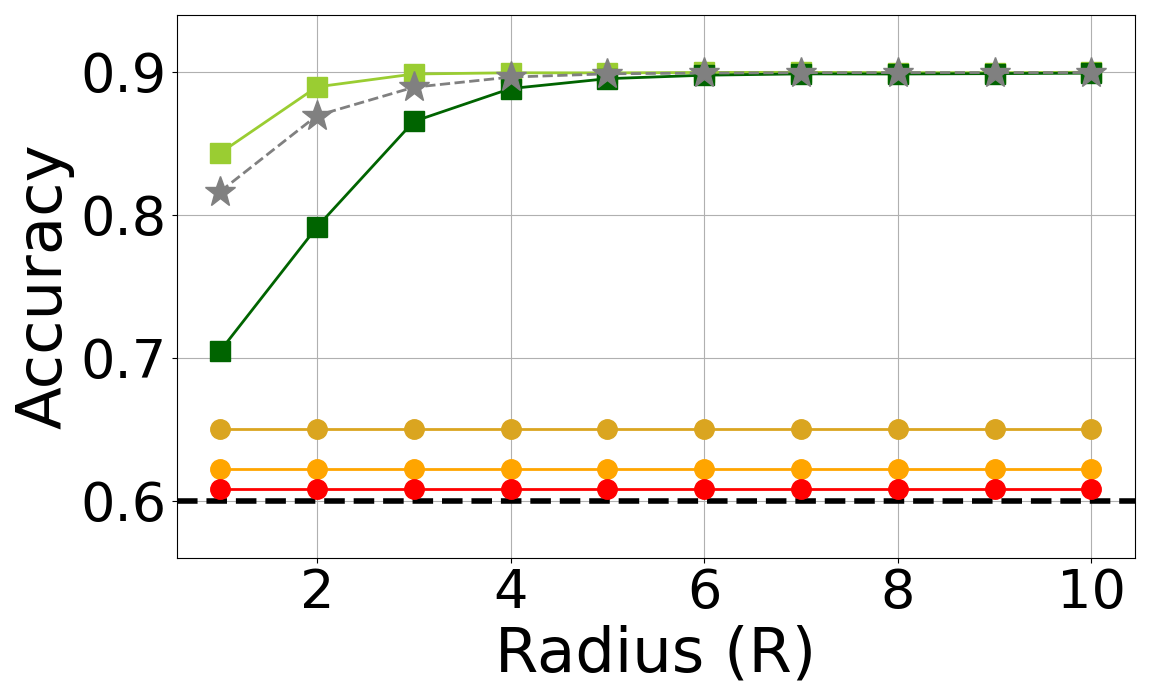} 
\caption{$a=10$, $b=0.5$, $\alpha=0.4$.}
\end{subfigure}
\begin{subfigure}{.33\columnwidth}
  \centering
\includegraphics[scale=.16]{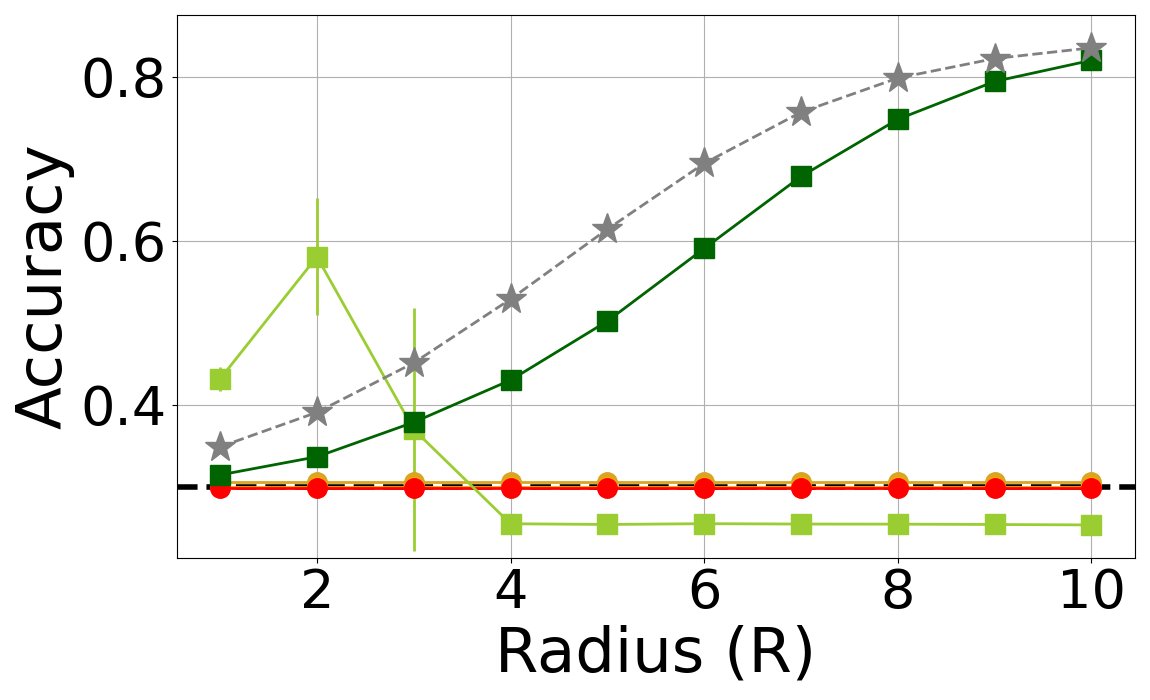} 
\caption{$a=10$, $b=0.5$, $\alpha=0.7$.}
\end{subfigure}\\
\caption{\label{fig:allapp4}Results of the experiments on synthetic datasets with 50,000 vertices, and $k = 4$. The black dashed line represents the accuracy of the noisy side information (without using the graph at all), namely $1-\alpha$.}
\end{figure*}
\begin{figure*}[ht]
\begin{subfigure}{.49\textwidth}
  \centering
   \includegraphics[scale=.16]{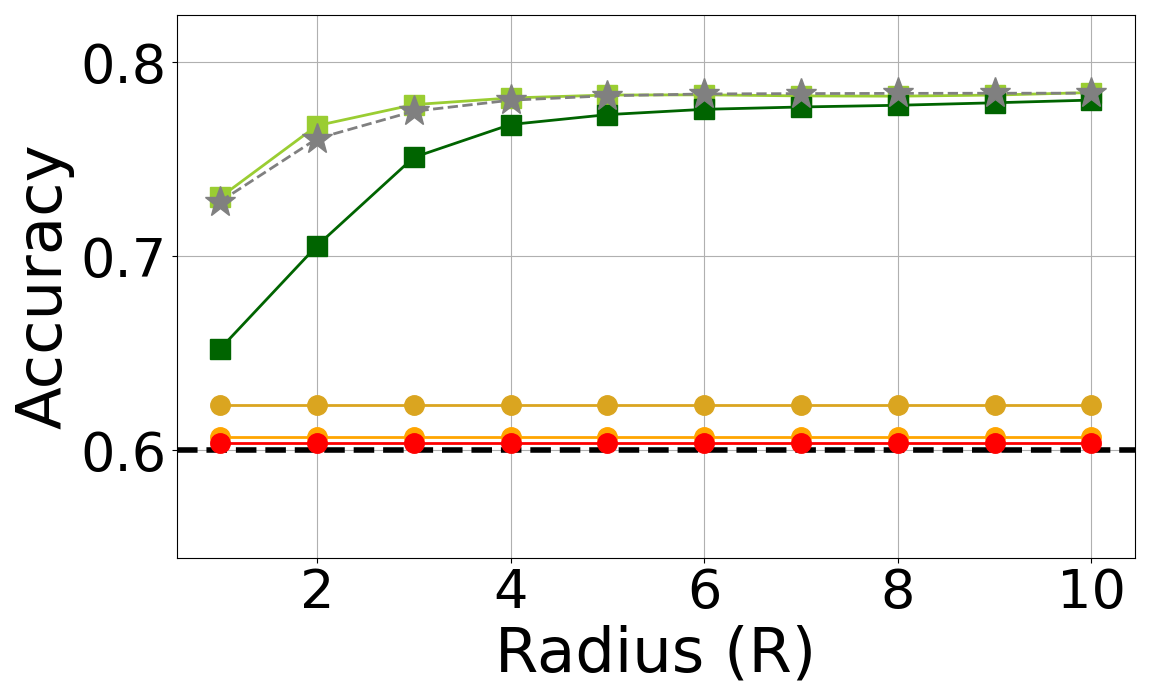} 
\caption{$a=6$, $b=0.05$, $\alpha=0.4$.}
\end{subfigure}
\begin{subfigure}{.49\textwidth}
  \centering
    \includegraphics[scale=.16]{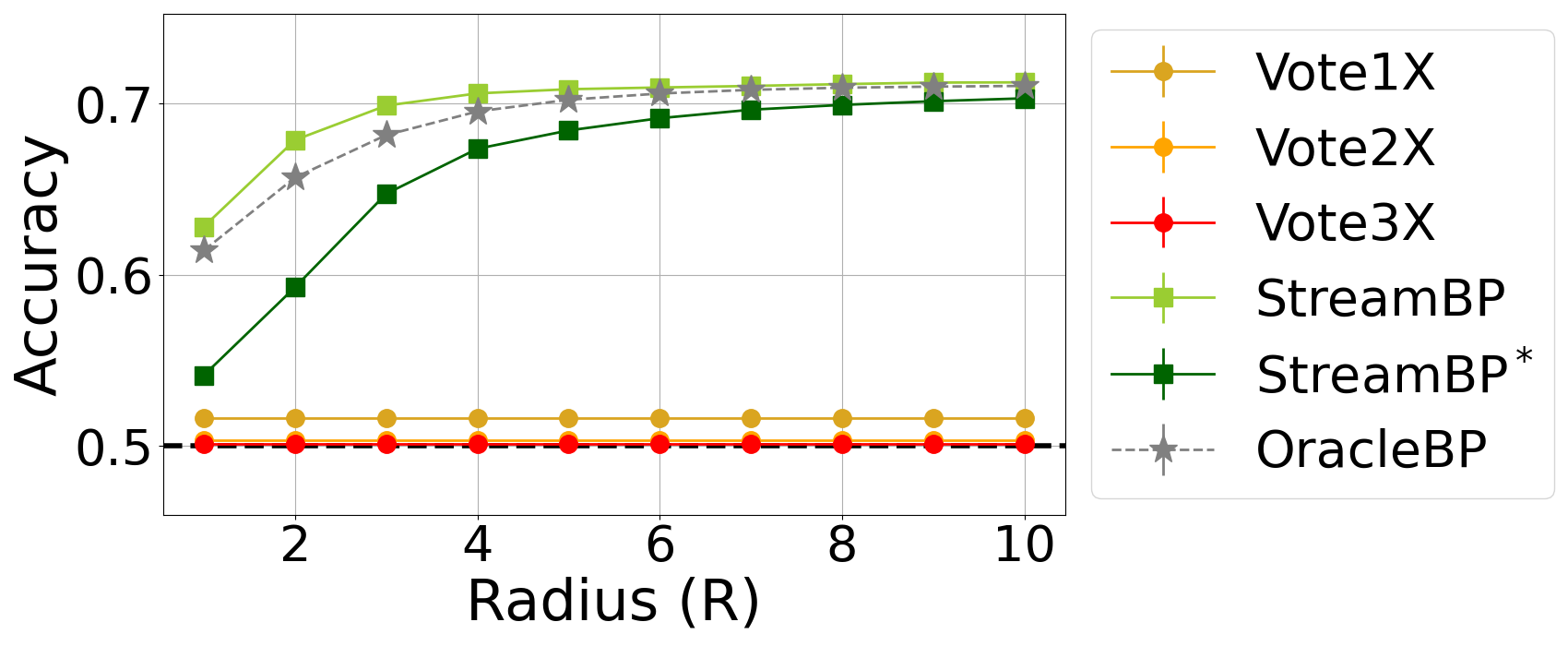} 
\caption{$a=6$, $b=0.05$, $\alpha=0.5$.}
\end{subfigure}\\

\begin{subfigure}{.33\textwidth}
  \centering
\includegraphics[scale=.16]{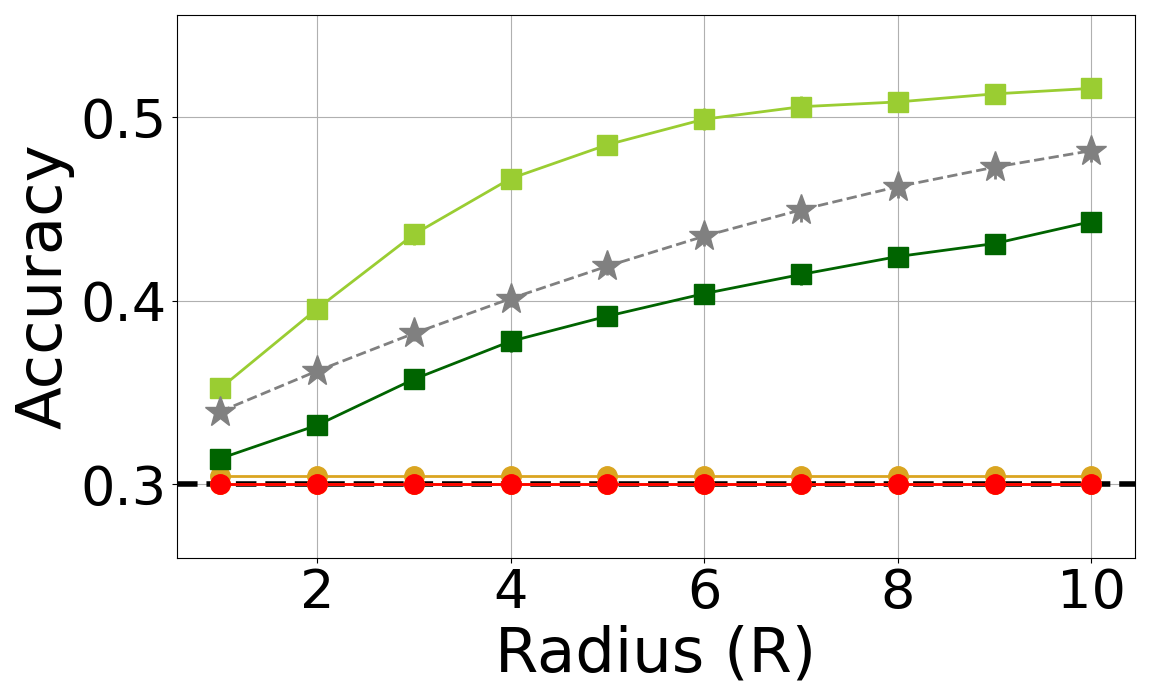} 
\caption{$a=6$, $b=0.05$, $\alpha=0.7$.}
\end{subfigure}
\begin{subfigure}{.33\textwidth}
  \centering
\includegraphics[scale=.16]{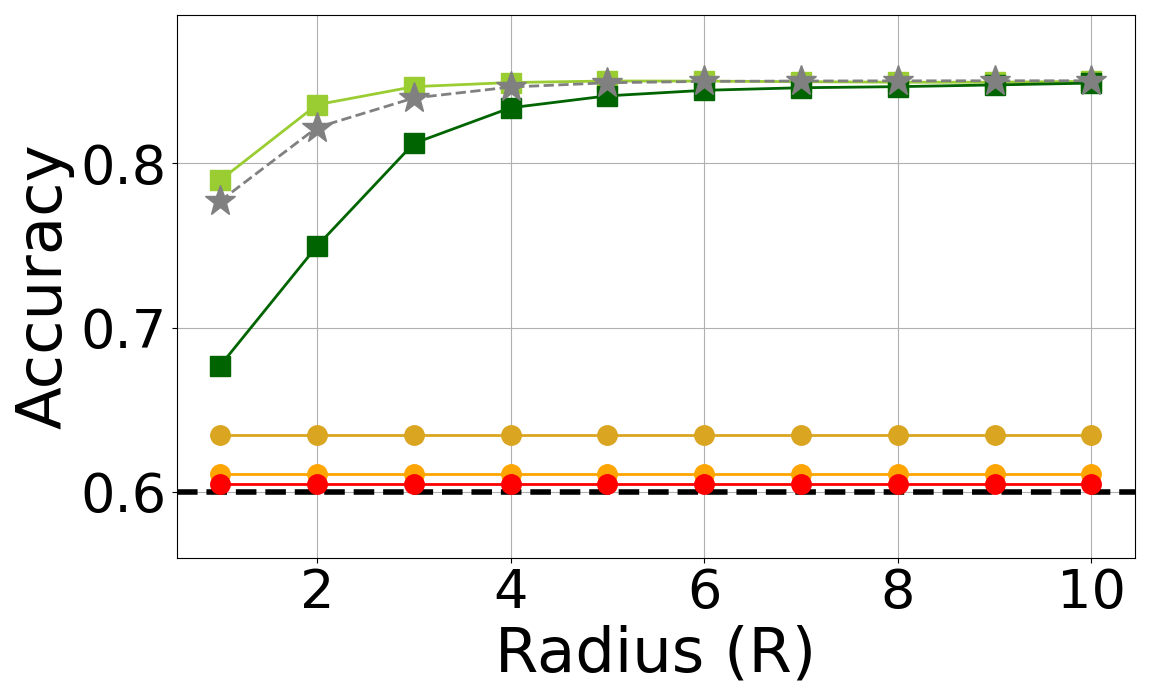} 
\caption{$a=8$, $b=0.1$, $\alpha=0.4$.}
\end{subfigure}
\begin{subfigure}{.33\textwidth}
  \centering
\includegraphics[scale=.16]{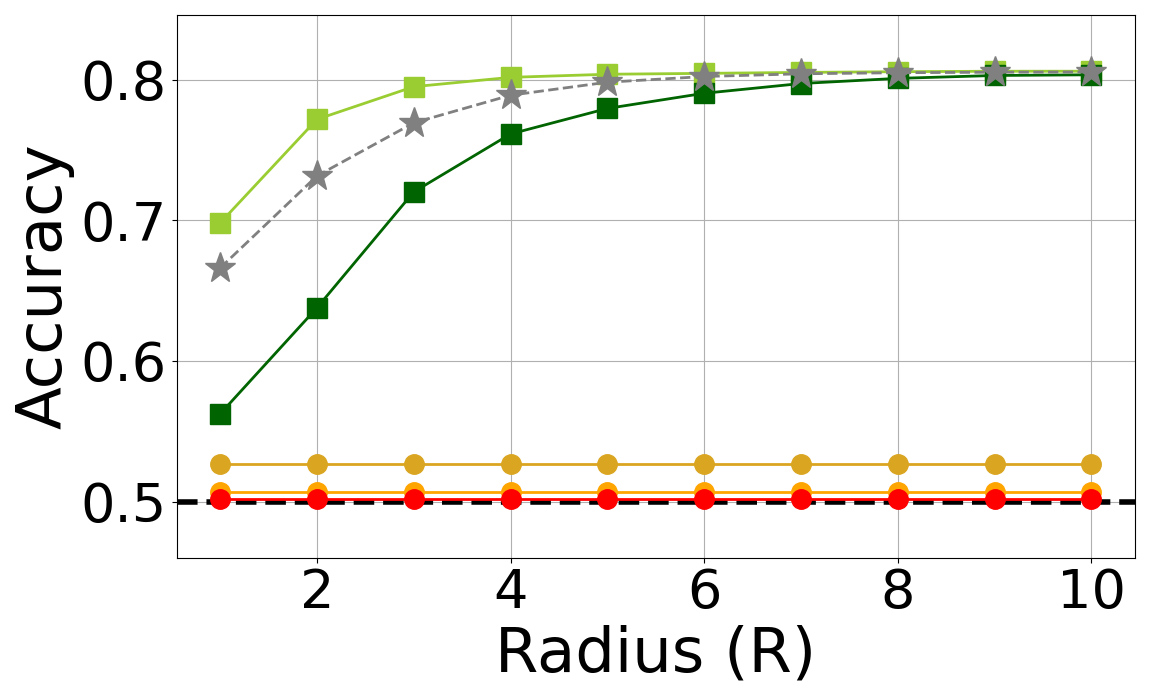} 
\caption{$a=8$, $b=0.1$, $\alpha=0.5$.}
\end{subfigure}\\
\begin{subfigure}{.33\textwidth}
  \centering
\includegraphics[scale=.16]{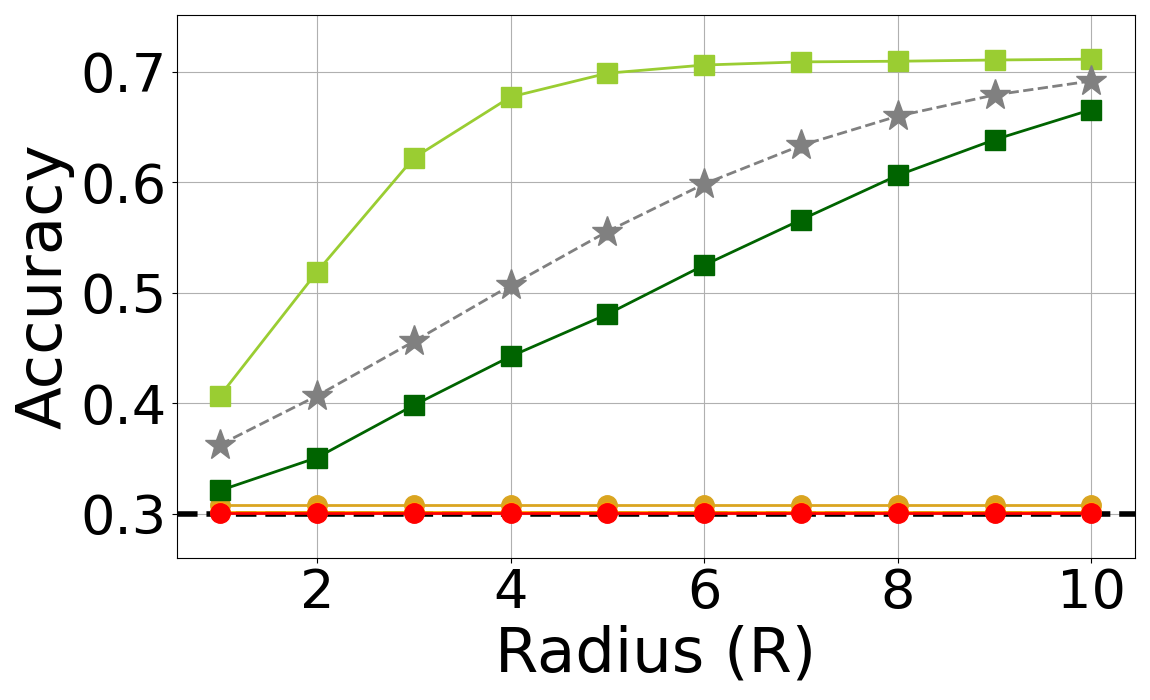} 
\caption{$a=8$, $b=0.1$, $\alpha=0.7$.}
\end{subfigure}
\begin{subfigure}{.33\textwidth}
  \centering
\includegraphics[scale=.16]{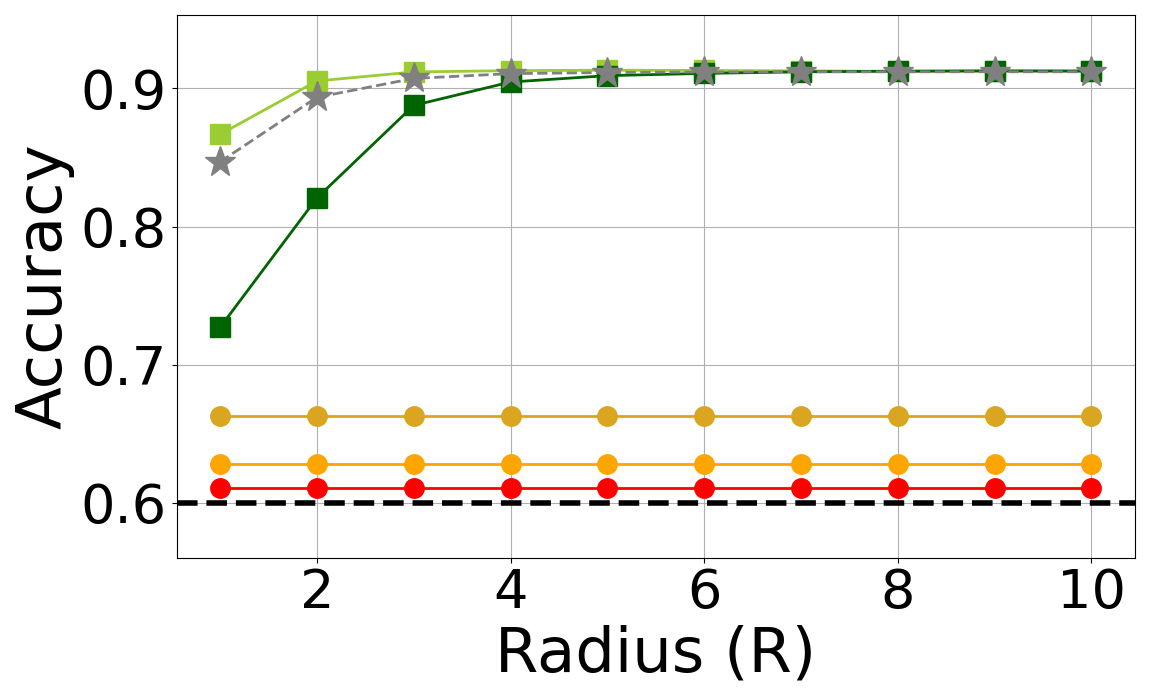} 
\caption{$a=13$, $b=0.5$, $\alpha=0.4$.}
\end{subfigure}
\begin{subfigure}{.33\textwidth}
  \centering
\includegraphics[scale=.16]{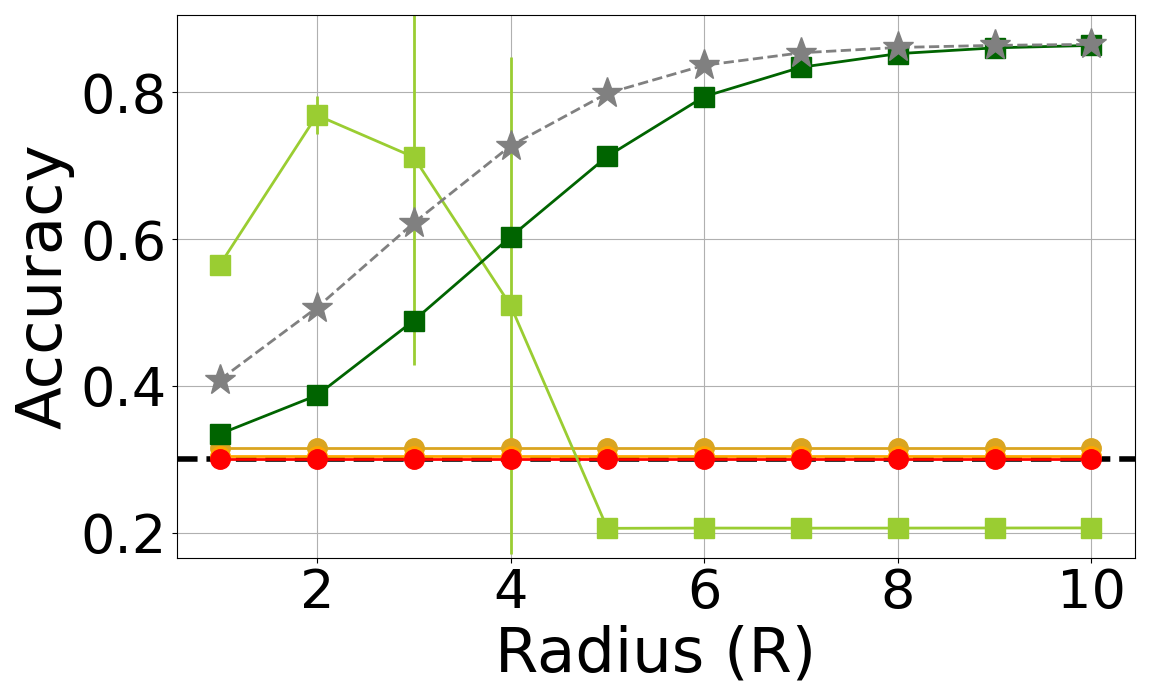} 
\caption{$a=13$, $b=0.5$, $\alpha=0.7$.}
\end{subfigure}\\
\caption{\label{fig:allapp5}Results of the experiments on synthetic datasets with 50,000 vertices, and $k = 5$. The black dashed line represents the accuracy of the noisy side information (without using the graph at all), namely $1-\alpha$.}
\end{figure*}

\paragraph{Variation of the accuracy of \ouralgo and \boundedalgo during their execution.}
We also investigate how the accuracy of the streaming algorithms \ouralgo and \boundedalgo varies during their execution on a given graph, as new vertices arrive one at a time. Figure~\ref{fig:accuracy_per_iteration} shows the results obtained for graphs generated according to the distribution $\STSSBM(n, a, b, k, \alpha)$, for various parameter settings.

\begin{figure*}[ht]
\begin{subfigure}{.49\textwidth}
  \centering
    \includegraphics[scale=.16]{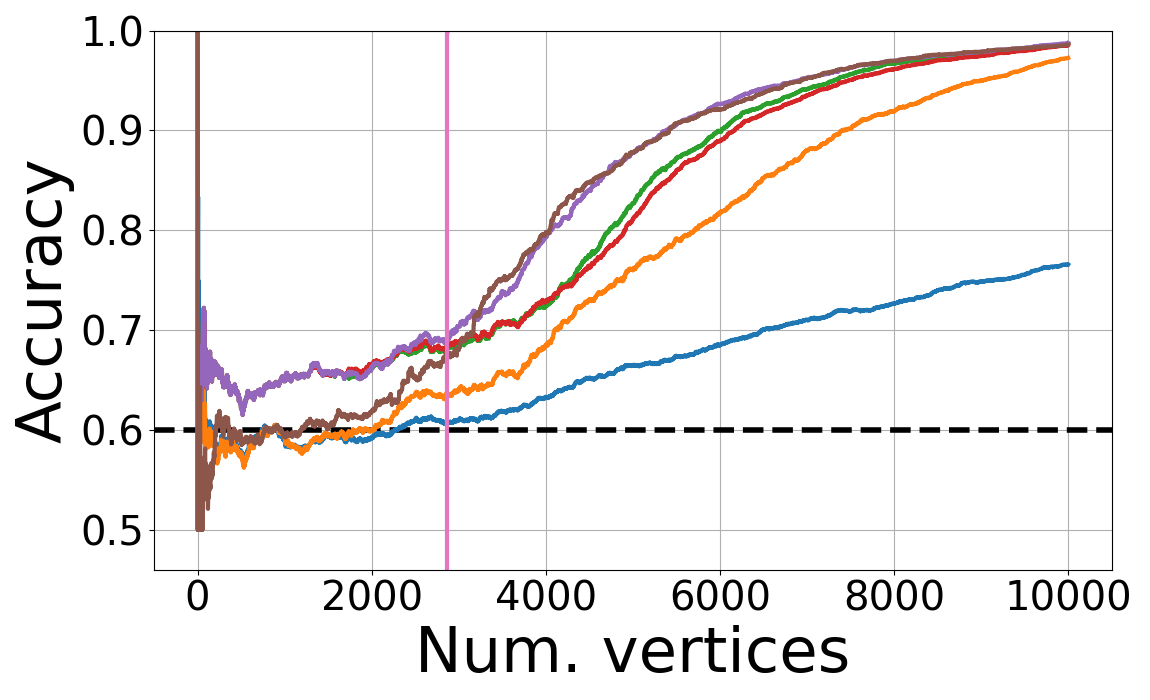} 
\caption{$k = 2$, $a=7.0$, $b = 0.1$, $\alpha = 0.4$.}
\end{subfigure}
\begin{subfigure}{.49\textwidth}
  \centering
\includegraphics[scale=.16]{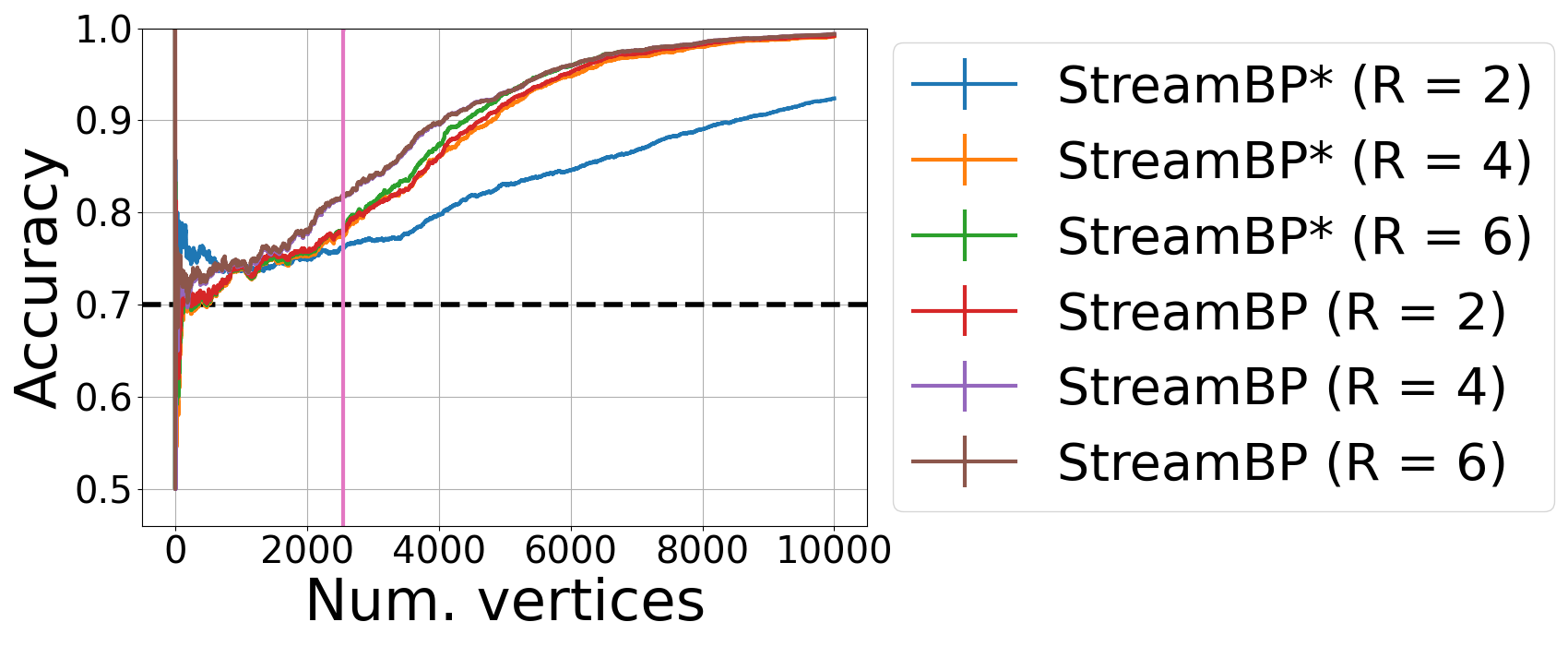} 
\caption{$k = 2$, $a=8.0$, $b = 0.5$, $\alpha = 0.3$.}
\end{subfigure} \\
\begin{subfigure}{.3\textwidth}
  \centering
   \includegraphics[scale=.16]{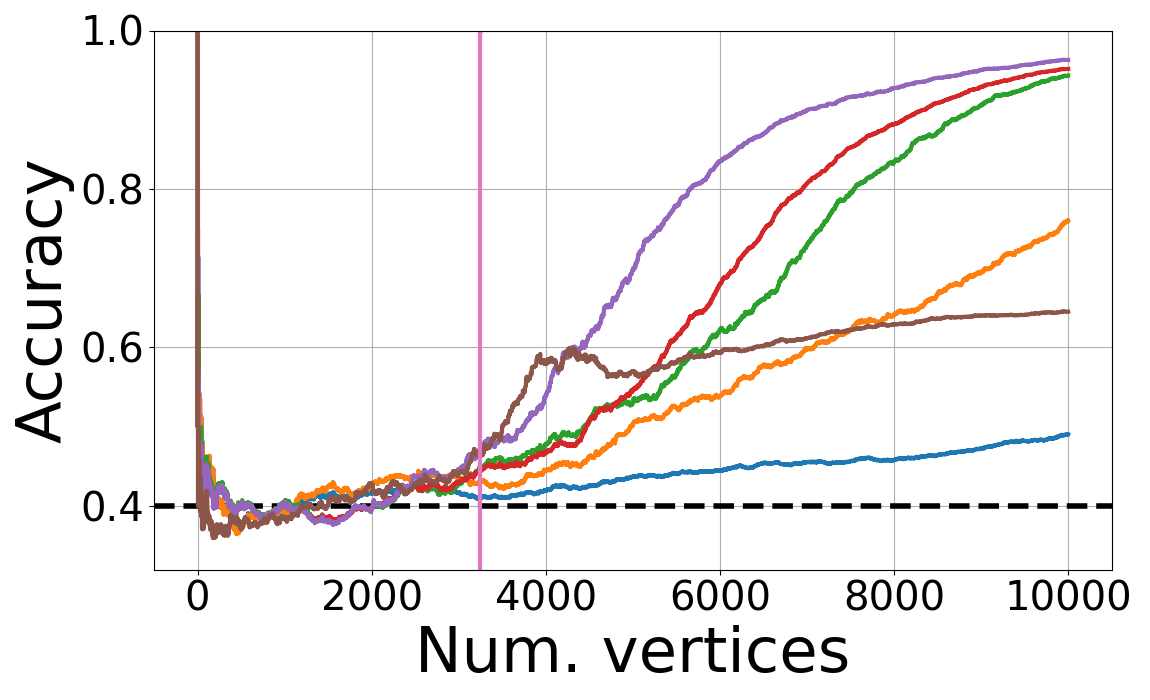}
\caption{$k = 3$, $a = 10.0$, $b = 0.2$, $\alpha = 0.6$.}
\end{subfigure}
\begin{subfigure}{.3\textwidth}
  \centering
    \includegraphics[scale=.16]{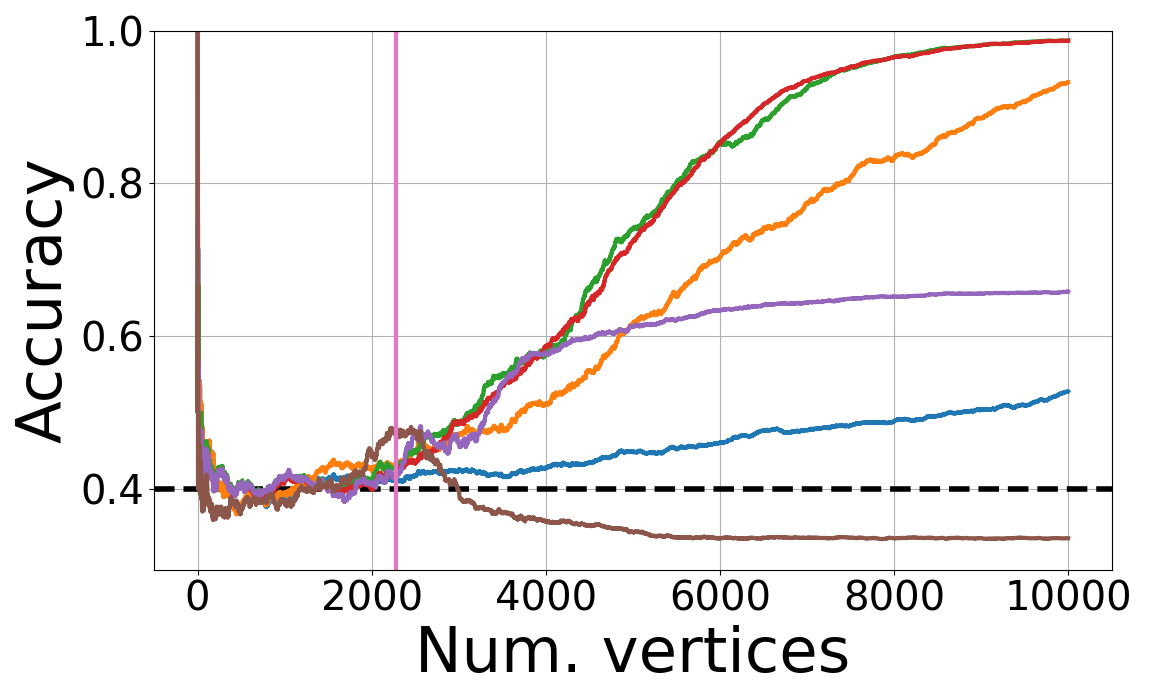} 
\caption{$k = 3$, $a = 15.0$, $b = 0.5$, $\alpha = 0.6$.}
\end{subfigure}
\begin{subfigure}{.3\textwidth}
  \centering
\includegraphics[scale=.16]{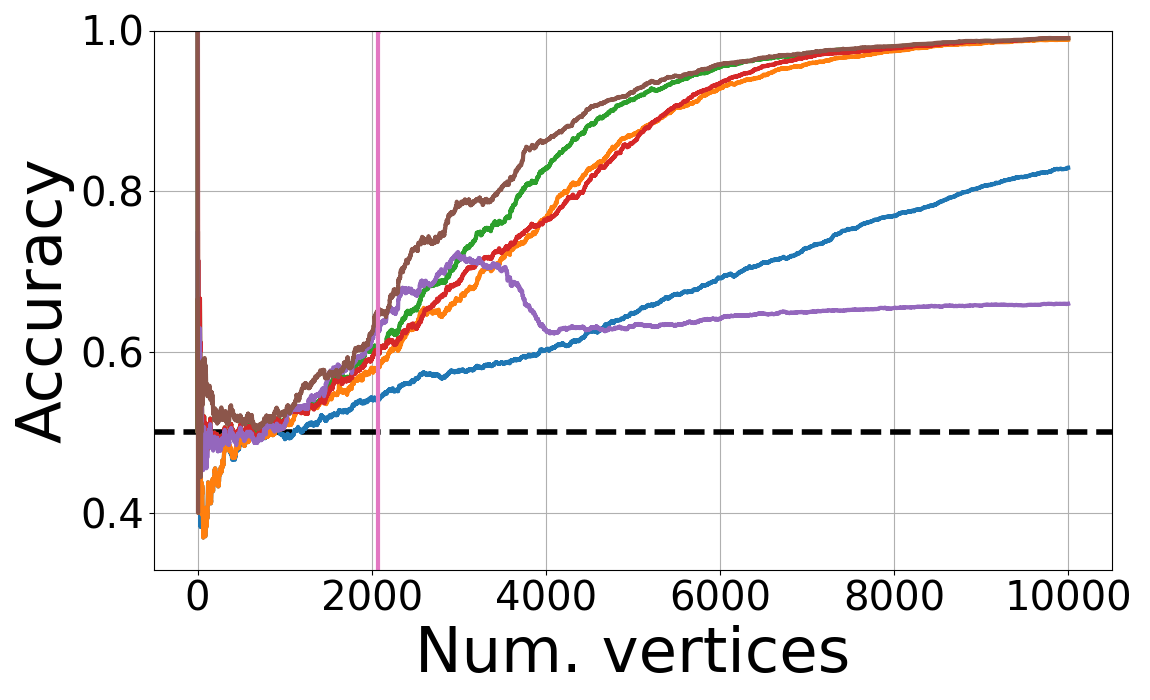}
\caption{$k = 3$, $a = 18.0$, $b = 1.0$, $\alpha = 0.5$.}
\end{subfigure} \\
\begin{subfigure}{.33\textwidth}
  \centering
   \includegraphics[scale=.16]{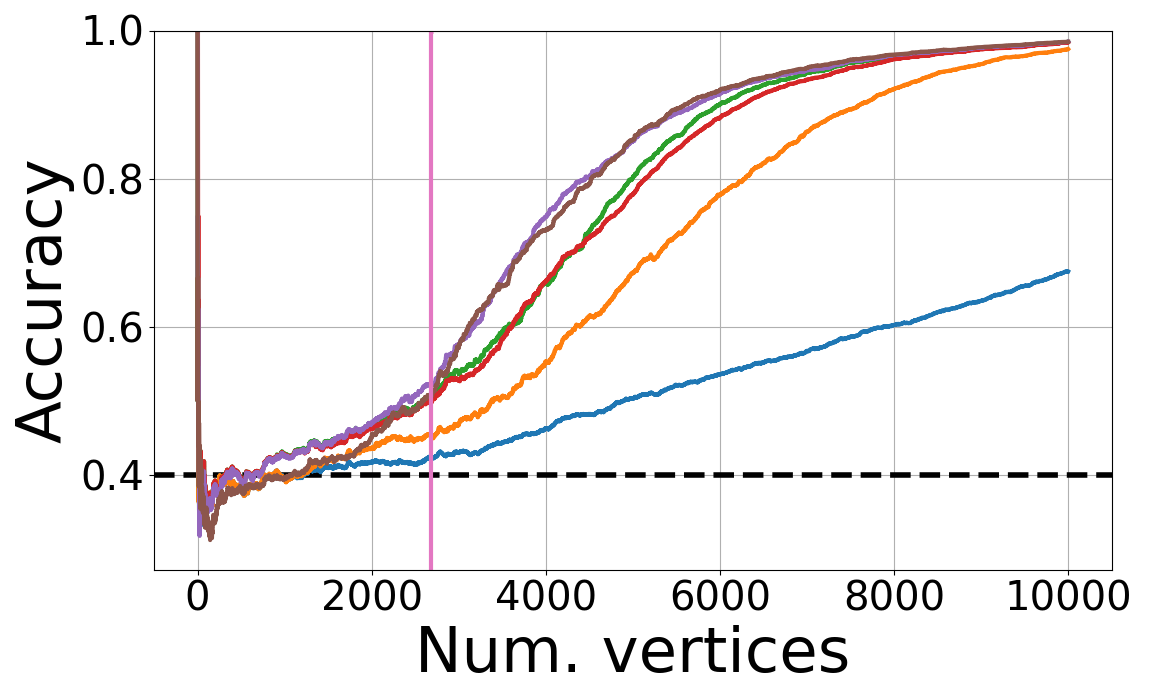}
\caption{$k = 4$, $a = 15.0$, $b = 0.01$, $\alpha = 0.6$.}
\end{subfigure}
\begin{subfigure}{.33\textwidth}
  \centering
    \includegraphics[scale=.16]{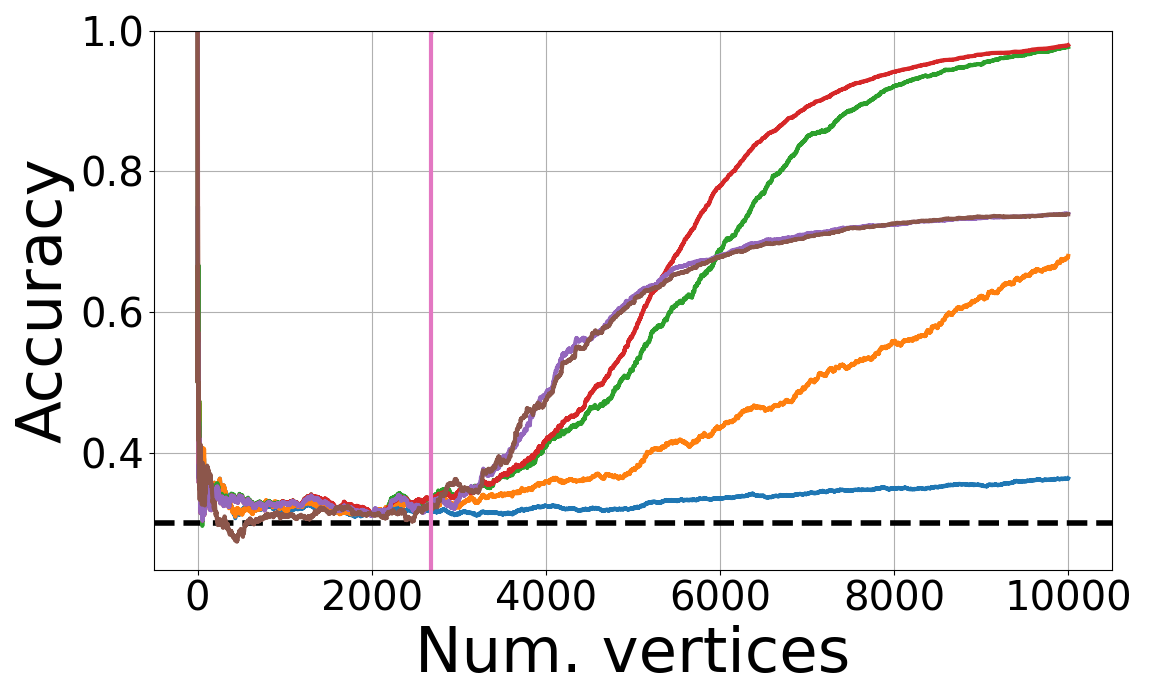} 
\caption{$k = 4$, $a = 15.0$, $b = 0.01$, $\alpha = 0.7$.}
\end{subfigure}
\begin{subfigure}{.33\textwidth}
  \centering
\includegraphics[scale=.16]{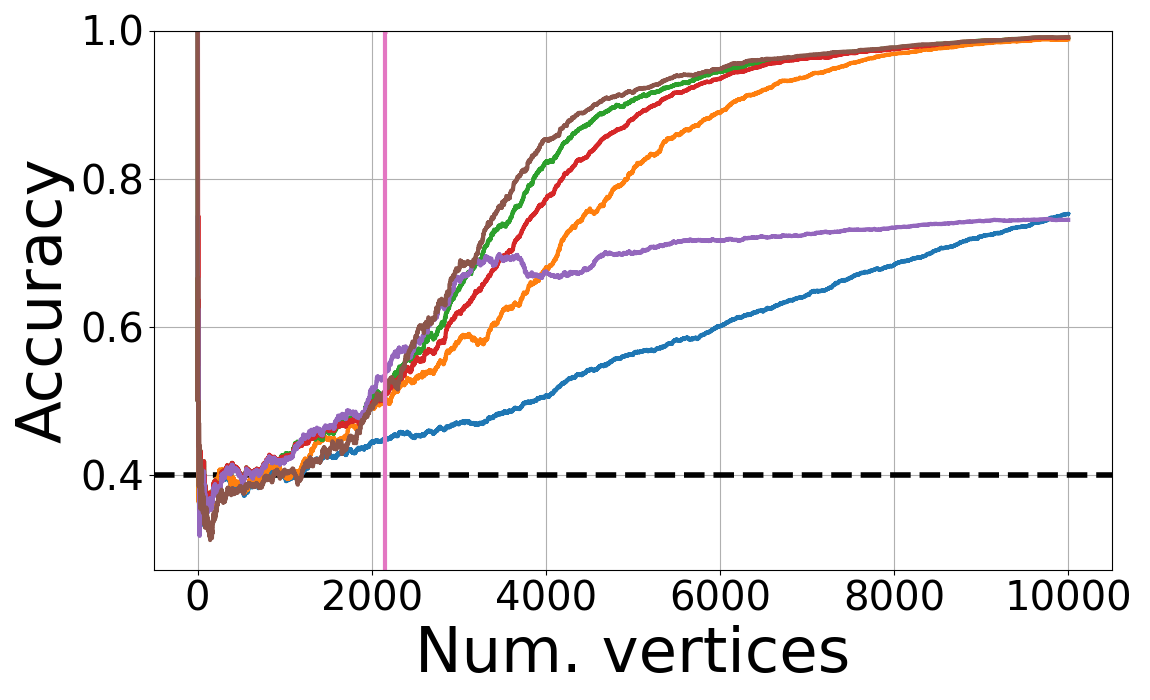}
\caption{$k = 4$, $a = 23.0$, $b = 1.0$, $\alpha = 0.6$.}
\end{subfigure} \\
\caption{\label{fig:accuracy_per_iteration}Accuracy per iteration obtained by running \ouralgo and \boundedalgo on a graph sampled from $\STSSBM(n, k, a, b, \alpha)$ with $n = 10,000$. The black dashed line represents the accuracy of the noisy side information (without using the graph at all), namely $1-\alpha$. The pink vertical line indicates the point where the signal-to-noise ratio crosses the threshold $1.0$.}
\end{figure*}

We generally observe a very high accuracy  (close to 1.0) for the first very few vertices, and then a sharp decline. This is not surprising given our use of the estimation accuracy $\est_ac(\cA)$ defined in Section~\ref{sec:model};
in particular, the accuracy upon arrival of the first vertex is always equal to $1.0$. After this initial sharp decline, the arrival of more vertices almost steadily improves the accuracy, which eventually stabilizes. Note that this improvement accelerates at some point; this is a trace of the phase transition at the KS threshold which is blurred because of side information. Also note that the truncation of a triple $(\tau,\ttau, \bG) \sim \STSSBM(n, k, a, b, \alpha)$ to the first $n' < n$ vertices induces the distribution $ \STSSBM(n', k, \frac{n'}{n}a, \frac{n'}{n}b, \alpha)$, whose signal-to-noise ratio is $\lambda(n') \triangleq \frac{n'}{n} \lambda$; i.e., the signal-to-noise ratio increases linearly with $n'$. Each plot in Figure~\ref{fig:accuracy_per_iteration} has a vertical line showing the point where $\lambda(n')$ crosses the threshold $1.0$.

\paragraph*{Robustness of \ouralgo and \boundedalgo w.r.t.\ the parameters $a$ and $b$.}
Note that our proposed algorithms use $a$ and $b$ as input parameters. When applying these algorithms to real-world datasets that do not conform to the \STSSBM~model, we must approximate these parameters. Indeed in Section~\ref{sec:empirical} we used empirical estimates of $a$ and $b$.

In this section we provide some evidence that our algorithms' behavior is robust to the choice of $a$ and $b$. In Figures~\ref{fig:robustness-a} and~\ref{fig:robustness-b} we present some results of experiments on synthetic data, where the algorithm and the model generating the input graph are given different $a$ or $b$ parameters. We observe that even with a relatively high discrepancy between the approximate and true parameters, \boundedalgo still achieves results comparable to the optimal setting of the parameters. Similar observations can be made about \ouralgo.

Various settings of $k$, $a$, $b$, and $\alpha$ are given as the caption to the individual plots. Here $a$ and $b$ signify the true parameters, according to which the input graph was generated. The performance of \boundedalgo is then plotted with various input parameters. For instance, the label ``a+200\%" in Figure~\ref{fig:robustness-a} indicates that the algorithm receives a parameter $a$ that is three times greater than the true value. Similarly 
``b-67\%" in Figure~\ref{fig:robustness-b} indicates that the algorithm receives a parameter $b$ that is 67\% less than the true value.

\begin{figure*}
\begin{subfigure}{.49\textwidth}
  \centering
   \includegraphics[scale=.13]{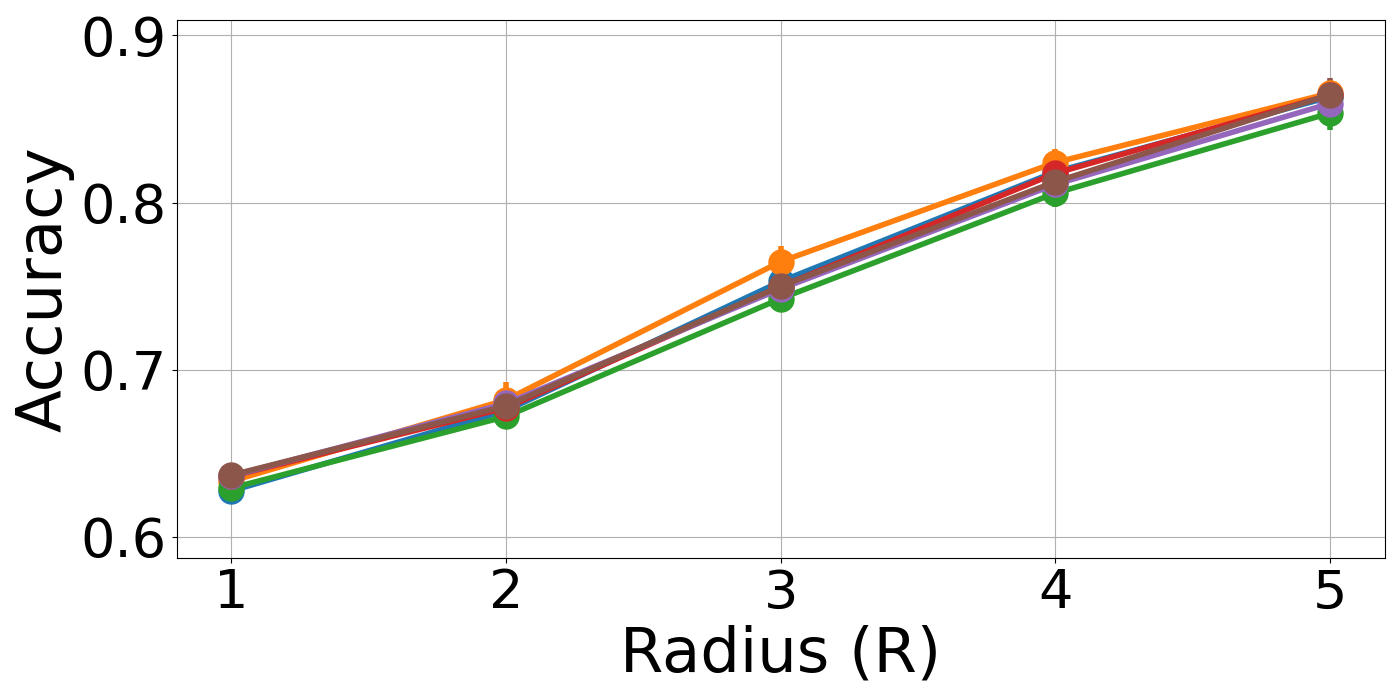}
\caption{$k=2$, $a=4.5$, $b=0.1$, $\alpha=0.4$.}
\end{subfigure}
\begin{subfigure}{.49\textwidth}
  \centering
    \includegraphics[scale=.13]{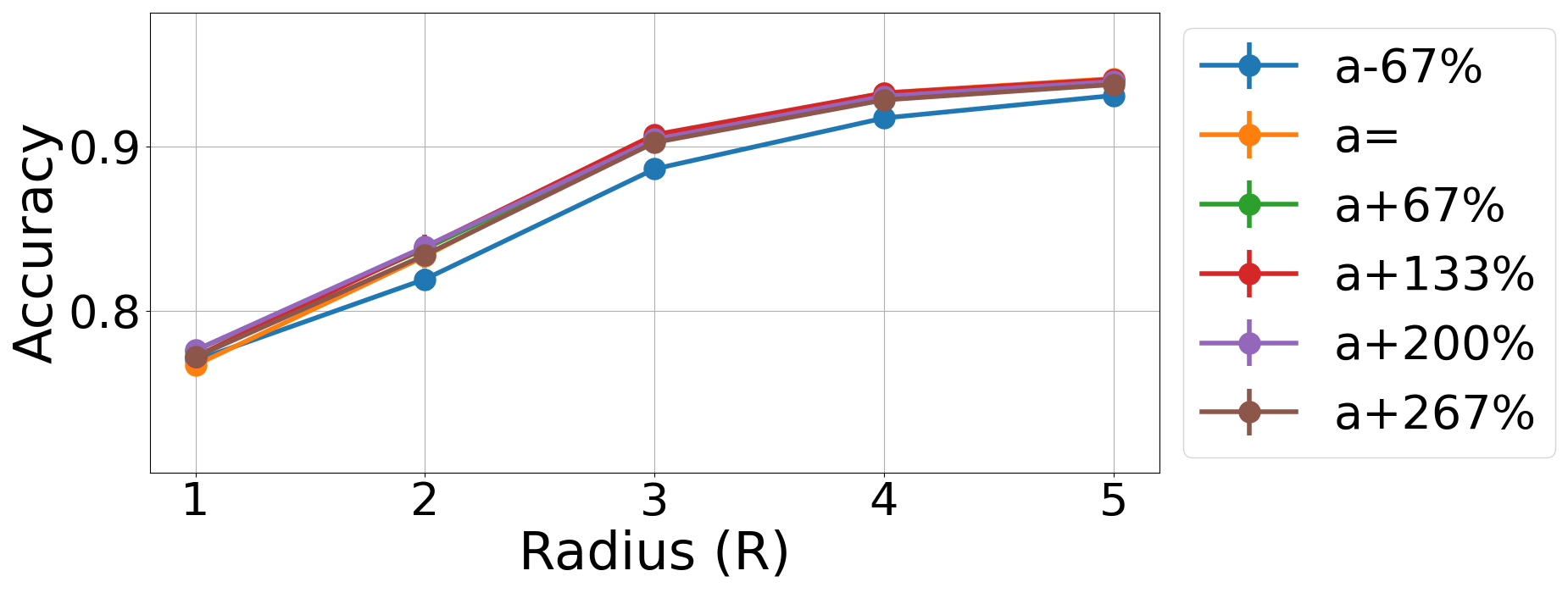}
\caption{$k=2$, $a=6$, $b=0.5$, $\alpha=0.3$.}
\end{subfigure}\\

\begin{subfigure}{.33\textwidth}
  \centering
\includegraphics[scale=.13]{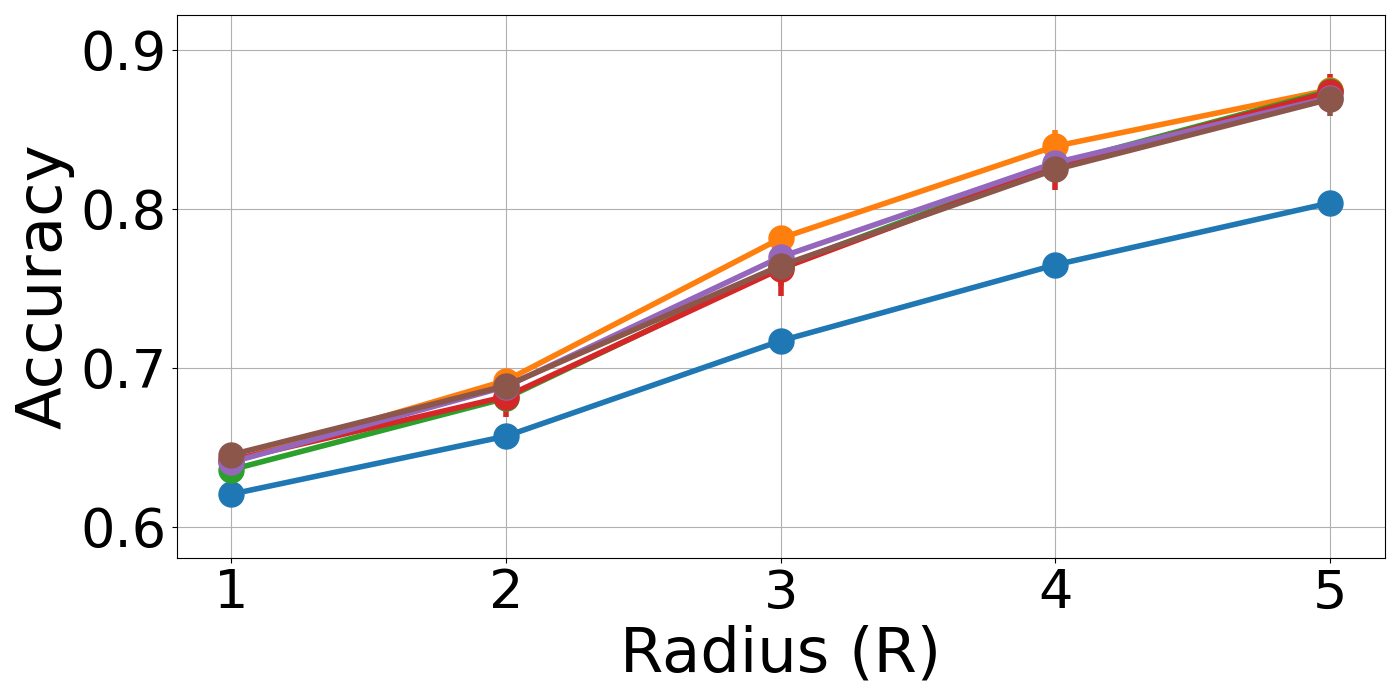}
\caption{$k=2$, $a=7$, $b=1$, $\alpha = 0.4$.}
\end{subfigure}
\begin{subfigure}{.33\textwidth}
  \centering
\includegraphics[scale=.13]{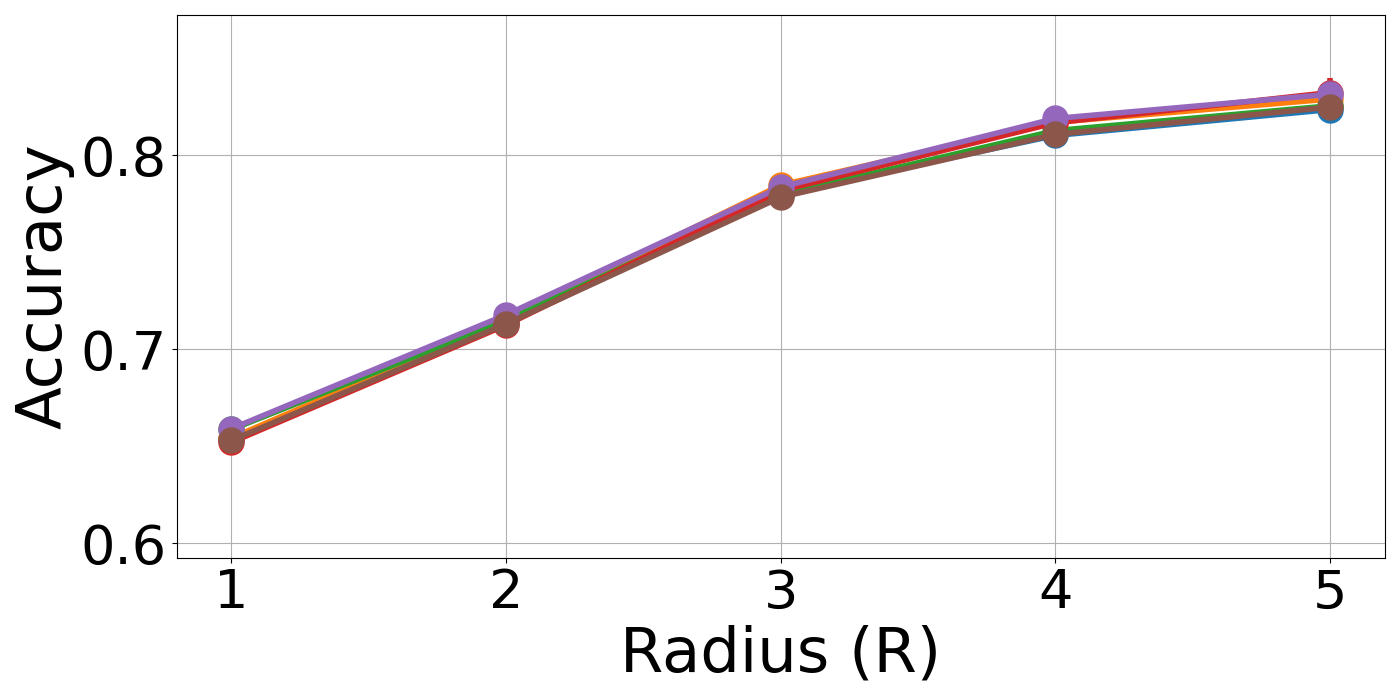} 
\caption{$k=3$, $a=5$, $b=0.1$, $\alpha=0.4$.}
\end{subfigure}
\begin{subfigure}{.33\textwidth}
  \centering
\includegraphics[scale=.13]{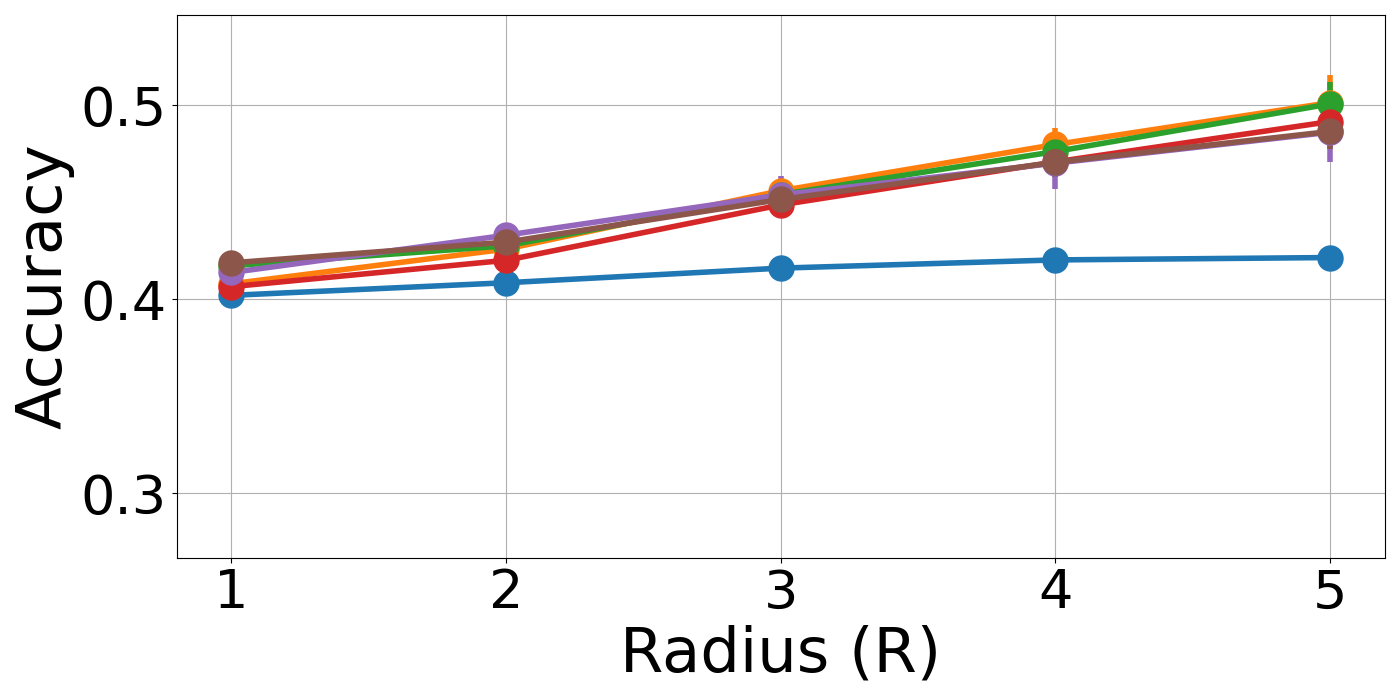} 
\caption{$k=3$, $a=7$, $b=1$, $\alpha=0.6$.}
\end{subfigure}\\
\caption{\label{fig:robustness-a}Robustness of \boundedalgo w.r.t. to the choice of $a$.}
\end{figure*}

\begin{figure*}
\begin{subfigure}{.49\textwidth}
  \centering
   \includegraphics[scale=.13]{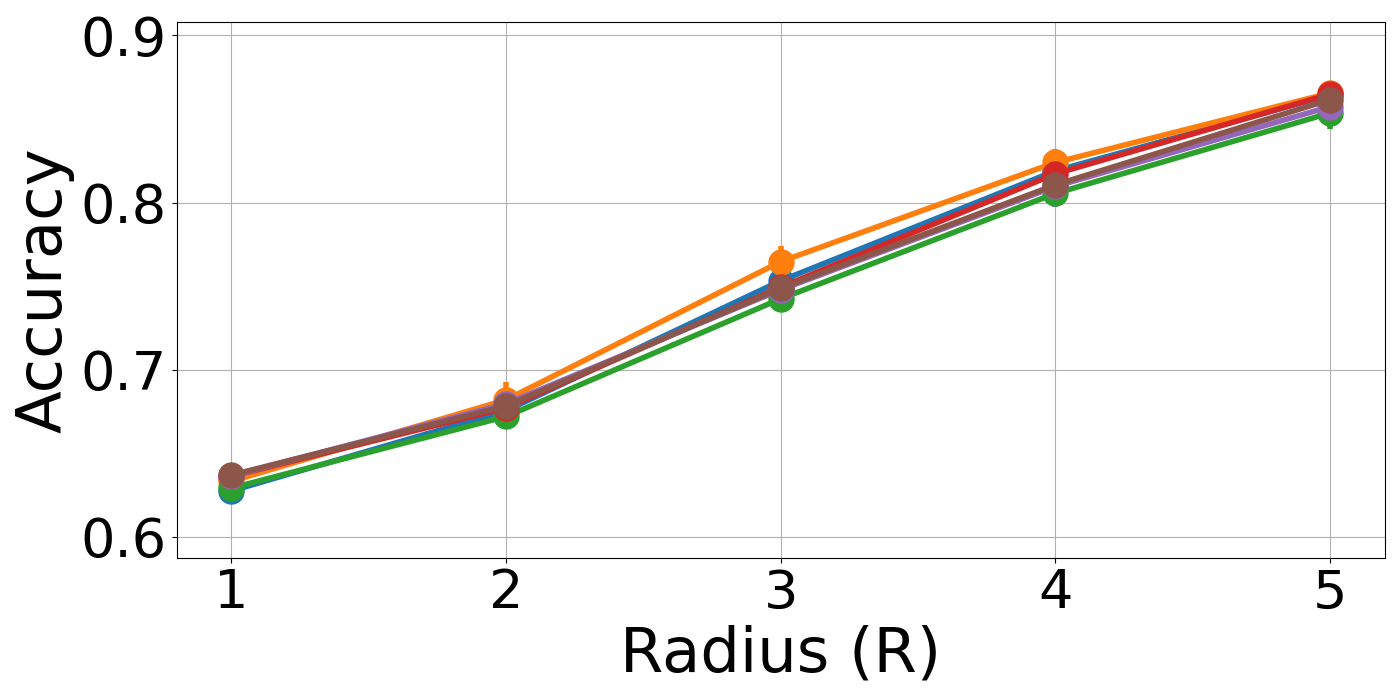}
\caption{$k=2$, $a=4.5$, $b=0.1$, $\alpha=0.4$.}
\end{subfigure}
\begin{subfigure}{.49\textwidth}
  \centering
    \includegraphics[scale=.13]{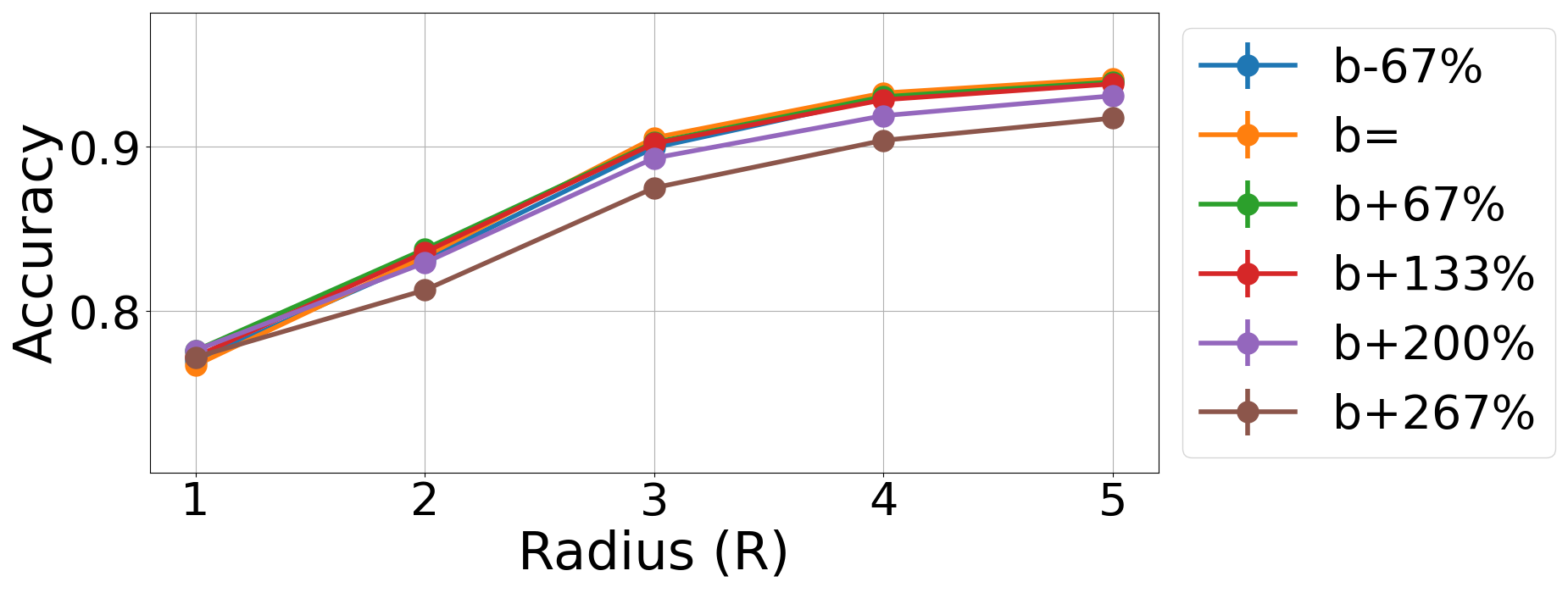}
\caption{$k=2$, $a=6$, $b=0.5$, $\alpha=0.3$.}
\end{subfigure}\\

\begin{subfigure}{.33\textwidth}
  \centering
\includegraphics[scale=.13]{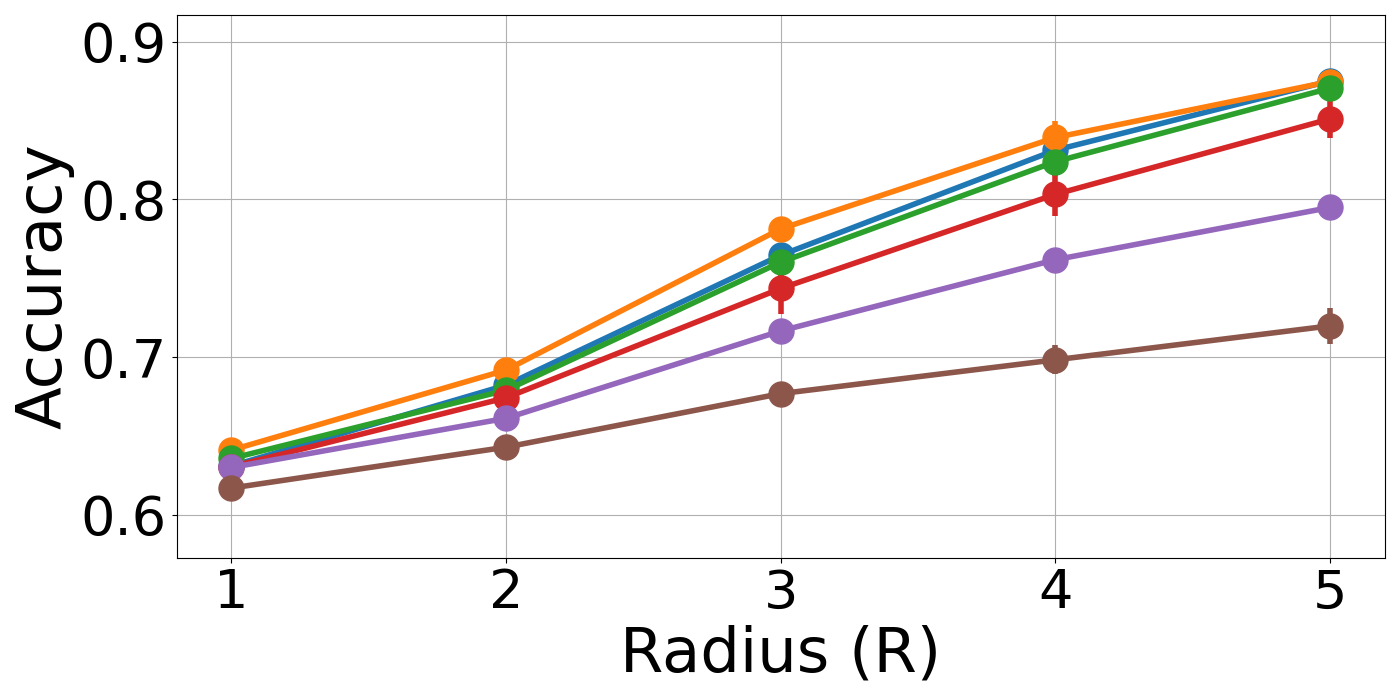}
\caption{$k=2$, $a=7$, $b=1$, $\alpha=0.4$.}
\end{subfigure}
\begin{subfigure}{.33\textwidth}
  \centering
\includegraphics[scale=.13]{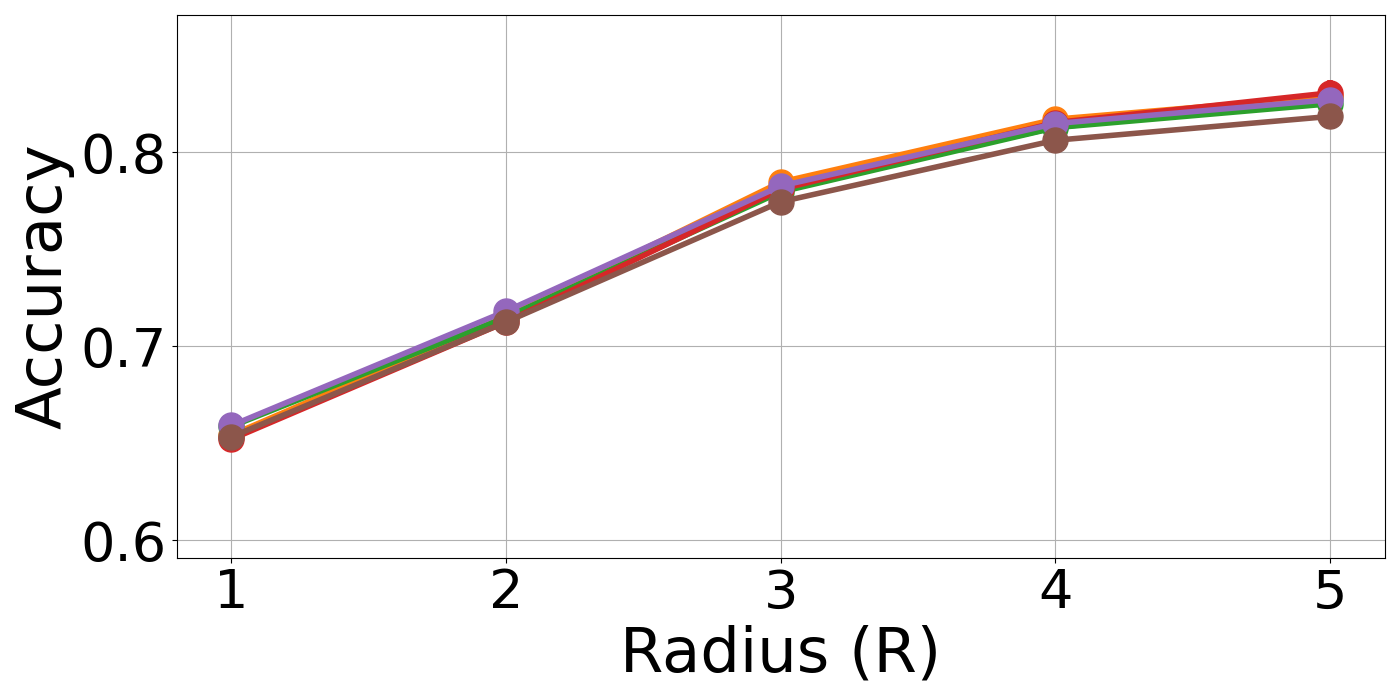} 
\caption{$k=3$, $a=5$, $b=0.1$, $\alpha=0.4$.}
\end{subfigure}
\begin{subfigure}{.33\textwidth}
  \centering
\includegraphics[scale=.13]{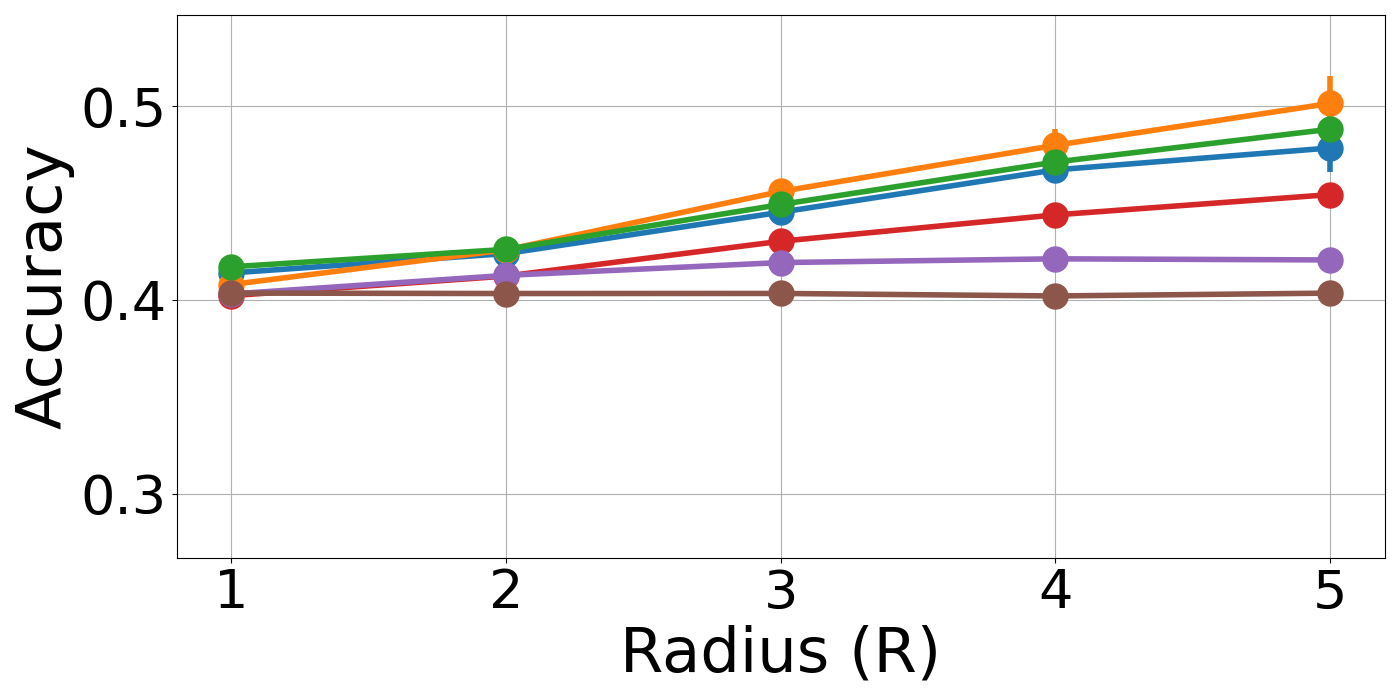} 
\caption{$k=3$, $a=7$, $1$, $\alpha=0.6$.}
\end{subfigure}\\
\caption{\label{fig:robustness-b}Robustness of \boundedalgo w.r.t. to the choice of $b$.}
\end{figure*}


\section{Experiments without side information}
In this section, we provide details of our experiments on synthetic data, when no side information is
provided to the algorithms. Here $\alpha$ is set to $1 - 1/k$, that is, $\ttau$ is completely independent
from $\tau$. Note that in this setting we cannot hope to have large overlap between the output labels
$\htau$ and the true labels $\tau$. Therefore, we must evaluate the algorithms based on the best possible
permutation of the output labels, as described in
Section~\ref{sec:defs} in the definition of $\est_ac(\cA)$:

$$\max_{\pi\in \fS_k}\frac{1}{n}  \sum_{v \in V(n)} \hskip-2mm \ind\left(\htau(v) =  \pi\circ\tau(v)\right).$$

Our observations confirm the theoretical result from Corollary~\ref{cor:streaming-without-sideinfo}. As predicted, when $r$ is small, no algorithm---neither the baselines nor our algorithms---can achieve any meaningful improvement over the trivial $1/k$ accuracy. We further observe that for sufficiently large $r$, both the offline algorithm \offlinealgo and our algorithm \boundedalgo perform significantly better than $1/k$. 
However, this occurs when $r$ is comparable to the diameter of the graph, when an online algorithm is no longer efficient. \ouralgo is unable to get any non-trivial result due to issues described in Section~\ref{sec:streambp_star} which become especially problematic for large values of $a$, $b$, and $r$.

We present the results of our experiments for various settings of $k$, $a$, and $b$ in Figure~ \ref{fig:no-side-info}. Note that the settings always satisfy the Kesten-Stigum condition from Equation~\eqref{eq:KS}. Without side information, this is necessary to see any non-trivial behavior, even for large radii.

\begin{figure*}
\begin{subfigure}{.49\textwidth}
  \centering
   \includegraphics[scale=.16]{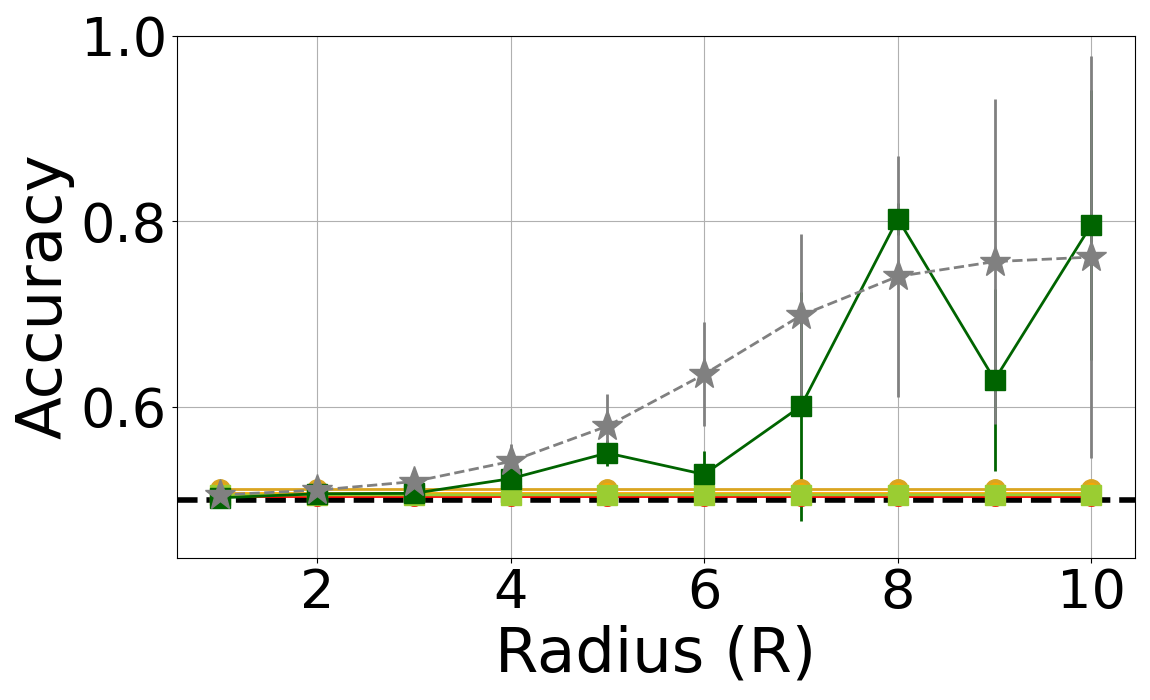}
\caption{$k=2$ $a=7.5$, $b=0.1$.}
\end{subfigure}
\begin{subfigure}{.49\textwidth}
  \centering
    \includegraphics[scale=.16]{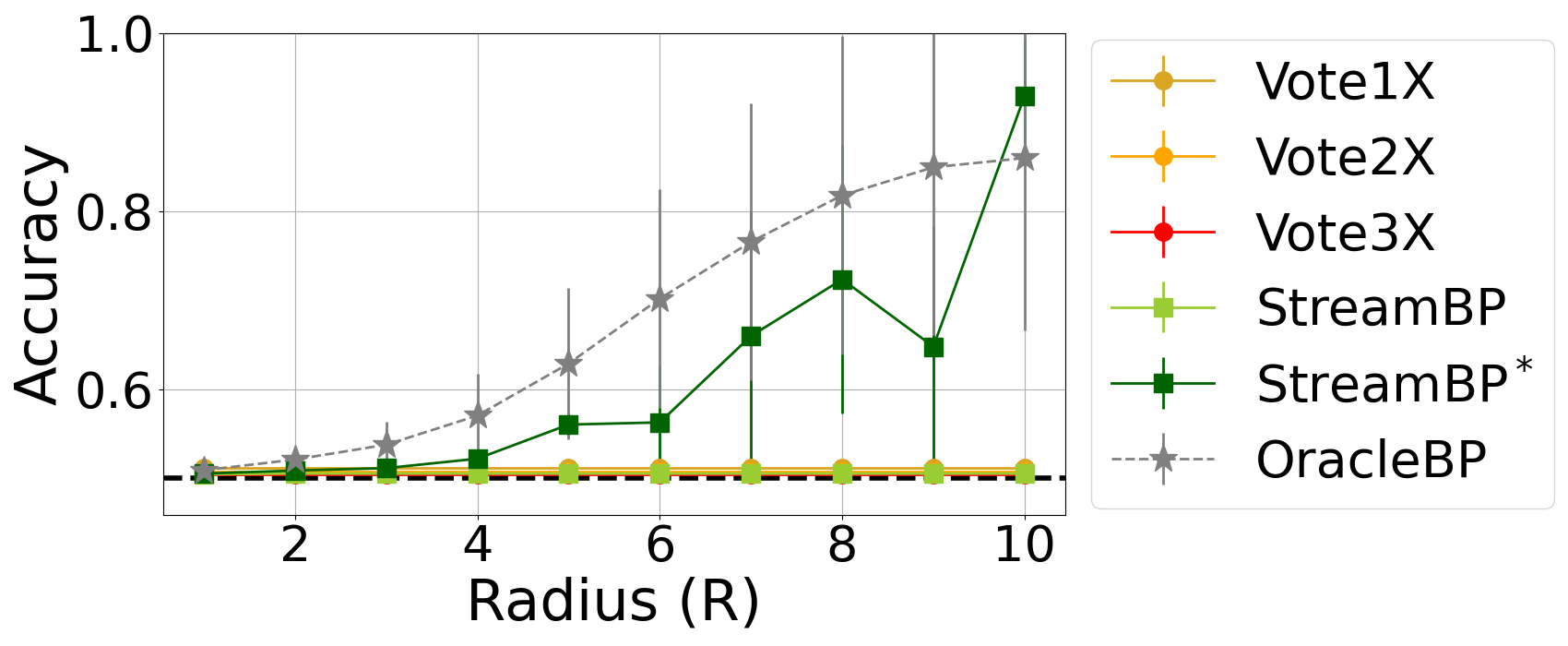} 
\caption{$k=2$, $a=8$, $b=0.1$.}
\end{subfigure}\\

\begin{subfigure}{.33\textwidth}
  \centering
\includegraphics[scale=.16]{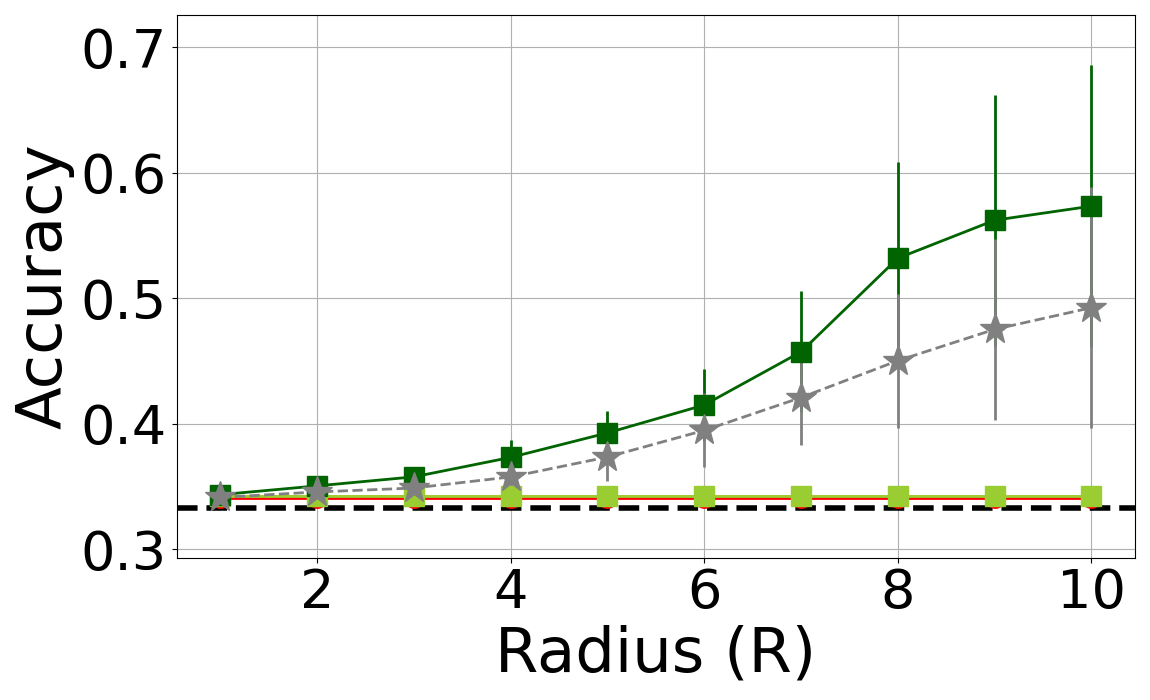} 
\caption{$k=3$, $a=7.5$, $b=0.1$.}
\end{subfigure}
\begin{subfigure}{.33\textwidth}
  \centering
\includegraphics[scale=.16]{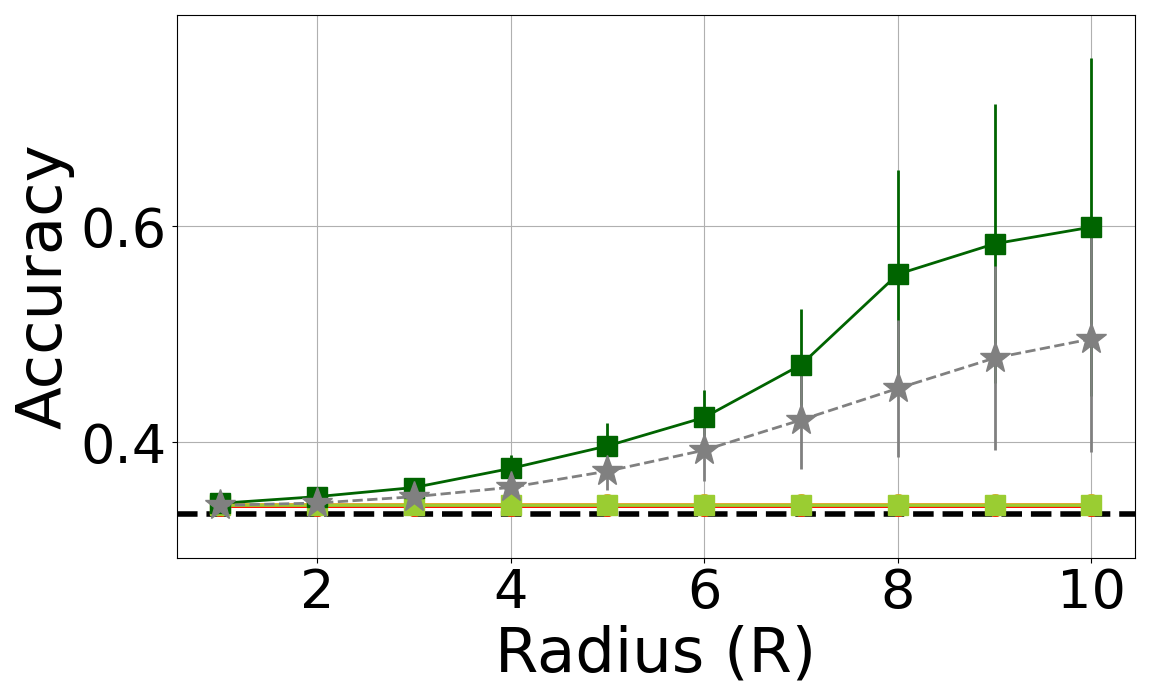} 
\caption{$k=3$, $a=8$, $b=0.1$.}
\end{subfigure}
\begin{subfigure}{.33\textwidth}
  \centering
\includegraphics[scale=.16]{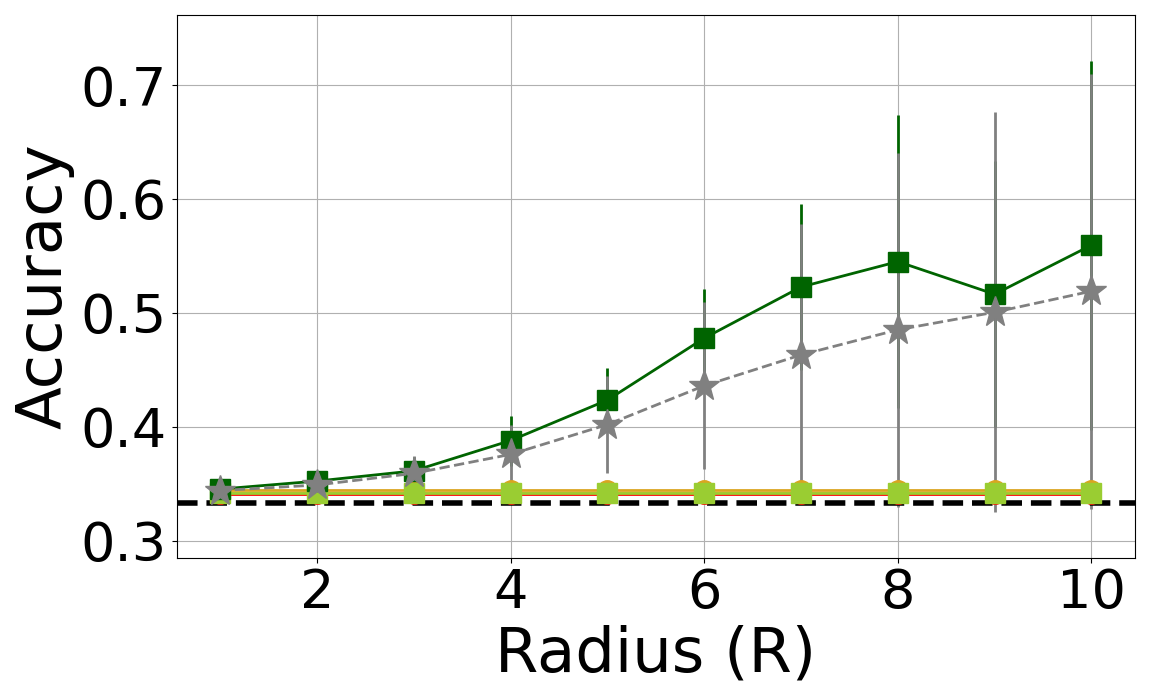} 
\caption{$k=3$, $a=10$, $b=0.3$.}
\end{subfigure}\\
\begin{subfigure}{.33\textwidth}
  \centering
\includegraphics[scale=.16]{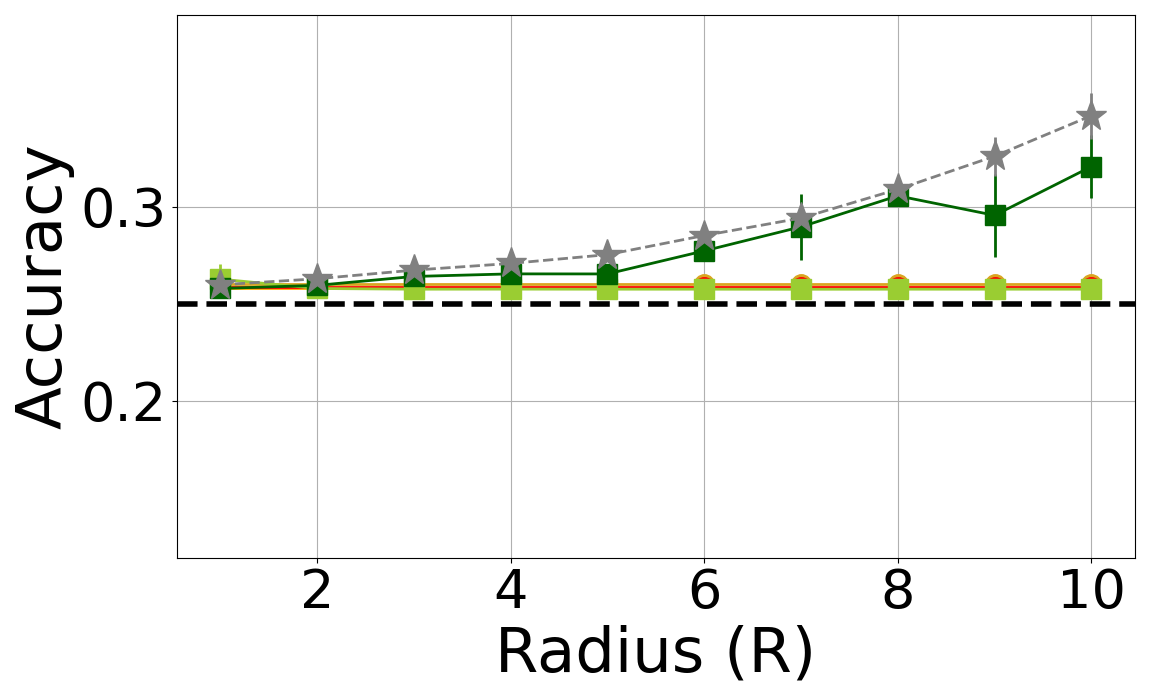} 
\caption{$k=4$, $a=7.5$, $b=0.1$.}
\end{subfigure}
\begin{subfigure}{.33\textwidth}
  \centering
\includegraphics[scale=.16]{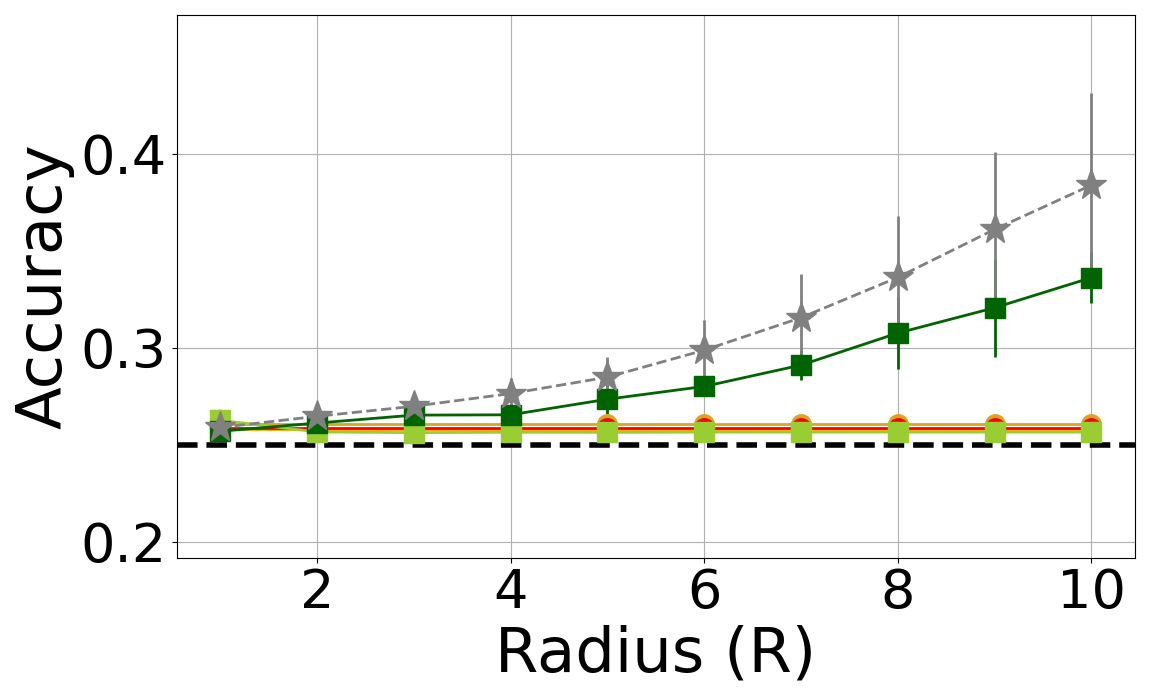} 
\caption{$k=4$, $a=8$, $b=0.1$.}
\end{subfigure}
\begin{subfigure}{.33\textwidth}
  \centering
\includegraphics[scale=.16]{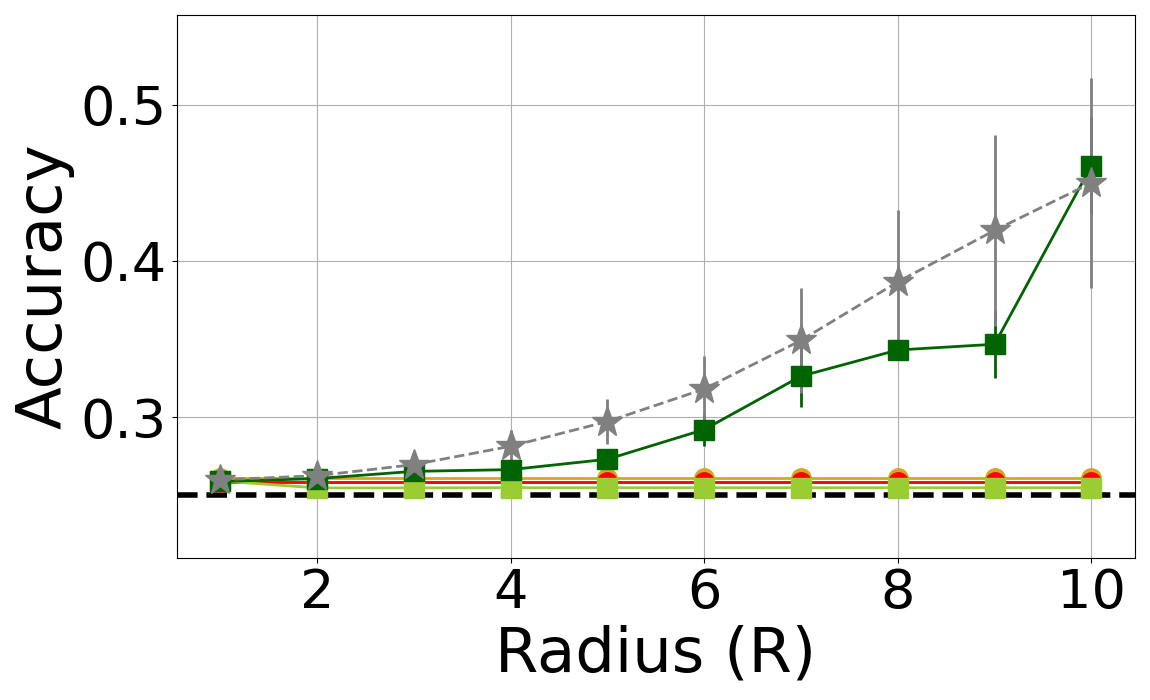} 
\caption{$k=4$, $a=10$, $b=0.3$.}
\end{subfigure}\\
\caption{\label{fig:no-side-info}Results of the experiments on synthetic datasets with 10,000 vertices, and no side information (i.e. $\alpha = 1 - 1/k$).}
\end{figure*}

\label{sec:nosideinfo}

\section{Additional results on real-world datasets}

\begin{figure*}
\begin{subfigure}{.29\textwidth}
  \centering
   \includegraphics[scale=.15]{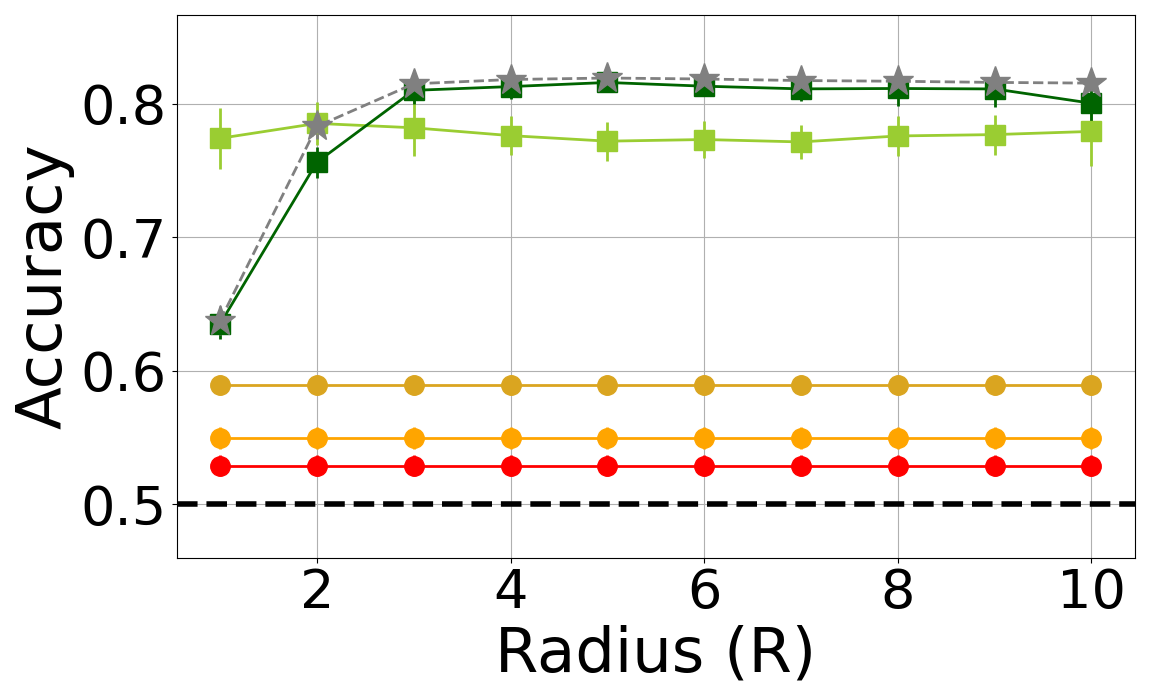} 
\caption{\label{fig:rw_original_beg}Cora ($k = 7$), $\alpha = 0.5$.}
\end{subfigure}
\begin{subfigure}{.29\textwidth}
  \centering
    \includegraphics[scale=.15]{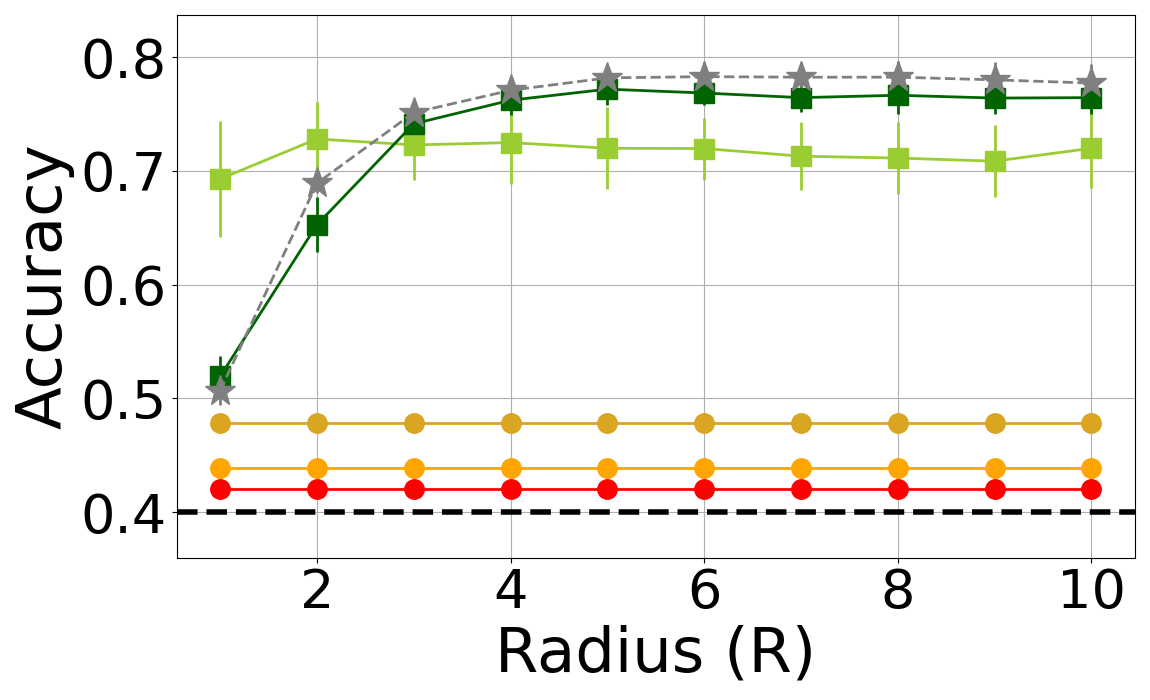} 
\caption{Cora ($k = 7$), $\alpha = 0.6$.}
\end{subfigure}
\begin{subfigure}{.4\textwidth}
  \centering
\includegraphics[scale=.15]{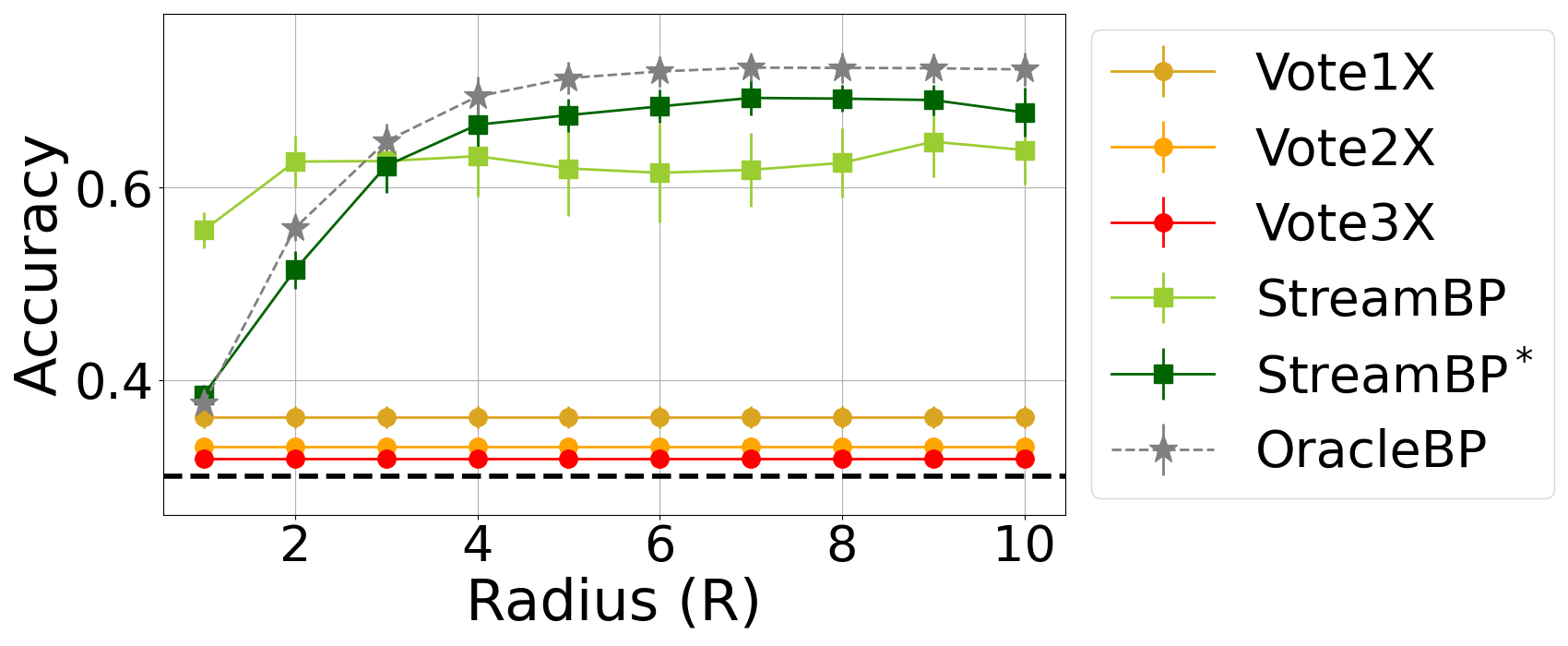} 
\caption{Cora ($k = 7$), $\alpha = 0.7$.}
\end{subfigure} \\
\begin{subfigure}{.29\textwidth}
   \includegraphics[scale=.15]{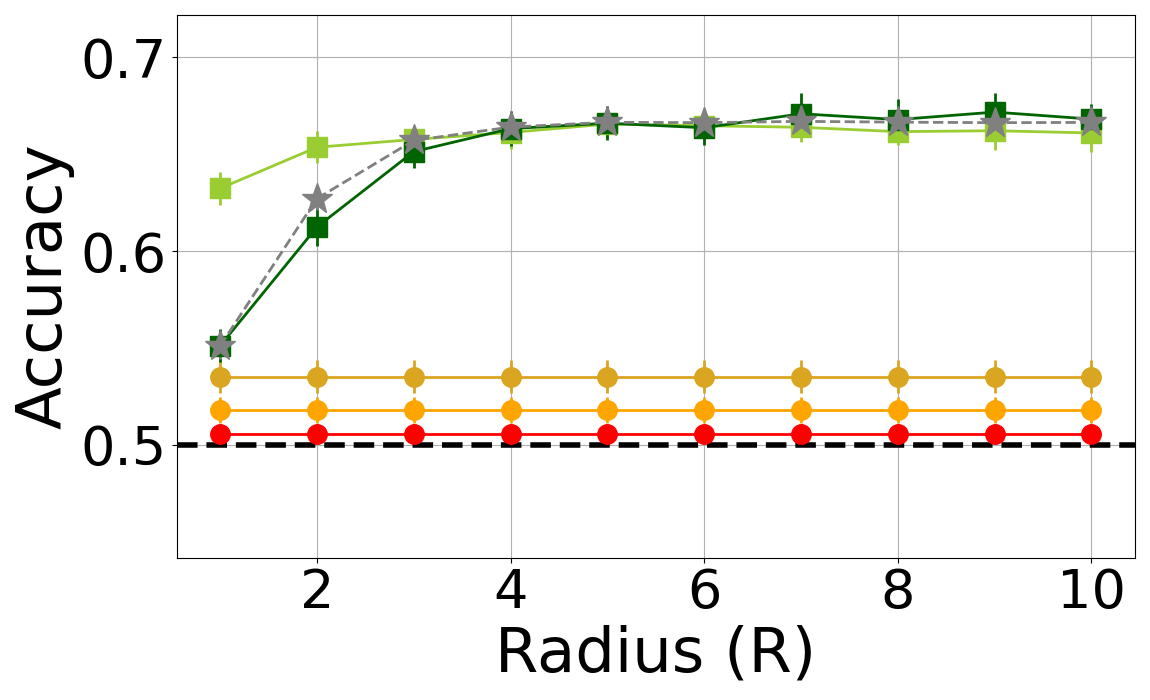} 
\caption{Citeseer ($k = 6$), $\alpha = 0.5 $.}
\end{subfigure}
\begin{subfigure}{.29\textwidth}
  \centering
    \includegraphics[scale=.15]{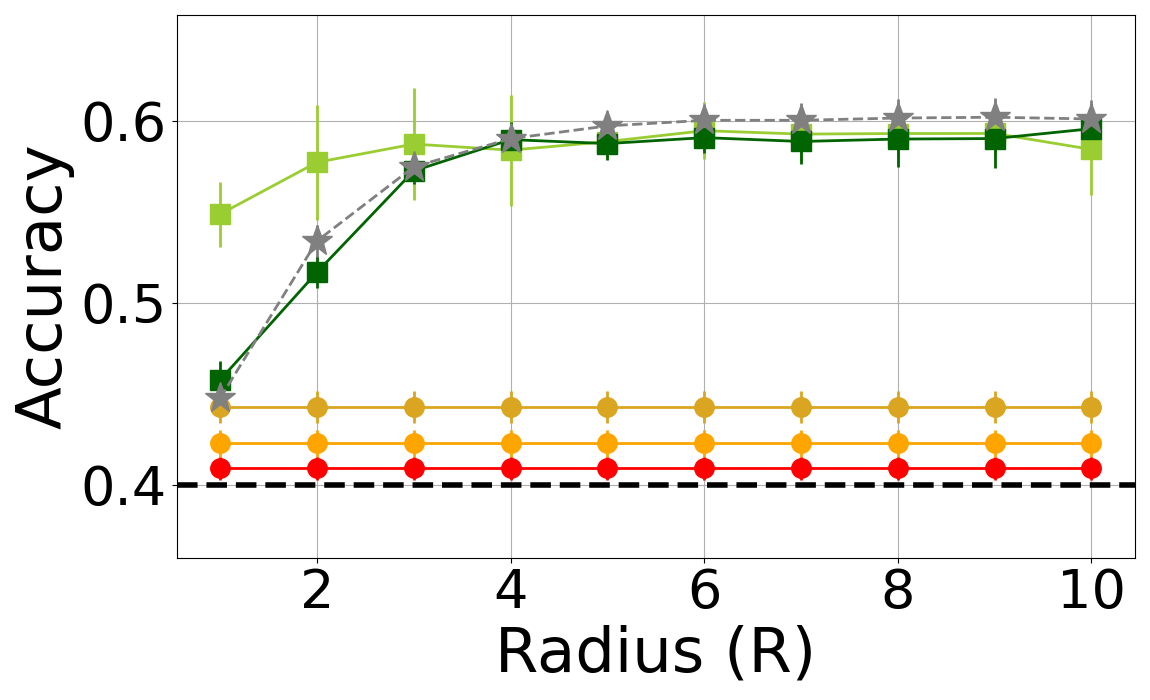} 
\caption{Citeseer ($k = 6$), $\alpha = 0.6$.}
\end{subfigure}
\begin{subfigure}{.4\textwidth}
  \centering
\includegraphics[scale=.15]{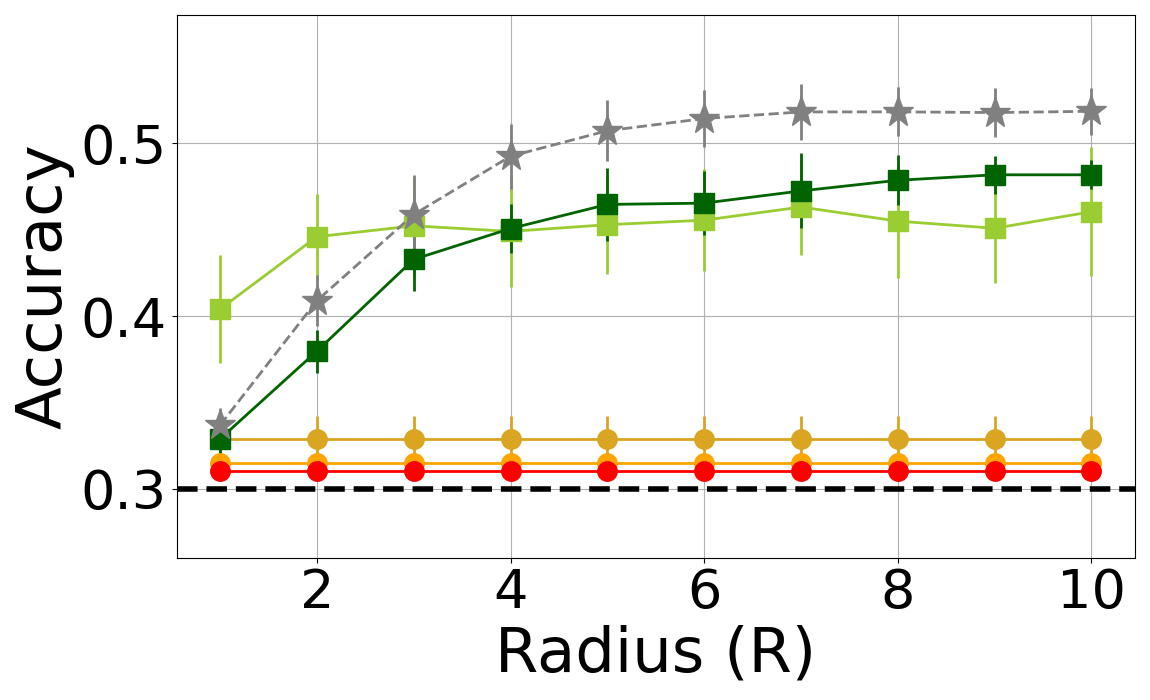} 
\caption{Citeseer ($k = 6$), $\alpha = 0.7$.}
\end{subfigure} \\
\begin{subfigure}{.29\textwidth}
  \centering
   \includegraphics[scale=.15]{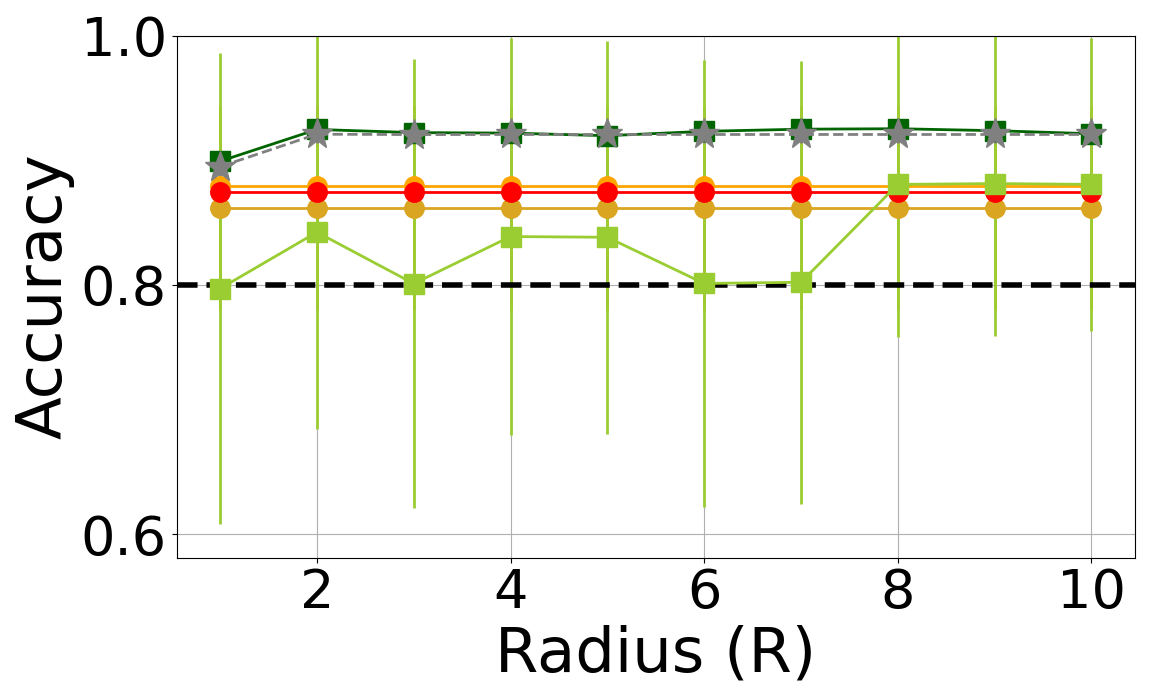} 
\caption{Polblogs ($k = 2$), $\alpha = 0.2$.}
\end{subfigure}
\begin{subfigure}{.29\textwidth}
  \centering
    \includegraphics[scale=.15]{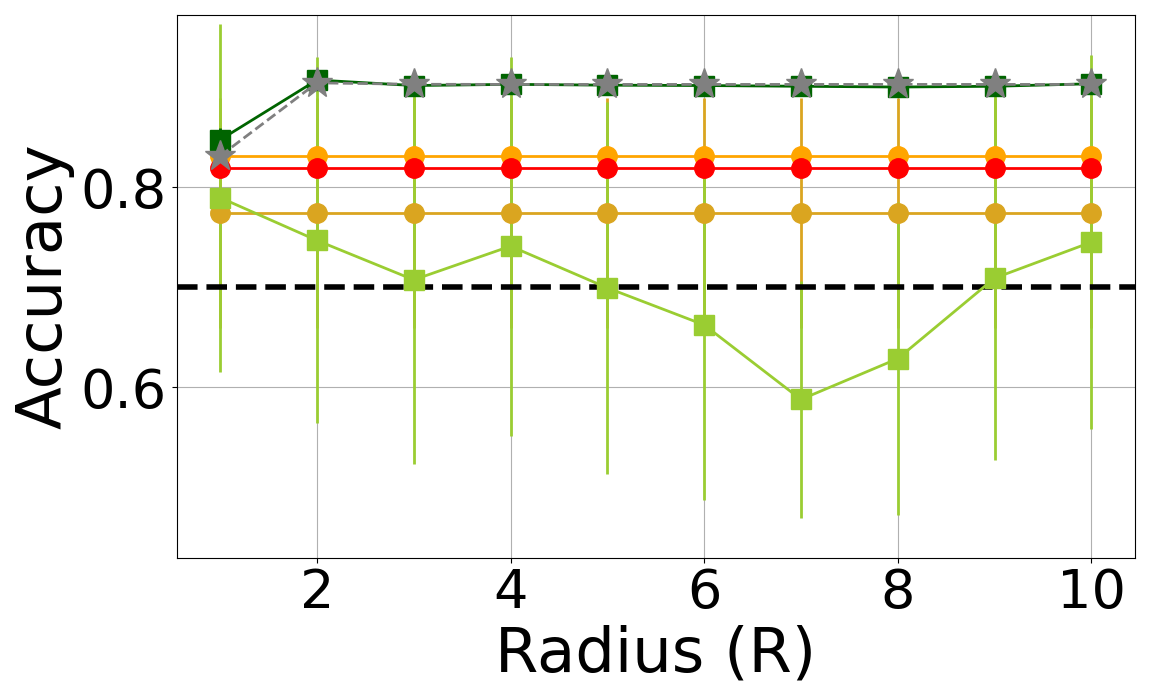} 
\caption{Polblogs ($k = 2$), $\alpha = 0.3$.}
\end{subfigure}
\begin{subfigure}{.4\textwidth}
  \centering
\includegraphics[scale=.15]{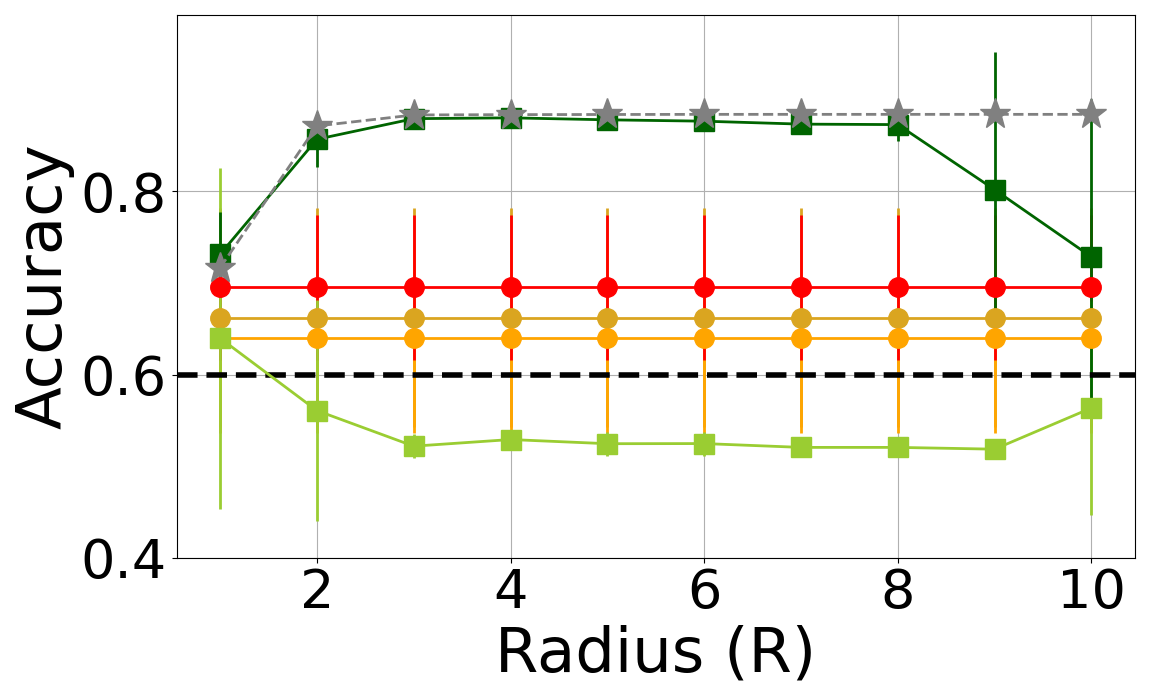} 
\caption{\label{fig:rw_original_end} Polblogs ($k = 2$), $\alpha = 0.4$.}
\end{subfigure} \\
\caption{\label{fig:real_world_datasets_appendix}Results of the experiments on real-world datasets. The black dashed line represents the accuracy of the noisy side information (without using the graph at all), namely $1-\alpha$.}
\end{figure*}

\label{sec:realworld}


Figure~\ref{fig:real_world_datasets_appendix} summarizes the results obtained by running the different algorithms on the datasets Cora, Citeseer, and Polblogs. For each dataset and choice of the radius (used by the algorithms \ouralgo, \boundedalgo, and \offlinealgo) and of the parameter $\alpha$, we run each algorithm $9$ times; each run independently chooses an arrival order of the vertices (uniformly at random among all permutations of the vertices) and the side information (as described in Section~\ref{sec:model}). We note that the accuracy of our streaming algorithm \boundedalgo is comparable to that of the {\em offline} algorithm \offlinealgo, and significantly superior to the accuracy of the voting algorithms. \ouralgo produces high-quality results for the datasets Cora and Citeseer, but behaves erratically in the Polblogs dataset, generally worse than the voting algorithms; that is likely due to the issues discussed in Section~\ref{sec:streambp_star}.

\section{Definitions and technical lemmas}
For completeness, we reproduce a standard lemma establishing that sparse graphs from \SBM \ are locally tree-like. 
\begin{lemma}[\cite{mossel2015reconstruction}]\label{lemma:locally-tree-like}
Let $(X, G) \sim \SSBM(n, 2, a / n, b / n)$ and $R = R(n) = \lfloor \frac{1}{10}\log (n) / \log(2(a + b))\rfloor$. Let $B_R := \{v \in [n]: d_G(1, v) \leq R\}$ be the set of vertices at graph distance at most $R$ from vertex 1, $G_R$ be the restriction of $G$ on $B_R$, and let $X_R = \left\{X_u: u \in B_R\right\}$. Let $T_R$ be a Galton-Watson tree with offspring Poisson$(a+b)/2$ and $R$ generations, and let $\tilde{X}^{(t)}$ be the labelling on the vertices at generation $t$ obtained by broadcasting the bit $\tilde{X}^{(0)}:= X_1$ from the root with flip probability $b / (a + b)$. Let $\tilde{X}_R = \{\tilde{X}_u^{(t)}: t \leq R\}$. Then, there exists a coupling between $(G_R, X_R)$ and $(T_R, \tilde{X}_R)$ such that 
\begin{align*}
    \lim\limits_{n \rightarrow \infty} \P\left( (G_R, X_R) = (T_R, \tilde{X}_R) \right) = 1.
\end{align*}
\end{lemma}
\noindent
For simplicity of presentation, we introduce the following definitions:

\begin{definition}[Labeled branching tree]\label{def:branching-tree}
	For $r \in \NN^+$, $d > 0$, $p \in (0,1)$, let $\P^T_{r,d,p}$ denote the law of a labeled Galton Watson branching tree: this tree has $r$ generations, with the offspring distribution being Poisson with expectation $d$. Each vertex in the tree is associated with a label taking values in $\{1,2\}$. The label at the root node is uniformly distributed over $\{1,2\}$, and labels on the rest of the vertices are obtained by broadcasting the label from the root node with flip probability $p$.  
\end{definition}

\section{Analysis of local streaming algorithms}
\subsection{Proof of Theorem \ref{thm:streaming-local-implies-local}}
With a slight abuse of notations, in this part of the proof, when we refer to $V_r^t(v), E_r^t(v),\Ball_r^t(v)$ and $\cG_r^t(v)$, we do not assume we know the revealing orderings of vertices within. 
For the sake of simplicity, we consider $k = 2$, cases with $k > 2$ can be proven similarly. Let $d = (a + b)/2$ be the average degree, $r \in \mathbb{N}^+$ satisfying 
$$r \geq 2^{2R+3}R\left(1 + (1 - e^{-d})^{-1}\sum\limits_{i=0}^Rd^i\right)d^{2R}e.$$
Let $(\lambda_{r + R}, T_{r + R}) \sim \P^T_{r + R, d, \frac{b}{a + b}}$ as in Definition \ref{def:branching-tree}, with $T_{r + R}$ being the graph and $\lambda_{r + R}$ being the set of labels. By Lemma \ref{lemma:locally-tree-like}, for any $\epsilon > 0$, there exists $n_{\epsilon} \in \mathbb{N}^+$ which is a function of $r,R$ and $\epsilon$, such that for $n \geq n_{\epsilon}$, there exists a coupling of $(\lambda_{r + R}, T_{r + R})$ and $(\tau({V}_{r + R}^n(v_0)), \Ball_{r + R}^n(v_0))$ preserving $u_0$ paired up with $v_0$, and satisfies
\begin{align*}
    \P\left((\lambda_{r + R}, T_{r + R}) \neq (\tau({V}_{r + R}^n(v_0)), \Ball_{r + R}^n(v_0))\right) \leq \epsilon / 2.
\end{align*}
If $\mathcal{G}_{v_0}^n$ is not a subgraph of $\Ball_r^n(v_0)$, then there must exist $v_b \in {D}_r^n(v_0)$ such that $v_b$ belongs to $\cG_{v_0}^n$. This is equivalently saying that there exists an ``information flow" starting at $v_b$, proceeds as vertices are gradually revealed, and finally could reach $v_0$ by the end.


\begin{definition}[Information flow]\label{def:generalized-information-flow}
	Given the graph $G(n) = (V(n), E(n))$, an \emph{information flow} with origin at $v_b$ and end at $v_0$ is defined as a sequence of vertices $p_1, p_2, \cdots, p_l \in V(n)$, such that 
	\begin{enumerate}
		\item $t^{\ast}(p_i) = t_i$, and $t_i < t_{i + 1}$ for all $i \in [l - 1]$.
		\item $V_R^{n}(p_i) \cap V_R^{n}(p_{i + 1}) \neq \emptyset$. Furthermore, $\min\{d(v_0, v): v \in V_R^n(p_{i + 1})\} < \min\{d(v_0, v): v \in V_R^n(p_{i})\}$, for all $i \in [l - 1]$. 
		\item $v_b \in V_R^{n}(p_1)$, $v_0 \in V_R^{n}(p_l)$, $v_0 \notin V_R^{n}(p_i)$ for all $i \in [l - 1]$.
	\end{enumerate}
\end{definition}
Notice that on the event $\Ball_{r + R}^n(v_0)$ is a tree, a necessary condition for $v_b \in \cG_{v_0}^n$ is that there exists an information flow $u_1, u_2, \cdots, u_l$ with origin at $v_b$ and end at $v_0$. For $i \in [l]$, among all vertices on the shortest path connecting $v_0$ and $v_b$, let $v_i$ be a vertex with the smallest graph distance to $u_i$. A moment of thought reveals that we can find an eligible information flow such that $v_1, \cdots, v_l$ are distinct. 


Denote the set of vertices on the path connecting $v_0$ and $v_b$ by $V_p$. Given $l$ and the graph, the number of (unordered) vertex combinations $\{v_1, v_2, \cdots, v_l\}$ is upper bounded by ${r}\choose{l}$. 
For each $x \in \{v_1, v_2, \cdots, v_l\}$, we define the following set:
\begin{align*}
	V_p^x := \left\{ v \in {V}_R^n(x): \nexists v' \in V_p, v' \neq x, d(v, v') < d(v, x) \right\}.	
\end{align*}
Given $\{v_1, v_2, \cdots, v_l\}$, the total number of possible unordered vertex combinations $\{u_1, u_2, \cdots, u_l\}$ is upper bounded by $\prod_{i=1}^l |  V_p^{v_i} |$. Furthermore, given unordered set $\{u_1, u_2, \cdots, u_l\}$, by definition one can specify their relative ordering $u_1, u_2, \cdots, u_l$ by sorting the following distances: $\min \{d(v_0, v): v \in V_R^n(u_i)\}$. Finally, for such combinations to exist, one must have $\lfloor {r}/{2R} \rfloor \leq l \leq r$.  

As a result, given $\Ball_{r + R}^n(v_0)$, if it is a tree, then the conditional probability of $\mathcal{G}_{v_0}^n$ not being a subgraph of $\Ball_r^n(v_0)$ is upper bounded by
\begin{align*}
\sum\limits_{v_b \in {V}(n)}\ind\{v_b \in {D}_r^n(v_0)\} \sum\limits_{l = \lfloor {r / 2R}\rfloor}^{r}\sum\limits_{\{v_1,v_2,\cdots,v_l\}}\prod\limits_{i=1}^l |  V_p^{v_i} |\times \frac{1}{l!}.
\end{align*}
Let $\P_{bt}(\cdot)$ denote the probability distribution over $(\lambda_{r + R}, T_{r + R})$, and let $\E_{bt}(\cdot)$ be the corresponding notation for taking expectation under that distribution. Then we have
\begin{align}
    &\P(\cG_{v_0}^n \mbox{ is not a subgraph of } \Ball_r^n(v_0))\nonumber\\
    \leq & \P((\lambda_{r + R}, T_{r + R}) \neq (\tau({V}_{r + R}^n(v_0)), \Ball_{r + R}^n(v_0))) + \nonumber \\
    &\P\left(\cG_{v_0}^n \mbox{ is not a subgraph of } \Ball_r^n(v_0), (\lambda_{r + R}, T_{r + R}) = (\tau({V}_{r + R}^n(v_0)), \Ball_{r + R}^n(v_0))\right)\nonumber\\
    \leq & \epsilon / 2 + \sum\limits_{v_b \in {V}(n)} \E_{bt}\left[\ind\{v_b \in {D}_r^n(v_0)\} \sum\limits_{l = \lfloor {r / 2R}\rfloor}^{r}\sum\limits_{\{v_1,v_2,\cdots, v_l\}}\prod\limits_{i=1}^l | V_p^{v_i} |\times \frac{1}{l!}\right] \nonumber \\
    \overset{(i)}{\leq} & \epsilon / 2 + \sum\limits_{v_b \in {V}(n)}\E_{bt}\left[\ind\{v_b \in {D}_r^n(v_0)\}\sum\limits_{l = \lfloor r / 2R\rfloor}^r \sum\limits_{\{v_1,v_2,\cdots,v_l\}} \frac{1}{l!} \times \underbrace{ \left(1 + (1 - e^{-d})^{-1}\sum\limits_{i=1}^Rd^i\right)^l }_{C(d,R)^l}\right] \nonumber \\
    \leq & \epsilon / 2 + \sum\limits_{l = \lfloor r / 2R\rfloor}^r {r \choose l}\frac{1}{l!} C(d,R)^l d^r,\label{eq-thm1-1}
\end{align}
where \emph{(i)} uses the fact that conditioning on $v_b \in D_r^n(v_0)$, the path connecting $v_0$ and $v_b$, and the choice of $v_1, \cdots, v_l$, in a Poisson branching process, the random variables $\{| V_p^{v_i} |: i \in [l]\}$ are conditionally independent. When $R = 1$, we have the conditional distribution $\P(| V_p^{v_i}  | = \cdot)\overset{}{=}\P( X = \cdot \mid X \geq 1)$, with $X \sim\mbox{Poisson}(d)$, thus the bound follows. Upper bound when $R > 1$ can be derived similarly. 

By Stirling's formula, there exists numerical constants $C_1$ and $C_2$, such that for any positive integer $m$,
\begin{align*}
    C_2 \sqrt{ m}\left(\frac{m}{e}\right)^m \leq m! \leq C_1 \sqrt{ m}\left(\frac{m}{e}\right)^m.
\end{align*}
Then application of Stirling's formula gives us
\begin{align}
    & \sum\limits_{l = \lfloor r / 2R \rfloor}^r{r\choose l}\frac{1}{l!}C(d,R)^ld^r \leq \sum\limits_{l = \lfloor r / 2R \rfloor}^r \frac{2^rd^r}{C_2}\left(\frac{e}{l}\right)^lC(d,R)^l \leq \frac{2^{2R} d^{2R}}{C_2 } \sum\limits_{l = \lfloor r / 2R \rfloor}^{\infty} \left( \frac{2^{2R + 2}R C(d,R) d^{2R}e}{r} \right)^l. \label{eq:39}
\end{align}
where we use the fact that $r / 2R - 1 \leq l \leq r / 2R $ and $2R/(r - 2R) \leq 4R / r$ for all $r \geq 2^{2R+3}RC(d,R)d^{2R}e$. Furthermore, with $r$ having value exceeding this threshold,$\left( \frac{2^{2R + 2}R C(d,R) d^{2R}e}{r} \right) \leq \frac{1}{2}$, thus
\begin{align}
	\frac{2^{2R} d^{2R}}{C_2 } \sum\limits_{l = \lfloor r / 2R \rfloor}^{\infty} \left( \frac{2^{2R + 2}R C(d,R) d^{2R}e}{r} \right)^l \leq &  \frac{2^{2R + 1} d^{2R}}{C_2 }  \left( \frac{2^{2R + 2}R C(d,R) d^{2R}e}{r} \right)^{\lfloor r / 2R \rfloor} \nonumber \\
	 = & C_R\left(\frac{2^{2R+2}RC(d,R)d^{2R}e}{r}\right)^{\lfloor r / 2R \rfloor}, \label{eq:15}
\end{align}
where $C_R$ is a constant that depends only on $d,R$. Then there exists an $r_{\epsilon} \in \mathbb{N}^+$, such that for $r \geq r_{\epsilon}$, 
\begin{align}
	C_R\left(\frac{2^{R + 2}RC(d,R)d^{2R}e}{r}\right)^{\lfloor r / 2R\rfloor} \leq \epsilon / 2.\label{eq:16}
\end{align}
Combining equations \eqref{eq-thm1-1}, \eqref{eq:39}, \eqref{eq:15} and \eqref{eq:16} gives us $\P(\cG_{v_0}^n \mbox{ is a subgraph of } \Ball_r^n(v_0)) \leq \epsilon$ for large enough $n$ and $r$. Note that the choice of $r_{\epsilon}$ indeed only depends on $R,d$ and $\epsilon$. Having decided the value of $r_{\epsilon}$, what remains to be done is to select $n_{\epsilon}$ large enough to accommodate with this choice of $r_{\epsilon}$ such that for $n \geq n_{\epsilon}$, the $r_{\epsilon} + R$ neighborhood of $v_0$ is locally tree-like with high probability as in Lemma \ref{lemma:locally-tree-like}. Note that since the choice of $r_{\epsilon}$ depends only on $R,d$ and $\epsilon$, the choice of $n_{\epsilon}$ depends only on $R,d$ and $\epsilon$. This finishes the proof of the first part of Theorem \ref{thm:streaming-local-implies-local}. 

As for the second part of the theorem, we simply reverse the direction of information flow analysis and everything else remains the same.

\subsection{Proof of Corollary \ref{cor:streaming-without-sideinfo}}
By Theorem \ref{thm:streaming-local-implies-local}, for any $\epsilon > 0$, there exists $n_{\epsilon}, r_{\epsilon} \in \mathbb{N}^+$, such that for all $n \geq n_{\epsilon}$,
\begin{align*}
	& \E\left[ \max\limits_{\pi \in \fS_k} \frac{1}{n}\sum\limits_{i = 1}^n \ind\{ \mathcal{A}(i; G(n)) = \pi \circ \tau(i)\} \right] \\
	\leq & \E\left[ \max\limits_{\pi \in \fS_k} \frac{1}{n}\sum\limits_{i = 1}^n \left( \ind\{ \mathcal{A}(i; G(n)) = \pi \circ \tau(i), \cG_i^n \mbox{ is a subgraph of }\Ball_{r_{\epsilon}}^n(i)\} + \ind\{\cG_i^n \mbox{ is not a subgraph of }\Ball_{r_{\epsilon}}^n(i)\}\right) \right] \\
	\leq & \epsilon + \sup\limits_{\mathcal{A}'}\E\left[ \max\limits_{\pi \in \fS_k} \frac{1}{n}\sum\limits_{i = 1}^n \ind\{ \mathcal{A}'(i; G(n)) = \pi \circ \tau(i)\} \right],
\end{align*}
where in the last line $\cA'$ is taken over the family of $r_{\epsilon}$-local algorithms. Lemma \ref{lemma:locally-tree-like} and Corollary \ref{cor:concentrate} imply that the last line above has limiting supremum no larger than $1 / k + \epsilon$ as $n \rightarrow \infty$ (for detailed arguments, see \cite{kanade2016global}). Since $\epsilon$ is arbitrary, then the corollary directly follows.

\section{Analysis of local streaming algorithms with summary statistics}
For the sake of simplicity, we assume $m = 1$. The extension to $m \geq 2$ is straightforward.
\subsection{An auxiliary algorithm}
To prove Theorem \ref{thm:local-with-summary-stat-is-trivial}, we first introduce an auxiliary algorithm (Algorithm \ref{alg:aux-local-alg-with-summary-stat}) which is close to a local algorithm. Then we show that the algorithm with summary statistics can be well approximated by the proposed auxiliary algorithm (Lemma \ref{lemma:two-algs-are-close}). Finally, we show that the auxiliary algorithm can not achieve non-trivial estimation accuracy (Lemma \ref{lemma:aux-alg-is-trivial}).  

Let $\delta \in (0,1)$ be a small constant independent of $n$, and let $n_h = |\{v \in V(n): \tau(v) = h\}|$, $h \in [k]$. As shown in Algorithm \ref{alg:aux-local-alg-with-summary-stat}, we run a global algorithm up to time $\lceil \delta n \rceil$, followed by a local algorithm. To represent the information up to time $\lceil \delta n \rceil$ and the size of communities at time $n$, we introduce the following sigma-algebra:
 \begin{align*}
 	\mathcal{F}_a := \sigma \left\{ G(\lceil \delta n \rceil), (i,v(i), \xi_{v(i)}, \tau(v(i))), w_{v(i)}^{i - 1}, e_{v(j)v(s)}^{j \vee s - 1}, n_h: 1 \leq i,j,s \leq \lceil \delta n \rceil,  (v(j), v(s)) \in E(\lceil \delta n \rceil),  h \in [k]  \right\}.
 \end{align*} 
In the auxiliary algorithm, we attach a number $b_i^t$ to node $i$ at time $t$,  with the same initialization as $w$: $b_i^{t_*(i) - 1} = w_i^{t_*(i) - 1}$. Similarly, we attach to edge $(v(i), v(j))$ at time $t$ a number $c_{ij}^t = c_{ji}^t$ with initialization $c_{ij}^{t_*(i) \vee t_*(j) - 1} = e_{ij}^{t_*(i) \vee t_*(j) - 1}$. Note that $w_i^{t_*(i) - 1}$ and $e_{ij}^{t_*(i) \vee t_*(j) - 1}$ are as defined in the original algorithm. Denote the vector containing all $b$'s attached to vertices in $\localV_s$ at time $t$ by $\bb_s^t$, and the vector containing all $c$'s attached to edges in $\localE_s$ at time $t$ by $\bc_s^t$. Similarly, let $\bw_s^t$ and $\be_s^t$ denote the restrictions of $\bw^t$ and $\be^t$ to $\localV_s$ and $\localE_s$, respectively. For $\lceil \delta n \rceil \leq t \leq [n]$, let
\begin{align*}
	\mathsf{average}_b(t) = \frac{1}{|V(t)|} \sum_{v \in V(t)} b_{v}^t, \qquad \mathsf{average}_c(t) = \frac{1}{|E(t)|}\sum_{(j,k) \in E(t)} c_{jk}^t.
\end{align*}
Then we introduce the following auxiliary algorithm:
 \begin{breakablealgorithm}\label{alg:aux-local-alg-with-summary-stat}
	\caption{Auxiliary algorithm with summary statistics}
	\begin{algorithmic}[1]
		\For{$i \in V(\lceil \delta n \rceil)$}
			\State $b_i^{\lceil \delta n \rceil} = w_i^{\lceil \delta n \rceil}$
		\EndFor
		\For{$(i, j) \in E(\lceil \delta n \rceil)$}
			\State $c_{ij}^{\lceil \delta n \rceil} = c_{ji}^{\lceil \delta n \rceil} = e_{ij}^{\lceil \delta n \rceil}$
		\EndFor
		\State $\bar{b}^{\lceil \delta n \rceil} \gets \mathsf{average}_b(\lceil \delta n \rceil)$
		\State $\bar{c}^{\lceil \delta n \rceil} \gets \mathsf{average}_c(\lceil \delta n \rceil)$
		\For{$t = \lceil \delta n \rceil + 1, \lceil \delta n \rceil + 2, \cdots, n$}
			\For{$i \in \localV_t$}
				\State $b_i^t \gets F_w^t(\bb_t^{t - 1}, \bc_t^{t - 1}, \E[\bar{b}^{t - 1} \mid \mathcal{F}_a], \E[\bar{c}^{t - 1}\mid \mathcal{F}_a] \mid i)$ 
			\EndFor
			\For{$(i, j) \in \localE_t$}
				\State $c_{ij}^t \gets F_e^t(\bb_t^{t - 1}, \bc_t^{t - 1}, \E[\bar{b}^{t - 1}\mid \mathcal{F}_a], \E[\bar{c}^{t - 1}\mid \mathcal{F}_a] \mid i, j)$
			\EndFor
			\State $\bar{b}^t \gets\mathsf{average}_b(t)$
			\State $\bar{c}^t \gets\mathsf{average}_c(t)$
		\EndFor
		\For{$t = 1,2,\cdots,n$}
			\State Output $\htau(b_t ^n)$ as an estimate for $\tau(t)$.
		\EndFor
	\end{algorithmic}
\end{breakablealgorithm}
We can prove the following lemmas regarding Algorithm \ref{alg:aux-local-alg-with-summary-stat}:
\begin{lemma}\label{lemma:two-algs-are-close}
Under the conditions stated in Theorem \ref{thm:local-with-summary-stat-is-trivial}, Algorithm \ref{alg:aux-local-alg-with-summary-stat} and the original algorithm proposed in Section \ref{sec:summary-stat} are asymptotically equivalent, in the sense that as $n \rightarrow \infty$,
\begin{align*}
\frac{1}{n}\sum\limits_{i =1 }^n|w_i^n - b_i^n| + \frac{1}{|E(n)|}\sum\limits_{(i,j) \in E(n)}|e_{ij}^n - c_{ij}^n| \overset{P}{\rightarrow} 0.	
\end{align*}
\end{lemma}
\begin{lemma}\label{lemma:aux-alg-is-trivial}
Under the conditions stated in Theorem \ref{thm:local-with-summary-stat-is-trivial}, for any $\epsilon > 0$, the following holds:
\begin{align*}
	\limsup\limits_{\delta \rightarrow 0^+}\limsup\limits_{n \rightarrow \infty}\E\left[\max\limits_{\pi\in \fS_k} \frac{1}{n} \sum\limits_{i = 1}^n \ind\{\htau(b_i^n + \epsilon U_i) = \pi \circ \tau(i)\}\right] = \frac{1}{k}. 
\end{align*}
\end{lemma}
\noindent
We defer the proofs of Lemma \ref{lemma:two-algs-are-close} and \ref{lemma:aux-alg-is-trivial} to later parts of the appendix. With these two lemmas, we are able to prove Theorem \ref{thm:local-with-summary-stat-is-trivial}.

\subsection{Proof of Theorem \ref{thm:local-with-summary-stat-is-trivial}}
For any $\delta \in(0,1)$, using Lemma \ref{lemma:two-algs-are-close}, we conclude that for any $\zeta \in(0,1)$, $n$ large enough, with probability at least $1 - \zeta$, $\frac{1}{n} \sum_{i = 1}^n |w_i^n - b_i^n| \leq \zeta^2$. If this happens, then $\# \{ i: |w_i^n - b_i^n| \geq \zeta \} \leq \zeta n$. For $\epsilon$ given in the theorem, let $\psi_{\epsilon}(\zeta) := \mbox{TV}(\epsilon Z_1, \zeta + \epsilon Z_2)$ where $Z_1, Z_2 \sim \Unif[-1,1]$ and TV stands for the total variation distance between probability distributions. One can verify that $\psi_{\epsilon}(\zeta) \rightarrow 0$ as $\zeta \rightarrow 0$. Then we have
\begin{align}
	& \E\left[\max\limits_{\pi\in \fS_k} \frac{1}{n} \sum\limits_{i = 1}^n \ind\{\htau(w_i^n + \epsilon U_i) = \pi \circ \tau(i)\}\right] \nonumber\\
	\leq & \E\left[\max\limits_{\pi\in \fS_k} \frac{1}{n} \sum\limits_{i = 1}^n \ind\{\htau(w_i^n + \epsilon U_i) = \pi \circ \tau(i)\} \ind\{\frac{1}{n} \sum\limits_{i = 1}^n |w_i^n - b_i^n|\leq \zeta^2\}\right] + \zeta \nonumber \\
	\leq & \E\left[\max\limits_{\pi\in \fS_k} \frac{1}{n} \sum\limits_{i = 1}^n \ind\{\htau(w_i^n + \epsilon U_i) = \pi \circ \tau(i), |w_i^n - b_i^n| \leq \zeta\} \ind\{\frac{1}{n} \sum\limits_{i = 1}^n |w_i^n - b_i^n|\leq \zeta^2\}\right] + 2\zeta := \triangle. \label{eq:52}
\end{align}
Conditioning on given values of $w_i^n$ and $b_i^n$, we may bound the total variation distance between $w_i^n + \epsilon U_i$ and $b_i^n + \epsilon U_i$. Specifically, conditional on all $w_i^n$ and $b_i^n$, there exists $U_i' \iidsim \Unif[-1,1]$ for $i \in [n]$, independent of $w_i^n$ and $b_i^n$, such that with probability at least $1 - \psi_{\epsilon}(|w_i^n - b_i^n|)$, we have $w_i^n + \epsilon U_i = b_i^n + \epsilon U_i'$. Note that $U_i'$ is independent of $w_i^n,b_i^n$, but dependent on $U_i$. Then we have
\begin{align*}
	\triangle \leq & \E\left[\max\limits_{\pi\in \fS_k} \frac{1}{n} \sum\limits_{i = 1}^n \ind\{\htau(b_i^n + \epsilon U_i') = \pi \circ \tau(i), |w_i^n - b_i^n| \leq \zeta\} \ind\{\frac{1}{n} \sum\limits_{i = 1}^n |w_i^n - b_i^n|\leq \zeta^2\}\right] + 2\zeta + \psi_{\epsilon}(\zeta) \\
	\leq & \E\left[ \max\limits_{\pi\in \fS_k} \frac{1}{n} \sum\limits_{i = 1}^n \ind\{\htau(b_i^n + \epsilon U_i') = \pi \circ \tau(i)\} \right] + 2\zeta + \psi_{\epsilon}(\zeta).
\end{align*}
Since $\zeta$ can be arbitrarily small, 
\begin{align*}
	\limsup\limits_{n \rightarrow \infty}\E\left[\max\limits_{\pi\in \fS_k} \frac{1}{n} \sum\limits_{i = 1}^n \ind\{\htau(w_i^n + \epsilon U_i) = \pi \circ \tau(i)\}\right] \leq \limsup\limits_{n \rightarrow \infty}\E\left[ \max\limits_{\pi\in \fS_k} \frac{1}{n} \sum\limits_{i = 1}^n \ind\{\htau(b_i^n + \epsilon U_i') = \pi \circ \tau(i)\} \right].
\end{align*}
This holds for any value of $\delta$. Taking $\delta \rightarrow 0^+$ then using Lemma \ref{lemma:aux-alg-is-trivial} finishes the proof of Theorem \ref{thm:local-with-summary-stat-is-trivial}.

\subsection{Proof of Lemma \ref{lemma:two-algs-are-close}}\label{sec:proof-of-lemma2}
For simplicity of presentation, in the proof of this lemma we drop the edge variables (i.e., setting $e_{ij}^t = c_{ij}^t = 0$), and consider only the vertex variables. The proof involving edge variables can be conducted almost identically. By the Lipschitz continuity assumption, as the $(t + 1)$-th vertex joins we have
\begin{align}
	\mbox{for all }l \in \localV_{t + 1}, \qquad & |w_l^{t + 1} - b_l^{t + 1}|  \leq L_F \sum\limits_{i \in \localV_{t + 1}}|w_i^t - b_i^t| + L_F|\bar{w}^t - \bar{b}^t|  + L_F|\bar{b}^t - \E[\bar{b}^t | \mathcal{F}_a]|, \label{eq-4.2-1} \\
	\mbox{for all }l \notin \localV_{t + 1}, \qquad & |w_l^{t + 1} - b_l^{t + 1}| = |w_l^{t} - b_l^{t}|.\label{eq-4.2-3}
\end{align}
%
Using equations \eqref{eq-4.2-1} and \eqref{eq-4.2-3}, we have
\begin{align}\label{eq-4.2-5}
	& |\bar{w}^{t + 1} - \bar{b}^{t + 1}| 	\leq |\bar{w}^t - \bar{b}^t| + \frac{1}{t +1} \sum\limits_{i \in \localV_{t + 1}} |w_i^{t + 1} - b_i^{t + 1}| + \frac{1}{t + 1} \sum\limits_{i \in \localV_{t + 1}} |w_i^t - b_i^t| \nonumber\\
	\leq & \left(1 + \frac{L_F|\localV_{t + 1}|}{t + 1}\right)|\bar{w}^t - \bar{b}^t|  + \frac{L_F|\localV_{t + 1}| + 1}{t + 1}\sum\limits_{i \in \localV_{t + 1}}|w_i^t - b_i^t|  + \frac{L_F|\localV_{t+1}|}{t + 1}|\bar{b}^t - \E[\bar{b}^t | \mathcal{F}_a]|.
\end{align}
For $t \in [n]$, let $\bd_t \in \mathbb{R}^{n + 1}$ with entries indexed from $0$ to $n$. The first entry is set to $|\bar{w}^t - \bar{b}^t|$, and for $i \in [n]$, the entry with index $i$ is set to $|w_{i}^t - b_{i}^t|$. By definition for all $1 \leq s \leq \lceil \delta n \rceil$, $\bd_s$ has only zero entries. 

 For $S_1, S_2, S \subseteq \{0\} \cup [n]$, let $\mathbf{1}_{S} \in \RR^{n + 1}$ be a vector with entries indexed by $0$ to $n$, and the entry with index $\alpha$ is 1 if and only if $\alpha \in S$ otherwise it is zero. For simplicity, let $\mathbf{1}_{t} = \mathbf{1}_{\{t\}}$, and let $Q(S_1, S_2) := \mathbf{1}_{S_1}\mathbf{1}_{S_2}^T \in \mathbb{R}^{(n + 1) \times (n + 1)}$. 
Then define the following matrices:
\begin{align*}
	 & A_t := Q(\localV_t, \{0\} \cup \localV_t) + \frac{|\localV_t| + 1}{t}Q(\{0\}, \{0\} \cup \localV_t ), \\
	 &  A_{t_1, t_2 }  :=  Q(\localV_{t_1} , \{0\} \cup \localV_{t_2}) + \frac{|\localV_{t_1}| + 1}{t_1}Q(\{0\}, \{0\} \cup \localV_{t_2} ).
\end{align*}
 Without loss of generality we may assume $L_F \geq 1$, then combining equations \eqref{eq-4.2-1}, \eqref{eq-4.2-3} and \eqref{eq-4.2-5} gives
 \begin{align*}
 \bd_{t + 1} \leq \left(I + L_F A_{t + 1}\right) \bd_t + L_F(|\bar{b}^t - \E[\bar{b}^t | \mathcal{F}_a]| )A_{t+1}  \mathbf{1}_{v(t+1)}	,
 \end{align*}
here ``$\leq$" refers to element-wise comparison.
Let 
$$\tilde{\bd}_{\lceil \delta n \rceil} = \bd_{\lceil \delta n \rceil} = \vec{0}, \qquad \tilde{\bd}_{t+1} = (I + L_F A_{t+1})\tilde{\bd}_t + (|\bar{b}^t - \E[\bar{b}^t | \mathcal{F}_a]|)L_FA_{t + 1} \mathbf{1}_{v(t+1)},$$ then we have $\bd_t \leq \tilde{\bd}_t$ for all $t \geq \lceil \delta n \rceil$. Furthermore, we have the following decomposition with $h_t$ defined in equation \eqref{eq:19}:
\begin{align}
	\langle \bd_n, \vec{1} \rangle \leq \langle \tilde{\bd}_n, \vec{1} \rangle = \sum\limits_{t = \lceil \delta n \rceil + 1}^{n - 1} h_t |\bar{b}^t - \E[\bar{b}^t | \mathcal{F}_a]|. \label{eq:59}
\end{align}
Before elaborating on the definition of $h_t$, we state the following lemma without proof. Notice that the proof is nothing but basic linear algebra.
\begin{lemma}\label{lemma-matrix-calculation}
	For $n \geq t_m > t_{m - 1} > \cdots > t_1 \geq \lceil \delta n \rceil + 1$, we have 
	\begin{align*}
		A(t_m, t_{m - 1}, \cdots, t_1) := A_{t_m}A_{t_{m - 1}} \cdots A_{t_1}	 =  \prod\limits_{k = 1}^{m - 1} \left(\frac{|\localV_{t_k}| + 1}{t_k}  + |\localV_{t_k} \cap \localV_{t_{k + 1}}| \right)A_{t_m, t_1}.
	\end{align*}
\end{lemma}
\noindent
Applying Lemma \ref{lemma-matrix-calculation}, we have for all
$\lceil \delta n \rceil + 1 \leq t \leq n - 1$, 
\begin{align}\label{eq:19}
	h_t =&  \sum\limits_{n \geq t_m > \cdots > t_1 = t + 1} \langle L_F^m A(t_m, t_{m - 1}, \cdots, t_1)\mathsf{1}_{v(t + 1)}, \vec{1} \rangle \nonumber \\
	=& \sum\limits_{n \geq t_m > \cdots > t_1 = t + 1} L_F^m \left( |\localV_{t_m}|  + \frac{|\localV_{t_m}| + 1}{t_m}  \right) \times  \prod\limits_{k = 1}^{m - 1}\left(|\localV_{t_k} \cap \localV_{t_{k + 1}}| + \frac{|\localV_{t_k}| + 1}{t_k} \right). 
\end{align}
For $ \lceil \delta n  \rceil + 1 \leq s \leq n - 1$, let $H(s) = \sum_{t = s}^{n - 1} h_t$, and let 
\begin{align*}
	C(s) = &  \sum\limits_{{n \geq t_m > \cdots > t_1 = s + 1}}  L_F^m \left( |\localV_{t_m}|  + \frac{|\localV_{t_m}| + 1}{t_m}  \right) \times  \prod\limits_{k = 1}^{m - 1}\left(|\localV_{t_k} \cap \localV_{t_{k + 1}}| + \frac{|\localV_{t_k}| + 1}{t_k} \right) \ind\left\{|\localV_{t_k} \cap \localV_{t_{k + 1}}|  > 0\right\}.
\end{align*}
Then we can provide an upper bound for $H(s)$ using $C(s)$ and $H(s + 1)$: 
$$H(s) \leq \left(1 + \frac{1}{\lceil \delta n \rceil  }C(s)\right)H(s + 1) + C(s).$$
By induction, and applying the fact that $\log(1 + x) \leq x$ for all $x > -1$, we have
\begin{align}
H_{\lceil \delta n \rceil + 1}	 \leq  \sum\limits_{s = \lceil \delta n \rceil + 1}^{n - 1} C(s) \prod\limits_{s = \lceil \delta n \rceil + 1}^{n - 1}\left(1 + \frac{C(s)}{\lceil \delta n \rceil }\right) \leq  \sum\limits_{s = \lceil \delta n \rceil + 1}^{n - 1} C(s) \exp\left( \frac{1}{\lceil \delta n \rceil }\sum\limits_{s = \lceil \delta n \rceil + 1}^{n - 1} C(s)\right).\label{eq-H}
\end{align}
Using equation \eqref{eq:59}, we can prove Lemma \ref{lemma:two-algs-are-close} if we can prove the following two lemmas. 
\begin{lemma}\label{lemma-bound-H}
Under the assumptions stated in Theorem \ref{thm:local-with-summary-stat-is-trivial}, for all $\delta > 0$, we have $ H_{\lceil \delta n \rceil + 1} = O_P(n)$.
\end{lemma}
\begin{lemma}\label{lemma-bound-the-sup}
Under the assumptions stated in Theorem \ref{thm:local-with-summary-stat-is-trivial}, we have
\begin{align*}
	\sup\limits_{\lceil \delta n \rceil + 1 \leq t \leq n} |\bar{b}^t - \E[\bar{b}^t | \mathcal{F}_a]| = o_P(1). 
\end{align*}
\end{lemma}
\noindent
The proof of Lemma \ref{lemma-bound-H} and \ref{lemma-bound-the-sup} will be deferred to later parts of the appendix. Using Lemma \ref{lemma-bound-H} and \ref{lemma-bound-the-sup}, then apply equation \eqref{eq:59}, we have
\begin{align*}
	\frac{1}{n} \sum\limits_{i = 1}^n |w_i^t - b_i^t|  \leq \frac{1}{n }\langle \bd_n \vec{1} \rangle \leq  \frac{1}{n } H(\lceil \delta n \rceil)  \sup\limits_{\lceil \delta n \rceil + 1 \leq t \leq n - 1} |\bar{b}^t - \E[\bar{b}^t | \mathcal{F}_a]| = o_P(1),	
\end{align*}
thus finishes the proof of Lemma \ref{lemma:two-algs-are-close}.

\subsection{Proof of Lemma \ref{lemma-bound-H}}

To prove Lemma \ref{lemma-bound-H} we shall first provide a uniform upper bound for the expectation of $C(s)$ which is independent of $n, \delta$ and $s$. If this holds, then by Markov's inequality $ \sum_{s = \lceil \delta n \rceil  + 1}^n C(s) = O_P(n)$. Plugging this into equation \eqref{eq-H} gives  $H_{\lceil \delta n \rceil + 1} = O_P(n)$, which finishes the proof of this lemma. Let $\delta_1 := \frac{1}{2}\left(\frac{(a + (k - 1)b)\delta}{k} \wedge 1\right)\delta$.
 For $n$ large enough, we have $\frac{|\localV_{t_k}| + 1}{t_k} < 1$ for all $\lceil \delta n \rceil + 1 \leq t_k \leq n$, then we have
\begin{align}
C(s) 
\leq & \sum\limits_{{n \geq t_m > \cdots > t_1 = s + 1}}2^m L_F^mC_{act}^m \ind\left\{{|\localV_{t_k} \cap \localV_{t_{k + 1}}|> 0 } \ \ k = 1,2,\cdots, m - 1 \right\} \nonumber \\
\leq & \sum\limits_{{n \geq t_m > \cdots > t_1 = s + 1}}2^m L_F^mC_{act}^m \ind\left\{ d(v(t_k), v(t_{k + 1})) \leq 2R, \ \ k = 1,2,\cdots, m - 1\right\}.   \label{eq:79}
\end{align}
Note that in equation \eqref{eq:79}, $m$ can be any positive integer, therefore, taking the expectation of $C(s)$ gives
\begin{align}
	\E[C(s)] \leq & \sum\limits_{m = 1}^{\infty} \sum\limits_{{n \geq t_m > \cdots > t_1 = s + 1}}2^m L_F^mC_{act}^m \P(d(v(t_k), v(t_{k + 1})) \leq 2R, \ \ k = 1,2,\cdots, m - 1) \nonumber \\
	\leq & \sum\limits_{m = 1}^{\infty}\sum\limits_{{n \geq t_m > \cdots > t_1 = s + 1}}2^m L_F^mC_{act}^m\E\left[\prod\limits_{k = 1}^{m - 1}\frac{|{V}_{2R}^n(v(t_k))|} {n - k} \right] \nonumber \\
	\leq & \sum\limits_{m = 1}^{\infty}\frac{2^mL^m C_{act}^m}{(m - 1)!}\E[|{V}_{2R}^n(v)|^m], \nonumber
\end{align}
where $v$ is an arbitrary vertex in $V(n)$. The following lemma provides an upper bound on $\E[|{V}_{2R}^n(v)|^m]$.

\begin{lemma}\label{lemma-control-of-moments}
Under the assumptions of Theorem \ref{thm:local-with-summary-stat-is-trivial}, we have
\begin{align*}
\limsup\limits_{n \rightarrow \infty}\sum\limits_{m = 1}^{\infty}\frac{2^mL_F^m C_{act}^m}{(m - 1)!}\E[|{V}_{2R}^n(v)|^m] < \infty.
\end{align*}
\end{lemma}
\noindent
The proof of Lemma \ref{lemma-control-of-moments} is deferred to next subsection. With Lemma \ref{lemma-control-of-moments}, we can apply Markov's inequality to prove $\sum_{s = \lceil \delta n \rceil  + 1}^n C(s) = O_P(n)$, then Lemma \ref{lemma-bound-H} follows from equation \eqref{eq-H}.

\subsection{Proof of Lemma \ref{lemma-control-of-moments}}
To prove lemma \ref{lemma-control-of-moments}, we introduce the following branching process:
\begin{enumerate}
	\item $X_0 = 1$.
	\item Let $\{ Z(i, t):  i, t \in \mathbb{Z}^+ \}$ be an array of i.i.d. $\mathbb{Z}^+$-valued random variables with distribution Binomial$(n, \frac{a \vee b}{n})$. For $t \geq 1$, define $X_t = \sum_{1 \leq i \leq X_{t - 1}}Z(i, t)$. 
\end{enumerate}
Then we have
\begin{align}\label{eq-b3-1}
	\E\left[|{V}_{2R}^n(v)|^m\right] \leq \E\left[\left(\sum\limits_{k = 0}^{2R}X_k \right)^m\right].
\end{align}
Let $f_j(t) := \E[\exp(tX_j)]$, $\mu_m := a \vee b$. By the proof of Theorem 2.3.1 in \cite{vershynin2018high}, we have $f_1(t) \leq \exp \left(\mu_m(e^{t} - 1)\right)$. 
For $\gamma > 0$, let $t_1^{\gamma} := \log \left( \frac{\gamma}{\mu_m} + 1\right)$, and $t_{j + 1}^{\gamma} =: \log \left( \frac{t_j^{\gamma}}{\mu_m} + 1\right)$ for $j \in \mathbb{N}^+$. Then by Proposition 5.2 in \cite{lyons2017probability}, we have
\begin{align*}
 f_{j + 1}(t_{j + 1}^{\gamma}) \leq f_j(t_j^{\gamma}) \leq \cdots \leq f_1(t_1^{\gamma})\leq \exp(\gamma).
\end{align*}
Then for all $1 \leq j \leq 2R$ and $\gamma_0 > 0$, we have 
\begin{align}\label{eq-b3-2}
	& \E[|X_j^{m}|] = \int_0^{\infty} m \gamma^{m - 1}\P(X_j \geq \gamma)d\gamma \nonumber\\
\leq &   \int_0^{\infty} m \gamma^{m - 1} \exp(-\gamma t_j^{\gamma } + \gamma) d\gamma \nonumber \\
= & \int_0^{\gamma_0} m \gamma^{m - 1} \exp(-\gamma t_j^{\gamma } + \gamma) d\gamma + \int_{\gamma_0}^{\infty} m \gamma^{m - 1} \exp(-\gamma t_j^{\gamma } + \gamma) d\gamma \nonumber \\
\leq &  \gamma_0^{m }\exp(\gamma_0) + \frac{m !}{(t_j^{\gamma_0} - 1)^{m}}.
\end{align}
By equation \eqref{eq-b3-2}, for any $\gamma_0 > 0$, we have
\begin{align}\label{eq-b3-3}
	 \E\left[\left(\sum\limits_{k = 0}^{2R}X_k \right)^m\right] 
	\leq  (2R + 1)^m \gamma_0^{m}\exp(\gamma_0) + (2R + 1)^{m - 1}\sum\limits_{j = 1}^{2R}\frac{m!}{(t_j^{\gamma_0} - 1)^{m}} + (2 R + 1)^{m - 1}. 
\end{align}
Notice that we can choose $\gamma_0$ large enough, such that $\min\limits_{1 \leq j \leq 2R}t_j^{\gamma_0} - 1 \geq 4L_FC_{act}(2R + 1)$.
Using equations \eqref{eq-b3-1} and \eqref{eq-b3-3}, we have
\begin{align*}
	& \sum\limits_{m = 1}^{\infty}\frac{2^mL_F^m C_{act}^m}{(m - 1)!}\E[|{V}_{2R}^n(v)|^m] \\
	\leq & \underbrace{\sum\limits_{m = 1}^{\infty}\frac{2^mL_F^m C_{act}^m}{(m - 1)!}(2R + 1)^m (\gamma_0^{m}\exp(\gamma_0) + 1)}_{\RN{1}} + \underbrace{\sum\limits_{m = 1}^{\infty}\frac{2^mL_F^m C_{act}^m(2R + 1)^mm}{(\min\limits_{1 \leq j \leq 2R}t_j^{\gamma_0} - 1)^{m }}}_{\RN{2}}.
\end{align*}
The choice of $\gamma_0$ gives us equation $\RN{2} \leq \sum_{m = 1}^{\infty} \frac{m}{2^m} < \infty$.
One can also derive equation $\RN{1} < \infty$ for any value of $\gamma_0$. Note that the derived upper bounds for equations $\RN{1}$ and $\RN{2}$ are independent of $n$. Thus we have finished proving Lemma \ref{lemma-control-of-moments}.

\subsection{Proof of Lemma \ref{lemma-bound-the-sup}}

We first show a weaker result. As $n \rightarrow \infty$, we want to show $\sup_{\lceil \delta n \rceil + 1 \leq t \leq n} \E[|\bar{b}^t- \E[\bar{b}^t|\mathcal{F}_a] |^2] \rightarrow 0$.
Since the $F_w$ functions are uniformly bounded by $L_F$, 
we only need to show as $n \rightarrow \infty$,
\begin{align}\label{eq-b4-4}
	\sup\limits_{ \lceil \delta n \rceil + 1 \leq t \leq n}\sup\limits_{1 \leq x < y \leq t} \left| \E\left[(b_{v(x)}^t - \E[b_{v(x)}^t | \mathcal{F}_a]) (b_{v(y)}^t - \E[b_{v(y)}^t | \mathcal{F}_a]) \right] \right| \rightarrow 0.
\end{align}
For $1 \leq x < y \leq n$, $r \in \NN^+$, we introduce the following sets in the probability space we consider:
\begin{align*}
A_x^{r} = \left\{ \cG_{v(x)}^n \mbox{ is a subgraph of } \Ball_r^n(v(x))\right\}, \qquad B_{x, y}^{r} = \left\{ {V}_{r + 1}^n(v(x)) \cap {V}_{r + 1}^n(v(y)) = \emptyset  \right\}.
\end{align*}
Then we have 
\begin{align}\label{eq-b4-3}
& |\E[(b_{v(x)}^t - \E[b_{v(x)}^t | \mathcal{F}_a])(b_{v(y)}^t - \E[b_{v(y)}^t | \mathcal{F}_a] )]	|  \\
 \leq & 4L_F^2\left(\P((B_{x, y}^{r})^c)+  \P((A_x^r)^c) + \P((A_y^r)^c)\right) + \underbrace{|\E[(b_{v(x)}^t - \E[b_{v(x)}^t | \mathcal{F}_a])(b_{v(y)}^t - \E[b_{v(y)}^t | \mathcal{F}_a] ) \ind_{A_x^r \cap A_y^r \cap \ind_{B_{x, y}^{r}}}]|}_{C(x,y,t,n,r)}. \nonumber
\end{align}
Then we provide an upper bound for $C(x,y,t,n,r)$. Recall that $\fS_n$ is the group of permutation over $[n]$. Let $\bG'$ be a random graph sequence over the same set of vertices $\{1,2,\cdots, n\}$, $\tau' \in [k]^n$, $t_{\ast}' \in \fS_n$, such that the marginal distribution of $(\bG', \tau', t_{\ast}')$ is the same as $(\bG,\tau, t_{\ast})$. Furthermore, we assume $G'(t) = G(t)$ for $1 \leq t \leq \lceil \delta n \rceil$. For $v \in V(n)$ with $t_{\ast}(v) \leq \lceil \delta n \rceil$, we assume $\tau(v) = \tau'(v)$. For all $h \in [k]$, we assume
\begin{align*}
	|\{v \in V(n): \tau(v) = h\}| = |\{v \in V(n): \tau'(v) = h\}|.
\end{align*}
Conditioning on $\mathcal{F}_a$, we assume $(\bG, \tau, t_{\ast})$ is conditionally independent of $(\bG', \tau',t_{\ast}')$, and has the same conditional distribution.

For $t \in [n]$, $v \in V(t)$, and $r \in \NN^+$, let $\Ball_r^t(v; \bG) = (V_r^t(v; \bG), E_r^t(v; \bG))$ denote the ball of radius $r$ in $G(t)$, $\Ball_r^t(v; \bG') = (V_r^t(v; \bG'), E_r^t(v; \bG'))$ denote the ball of radius $r$ in $G'(t)$. Notice that $\Ball_r^t(v; \bG) = \Ball_r^t(v)$. Let $v'(y) \in V(n)$ be the unique vertex such that $t_{\ast}'(v'(y)) = y$. Furthermore, let
\begin{align*}
	& O_r^t(v; \bG) = \left( \Ball_r^t (v; \bG), \{(u,t_{\ast}(u)): u \in V_r^t (v; \bG) \} \right), \\
	& O_r^t(v; \bG') = \left( \Ball_r^t (v; \bG'), \{(u,t_{\ast}'(u)): u \in V_r^t (v; \bG') \} \right). 
\end{align*} 
Conditioning on $\mathcal{F}_a$, $O_{r + 1}^n({v}'(y); \bG')$ is conditionally independent of $O_{r + 1}^n(v(x); \bG)$. Define the following set:
\begin{align*}
 	\mathsf{NonOverlap}(x,y) = \left\{t_{\ast}(V_{r + 1}^n(v(x))) \cap t_{\ast}'(V_{r + 1}^n(v'(y))) = \emptyset \right\}
\end{align*}
Given $O_{r + 1}^n(v(x))$, we are able to judge whether $\cG_{v(x)}^n$ is a subgraph of $\Ball_{r}^n(v(x))$. If $\cG_{v(x)}^n$ is not a subgraph of $\Ball_r^n(v(x))$, then formula inside the expectation of $C(x,y,t,n,r)$ is zero. If $\cG_{v(x)}^n$ is a subgraph of $\Ball_r^n(v(x))$, then
\begin{align*}
	&  \mathcal{L}\left({O}_{r + 1}^n(v(y); \bG) \mid {O}_{r + 1}^n(v(x); \bG), {V}^n_{r + 1}(v(y); \bG) \cap {V}^n_{r + 1}(v(x); \bG) = \emptyset, \mathcal{F}_a\right) \\
  = &  \mathcal{L}({O}_{r + 1}^n({v}'(y); \bG') \mid {O}_{r + 1}^n(v(x); \bG), {{V}}_{r + 1}^n({v}'(y); \bG') \cap {{V}}^n_{r + 1}(v(x); \bG) = \emptyset, \mathsf{NonOverlap}(x,y), \mathcal{F}_a), \\
  &  \mathcal{L}(O_{r + 1}^n(v(y); \bG) \mid \mathcal{F}_a) = \mathcal{L}(O_{r + 1}^n({v}'(y); \bG') \mid \mathcal{F}_a),
\end{align*}
where $\mathcal{L}(\cdot \mid \cdot)$ is the conditional probability distribution. Let $\cG_v^t(\bG')$ be the subgraph defined in Section \ref{sec:local-streaming-alg} for graph sequence $\bG'$. Then we define the following sets:
\begin{align*}
& \tilde{A}_y^{r} = \left\{ \cG_{{v}'(y)}^n(\bG') \mbox{ is a subgraph of } {\Ball}_r^n({v}'(y); \bG') \right\}, \\
& \tilde{C}_{x, y}^{r} = \left\{  {{V}}_{r + 1}^n({v}'(y); \bG') \cap {{V}}_{r + 1}^n(v(x); \bG) = \emptyset\right\} \cap \mathsf{NonOverlap}(x,y) .
\end{align*}
On $\bG'$ we can still run Algorithm \ref{alg:aux-local-alg-with-summary-stat}, and denote the obtained quantity at time $t$ for vertex ${v}'(y)$ by $\tilde{b}_{{v}'(y)}^t$. Then we have
\begin{align}\label{eq-b4-1}
	& |\E[(b_{v(x)}^t - \E[b_{v(x)}^t | \mathcal{F}_a])\ind_{A_x^r}\ind_{B_{x, y}^{r}}(b_{v(y)}^t - \E[b_{v(y)}^t | \mathcal{F}_a] )\ind_{A_y^r} \mid \mathcal{F}_a]| \nonumber\\
	=& |\E[(b_{v(x)}^t - \E[b_{v(x)}^t | \mathcal{F}_a])\ind_{A_x^r}\ind_{B_{x, y}^{r}}\E[(b_{v(y)}^t - \E[b_{v(y)}^t | \mathcal{F}_a]) \ind_{A_y^r} | O_{r + 1}^n(v(x); \bG), B_{x, y}^r, \mathcal{F}_a]  \mid \mathcal{F}_a]| \nonumber\\
	=& |\E[(b_{v(x)}^t - \E[b_{v(x)}^t | \mathcal{F}_a])\ind_{A_x^r}\ind_{B_{x, y}^{r}} \E[(\tilde{b}_{v'(y)}^t - \E[\tilde{b}_{v'(y)}^t | \mathcal{F}_a])\ind_{\tilde{A}_y^{r}} | O_{r + 1}^n(v(x); \bG), \tilde{C}_{x, y}^r, \mathcal{F}_a] \mid \mathcal{F}_a]| \nonumber\\
	=& \left|\E\left[(b_{v(x)}^t - \E[b_{v(x)}^t | \mathcal{F}_a])\ind_{A_x^r}\ind_{B_{x, y}^{r}}\left. \frac{ \E[(\tilde{b}_{ v'(y)}^t - \E[\tilde{b}_{v'(y)}^t | \mathcal{F}_a])\ind_{\tilde{A}_y^{r}} \ind_{\tilde{C}_{x, y}^r} | O_{r + 1}^n(v(x); \bG),  \mathcal{F}_a]}{\P(\tilde{C}_{x, y}^r \mid O_{r + 1}^n(v(x); \bG),  \mathcal{F}_a )} \right| \mathcal{F}_a\right]\right| 
\end{align}
Obviously $\P(\tilde{C}_{x, y}^r | O_{r + 1}^n(v(x); \bG), \mathcal{F}_a) \overset{P} \rightarrow 1$ as $n \rightarrow \infty$, also notice that the functions $\|F_w\|_{\infty}$ has a uniform upper bound, then equation \eqref{eq-b4-1} is no larger than 
\begin{align}
	 & \left|\underbrace{\E\left[\left. (b_{v(x)}^t - \E[b_{v(x)}^t | \mathcal{F}_a])\ind_{A_x^r}\ind_{B_{x, y}^{r}} (\tilde{b}_{ v'(y)}^t - \E[\tilde{b}_{v'(y)}^t | \mathcal{F}_a])\ind_{\tilde{A}_y^{r}} \ind_{\tilde{C}_{x, y}^r} \right| \mathcal{F}_a\right] }_{\Delta}\right| + \nonumber \\
	& 4L_F^2\E\left[ \left| 1 - \P(\tilde{C}_{x, y}^r | O_{r + 1}^n(v(x); \bG),  \mathcal{F}_a )^{-1} \right| \wedge 1 \mid \mathcal{F}_a \right].     \label{eq:89}
\end{align}
Notice that
\begin{align*}
	 \left|\Delta - \E\left[(b_{v(x)}^t - \E[b_{v(x)}^t | \mathcal{F}_a])\ind_{A_x^r}\ind_{B_{x, y}^{r}}\left. (\tilde{b}_{v'(y)}^t - \E[\tilde{b}_{v'(y)}^t | \mathcal{F}_a]) \right| \mathcal{F}_a\right]\right|	\leq   4L_F^2 \P((\tilde{A}_y^{r})^c | \mathcal{F}_a) + 4L_F^2 \P((\tilde{C}_{x, y}^r)^c | \mathcal{F}_a)
\end{align*}
Conditioning on $\mathcal{F}_a$, $(\tilde{b}_{ v'(y)}^t - \E[\tilde{b}_{v'(y)}^t | \mathcal{F}_a])$ is conditionally independent of $(b_{v(x)}^t - \E[b_{v(x)}^t | \mathcal{F}_a])\ind_{A_x^r}\ind_{B_{x, y}^{r}}$. Therefore,
\begin{align*}
	\E\left[(b_{v(x)}^t - \E[b_{v(x)}^t | \mathcal{F}_a])\ind_{A_x^r}\ind_{B_{x, y}^{r}}\left. (\tilde{b}_{v'(y)}^t - \E[\tilde{b}_{v'(y)}^t | \mathcal{F}_a]) \right| \mathcal{F}_a\right] = 0.
\end{align*}
Combining the above analysis with equations \eqref{eq-b4-3} and \eqref{eq:89}, we have
\begin{align*}
	& |\E[(b_x^t - \E[\bar{b}^t | \mathcal{F}_a])(b_y^t - \E[\bar{b}^t | \mathcal{F}_a] )]| \\
	\leq & 4L_F^2\P((B_{x, y}^{r})^c) + 4L_F^2 \P((A_y^r)^c)  + 4L_F^2 \P((A_x^r)^c) + 4L_F^2 \P((\tilde{A}_y^{r})^c ) + 4L_F^2 \P((\tilde{C}_{x, y}^r)^c ) + \\
	& 4L_F^2\E\left[ \left| 1 - \P(\tilde{C}_{x, y}^r | O_{r + 1}^n(v(x); \bG),  \mathcal{F}_a )^{-1} \right| \wedge 1 \right].
\end{align*}
According to the proof of Theorem \ref{thm:streaming-local-implies-local} we have
\begin{align*}
\lim\limits_{r \rightarrow \infty}\limsup\limits_{n \rightarrow \infty} \sup\limits_{x,y \in [n]} \left\{	4L_F^2 \P((A_x^r)^c) + 4L_F^2 \P((A_y^r)^c) + 4L_F^2 \P((\tilde{A}_y^{r})^c )\right\} = 0.
\end{align*}
Furthermore, we claim without proof that for any fixed $r$, 
\begin{align*}
	& \limsup\limits_{n \rightarrow \infty}\sup\limits_{x,y \in [n]} \left\{4L_F^2 	\P((B_{x, y}^{r})^c) + 4L_F^2 \P((\tilde{C}_{x, y}^r)^c ) \right\}= 0, \\
	& \limsup\limits_{n \rightarrow \infty} \sup\limits_{x,y \in [n]} 4L_F^2\E\left[ \left| 1 - \P(\tilde{C}_{x, y}^r | O_{r + 1}^n(v(x); \bG),  \mathcal{F}_a )^{-1} \right| \wedge 1 \right] = 0.
\end{align*}
For any $\epsilon > 0$, there exists $r(\epsilon) \in \NN^+$, such that
\begin{align*}
	\limsup\limits_{n \rightarrow \infty} \sup\limits_{x,y \in [n]} \left\{	4L_F^2 \P((A_x^{r(\epsilon)})^c) + 4L_F^2 \P((A_y^{r(\epsilon)})^c) + 4L_F^2 \P((\tilde{A}_y^{r(\epsilon)})^c )\right\} \leq \epsilon.
\end{align*}
Therefore, 
\begin{align*}
	\limsup\limits_{n \rightarrow \infty}\sup\limits_{ \lceil \delta n \rceil + 1 \leq t \leq n}\sup\limits_{1 \leq x < y \leq t} \left| \E\left[(b_{v(x)}^t - \E[b_{v(x)}^t | \mathcal{F}_a]) (b_{v(y)}^t - \E[b_{v(y)}^t | \mathcal{F}_a]) \right] \right|  \leq \epsilon.
\end{align*}
Since $\epsilon$ is arbitrary, then we conclude that
\begin{align*}
	\limsup\limits_{n \rightarrow \infty} \sup\limits_{\lceil n \delta \rceil \leq t \leq n} \E\left[\left|\bar{b}^t- \E[\bar{b}^t|\mathcal{F}_a]  \right|^2  \right] = 0.
\end{align*}
For any $M \in \mathbb{N}^+$, we have
\begin{align*}
\sup\limits_{\lceil \delta n \rceil + 1 \leq t \leq n }|\bar{b}^t - \E[\bar{b}^t | \mathcal{F}_a]| \leq \sup_{\lceil M \delta \rceil \leq m \leq M} |\bar{b}^{\lceil nm / M \rceil} - \E[\bar{b}^{\lceil nm / M \rceil} | \mathcal{F}_a]| + \frac{2L_F(C_{act} + 1)}{M \delta}.
\end{align*}
For any fixed $M$, 
$$\sup\limits_{\lceil M \delta \rceil \leq m \leq M} |\bar{b}^{\lceil nm / M \rceil} - \E[\bar{b}^{\lceil nm / M \rceil} | \mathcal{F}_a]| = o_P(1).$$ 
Since $M$ is arbitrary, we conclude that $\sup_{\lceil \delta n \rceil + 1 \leq t \leq n} |\bar{b}^t - \E[\bar{b}^t | \mathcal{F}_a]| = o_P(1)$, this finishes the proof of Lemma \ref{lemma-bound-the-sup}.
Furthermore, from the proof of Lemma \ref{lemma-bound-the-sup} we can deduce the following corollary:

\begin{corollary}\label{cor:concentrate}
	Assume $ (\tau, \bG) \sim  \STSSBM(n,k,a,b)$. Let  $\cA$ be any  algorithm such that $\cA(i; G(n)) \in [k]$ is a function of $\cG_i^n$ and $\mathcal{F}_a$. Then we have as $n \rightarrow \infty$, 
	\begin{align*}
		\sup\limits_{\cA}\Var\left.\left[ \frac{1}{n} \sum\limits_{i = 1}^n \ind\{\cA(i; G(n)) = \tau(i)\} \right| \mathcal{F}_a \right] \overset{P}{\rightarrow} 0. 
	\end{align*}
\end{corollary}

\begin{proof}
	To prove this corollary, we first show for all $x,y \in [n]$,
	\begin{align}
		& C(x,y) := \sup\limits_{\cA} \left| \P\left( \cA(v(x); G(n)) = \tau(v(x)),  \cA(v(y); G(n)) = \tau(v(y))\mid \mathcal{F}_a \right) - \right. \nonumber \\
		& \left.\P\left(\cA(v(x); G(n)) = \tau(v(x))\mid \mathcal{F}_a \right)\P\left(\cA(v(y); G(n)) = \tau(v(y))\mid \mathcal{F}_a \right)\right| \overset{P}{\rightarrow} 0. \label{eq:30}
	\end{align}
We use the notations defined in the proof of Lemma \ref{lemma-bound-the-sup}. Furthermore, we define the following quantities:
	\begin{align*}
		& \Delta_1 := \E\left[ \ind\{\cA(v(x); G(n)) = \tau(v(x))\} \ind_{A_x^r} \ind_{B_{x,y}^r} \ind\{\cA(v(y); G(n)) = \tau(v(y))\} \ind_{A_y^r} \mid \mathcal{F}_a  \right], \\
		& \Delta_2 := \E\left[ \ind\{\cA(v(x); G(n)) = \tau(v(x))\} \ind_{A_x^r} \ind_{B_{x,y}^r} \ind\{\cA({v}'(y); {G}'(n)) = \tau({v}'(y))\} \ind_{\tilde{A}_y^r}\ind_{\tilde{C}_{x,y}^r} \mid \mathcal{F}_a  \right].
	\end{align*}
	Similar to the proof of Lemma \ref{lemma-bound-the-sup}, we conclude that the following formulas hold:
	\begin{align*}
		& \left|\P\left( \cA(v(x); G(n)) = \tau(v(x)),  \cA(v(y); G(n)) = \tau(v(y))\mid \mathcal{F}_a \right) - \Delta_1 \right| \\
		 \leq & \P((A_x^r)^c \mid \mathcal{F}_a) + \P((A_y^r)^c \mid \mathcal{F}_a) + \P((B_{x,y}^r)^c \mid \mathcal{F}_a),  \\
		 \Delta_1 =& \E\left[ \ind\{\mathcal{A}(v(x); G(n)) = \tau(v(x))\}  \ind_{A_x^r} \ind_{B_{x,y}^r} \E\left[ \ind\{\cA({v}'(y); {G}'(n)) = \tau({v}'(y)) \} \ind_{\tilde{A_y^r}} \mid O_{r + 1}^n(v(x); \bG), \tilde{C}_{x,y}^r, \mathcal{F}_a \right]\mid \mathcal{F}_a \right] \\
		=& \E\left.\left[ \ind\{\mathcal{A}(v(x); G(n)) = \tau(v(x))\}  \ind_{A_x^r} \ind_{B_{x,y}^r} \frac{\E\left[ \ind\{\cA(\tilde{v}(y); \tilde{G}(n)) = \tau({v}'(y)) \} \ind_{\tilde{A_y^r}} \ind_{\tilde{C}_{x,y}^r} \mid O_{r + 1}^n(v(x); \bG),  \mathcal{F}_a \right]}{\P\left(\tilde{C}_{x,y}^r \mid O_{r + 1}^n(v(x); \bG),  \mathcal{F}_a \right)} \right| \mathcal{F}_a \right], \\
		 &\left| \Delta_1 - \Delta_2 \right| \leq \E\left[ \left|\P\left(\tilde{C}_{x,y}^r \mid O_{r + 1}^n(v(x); \bG),  \mathcal{F}_a \right)^{-1} - 1\right| \wedge 1 \mid \mathcal{F}_a\right], \\
		& \left| \Delta_2 - \P\left(\cA(v(x); G(n)) = \tau(v(x))\mid \mathcal{F}_a \right)\P\left(\cA(v(y); G(n)) = \tau(v(y))\mid \mathcal{F}_a \right)  \right| \\
		 \leq & \P((A_x^r)^c \mid \mathcal{F}_a) + \P((B_{x,y}^r)^c \mid \mathcal{F}_a ) + \P((\tilde{A}_y^r)^c \mid \mathcal{F}_a) + \P((\tilde{C}_{x,y}^r)^c \mid \mathcal{F}_a).
	\end{align*}
	Combining the above equations we have
	\begin{align}
		C(x,y) & \leq  2\P((A_x^r)^c \mid \mathcal{F}_a) + 2\P((B_{x,y}^r)^c \mid \mathcal{F}_a ) + 2\P((\tilde{A}_y^r)^c \mid \mathcal{F}_a) + \P((\tilde{C}_{x,y}^r)^c \mid \mathcal{F}_a) \nonumber \\
		& + \E\left[ \left|\P\left(\tilde{C}_{x,y}^r \mid O_{r + 1}^n(v(x); \bG),  \mathcal{F}_a \right)^{-1} - 1\right| \wedge 1 \mid \mathcal{F}_a\right] . \label{eq:31}
	\end{align}
	The upper bound on the right hand side of equation \eqref{eq:31} is independent of $\cA$, and converges in probability to zero by Theorem \ref{thm:streaming-local-implies-local}.  Thus we have finished proving \eqref{eq:30}. Furthermore, similar to the proof of Lemma \ref{lemma-bound-the-sup}, from \eqref{eq:31} we can conclude that as $n \rightarrow \infty$,
\begin{align*}
	\sup\limits_{x,y \in [n]}\E[C(x,y)] \rightarrow 0,  
\end{align*}
thus we finishes the proof of this corollary by Markov's inequality. 

\end{proof}

\subsection{Proof of Lemma \ref{lemma:aux-alg-is-trivial}}\label{sec:proof-of-lemma3}
For any $\pi \in \fS_k$ and $\delta > 0$, using Corollary \ref{cor:concentrate}, we have
\begin{align*}
	\frac{1}{n} \sum\limits_{i = 1}^n \ind\{ \htau(b_i^n + \epsilon U_i) = \pi \circ \tau(i) \} = \frac{1}{n} \sum\limits_{i = 1}^n \P\left( \htau(b_i^n + \epsilon U_i) = \pi \circ \tau(i)  \mid \mathcal{F}_a \right) + o_P(1). 
\end{align*}
If we can show the following equation for all $\pi \in \fS_k$, then we finishes the proof of Lemma \ref{lemma:aux-alg-is-trivial}.
\begin{align*}
	\limsup\limits_{\delta \rightarrow 0^+}\limsup\limits_{n \rightarrow \infty}\E\left[ \left| \frac{1}{n} \sum\limits_{i = 1}^n \P\left( \htau(b_i^n + \epsilon U_i) = \pi \circ \tau(i)  \mid \mathcal{F}_a \right) - \frac{1}{k} \right|\right] = 0.
\end{align*}
Since
\begin{align*}
	& \E\left[ \left| \frac{1}{n} \sum\limits_{i = 1}^n \P\left( \htau(b_i^n + \epsilon U_i) = \pi \circ \tau(i)  \mid \mathcal{F}_a \right) - \frac{1}{k} \right|\right]  \leq   \frac{1}{n}\sum\limits_{i= 1}^n\sum\limits_{h = 1}^k \E\left[ \left| \P\left( \tau(i) = h \mid \mathcal{F}_a, b_i^n \right) - \frac{1}{k} \right| \right],
\end{align*}
then we only need to show for all $ 1 \leq i \leq n$ and all $h \in [k]$, 
\begin{align}\label{eq:28}
	\limsup\limits_{\delta \rightarrow 0^+}\limsup\limits_{n \rightarrow \infty}\E\left[ \left| \P\left( \tau(i) = h \mid \mathcal{F}_a, b_i^n \right) - \frac{1}{k} \right| \right] = 0. 
\end{align}
%
By Theorem \ref{thm:streaming-local-implies-local}, for any $\epsilon > 0$, there exists $r_{\epsilon} \in \mathbb{N}^+$, such that $\limsup\limits_{n \rightarrow \infty} \P(\cG_{v}^n \mbox{ is not a subgraph of } \Ball_{r_{\epsilon}}^n(v)) \leq \epsilon$ for any $v \in [n]$. Then by conditional Jensen's inequality, we have
\begin{align}
	& \E\left[\left|\P\left( \tau(i) = h \mid b_i^n, \mathcal{F}_a \right) - \frac{1}{k} \right|  \right] \nonumber \\
	 \leq &   \P\left( \cG_{i}^n \mbox{ is not a subgraph of } \Ball_{r_{\epsilon}}^n(i)  \right) + \E\left[\left|\P\left( \tau(i) = h \mid \mathcal{F}_a, \Ball_{r_{\epsilon}}^n(i) \right) - \frac{1}{k} \right|\ind \left\{V(\lceil \delta n \rceil) \cap V_{r_{\epsilon}}^n(i) = \emptyset\right\} \right] \nonumber \\
	& + \P(V(\lceil \delta n \rceil) \cap V_{r_{\epsilon}}^n(i) \neq \emptyset)  \label{eq:106}
\end{align}
On the event $V(\lceil \delta n \rceil) \cap V_{r_{\epsilon}}^n(i) = \emptyset$, for $h_1, h_2 \in [k]$ with $h_1 \neq h_2$, we have
\begin{align}
	& \frac{\P\left(\tau(i) = h_1 | \mathcal{F}_a, \Ball_{r_{\epsilon}}^n(i) \right)}{\P(\tau(i) = h_2 | \mathcal{F}_a, \Ball_{r_{\epsilon}}^n(i) )}	= \frac{\P\left( \tau(i) = h_1 \mid  \Ball_{r_{\epsilon}}^n(i), \{n_h: h \in [k]\}, \tau(V(\lceil \delta n \rceil)) \right)}{\P\left(\tau(i) = h_2 \mid  \Ball_{r_{\epsilon}}^n(i), \{n_h: h \in [k]\}, \tau(V(\lceil \delta n \rceil)) \right)} \nonumber 
\end{align}
Therefore, the right hand side of equation \eqref{eq:106} is equal to 
\begin{align*}
	&   \E\left[\left|\P\left( \tau(i) = h \mid \Ball_{r_{\epsilon}}^n(i), \{n_h: h \in [k]\}, \tau(V(\lceil \delta n \rceil)) \right) - \frac{1}{k} \right|\ind \left\{V(\lceil \delta n \rceil) \cap V_{r_{\epsilon}}^n(i) = \emptyset\right\}\right] \\
	& + \P\left( \cG_{i}^n \mbox{ is not a subgraph of } \Ball_{r_{\epsilon}}^n(i)  \right)+ \P(V(\lceil \delta n \rceil) \cap V_{r_{\epsilon}}^n(i) \neq \emptyset ).
\end{align*}
Adapting from the proof of Proposition 2 in \cite{mossel2015reconstruction}, we conclude that
\begin{align*}
	\lim\limits_{\delta \rightarrow 0^+}\limsup\limits_{n \rightarrow \infty}\E\left[\left|\P\left( \tau(i) = h \mid \Ball_{r_{\epsilon}}^n(i), \{n_h: h \in [k]\}, \tau(V(\lceil \delta n \rceil)) \right) - \frac{1}{k} \right| \right] = 0.
\end{align*}
Furthermore, as $\delta \rightarrow 0^+$, $\limsup\limits_{n \rightarrow \infty}\P(V(\lceil \delta n \rceil) \cap V_{r_{\epsilon}}^n(i) \neq \emptyset ) \rightarrow 0$, therefore, 
\begin{align*}
	\limsup\limits_{\delta \rightarrow 0^+}\limsup\limits_{n \rightarrow \infty}\E\left[ \left| \P\left( \tau(i) = h \mid \mathcal{F}_a, b_i^n \right) - \frac{1}{k} \right| \right] \leq \epsilon.
\end{align*}
Since $\epsilon$ is arbitrary, we conclude that equation \eqref{eq:28} holds, this finishes the proof of this lemma. 

\section{Analysis of streaming belief propagation with side information}
In this section we prove Theorem \ref{thm-streaming-local-bp-better-than-local-bp}. Let $F(x) = \frac{1}{2}\log\left(\frac{e^{2x}a + b}{e^{2x}b + a}\right)$, $h_1 = \frac{1}{2}\log \frac{1 - \alpha}{\alpha}$, and $h_2 = \frac{1}{2}\log \frac{\alpha}{1 - \alpha}$.
With $k = 2$, notice that Algorithm \ref{alg:streaming-bp-with-side-information} can be equivalently reduced to the following form, with which we will continue our proof:
\begin{breakablealgorithm}\label{alg-streaming-bp-with-side-information}
	\caption{$R$-local streaming belief propagation with $k = 2$}
	\begin{algorithmic}[1]
		\State Initialization: $V(0) = E(0) = G(0) = \emptyset$.
		\For{$t = 1,2,\cdots, n$}
			\State $V(t) \gets V(t - 1) \cup \{v(t)\}$ 
			\State $E(t) \gets E(t - 1) \cup \{(v(t), v): v \in V(t - 1), (v(t), t) \in E\}$
			\State $G(t) \gets (V(t), E(t))$
			\For{$v \in {D}_1^t(v(t))$}
				\State $M_{v \rightarrow v(t)} \gets h_{\tilde{\tau}(v)} + \sum\limits_{v' \in {D}_1^{t - 1}(v)} F(M_{v' \rightarrow v})$
			\EndFor
			\For{$r = 1,2,\cdots, R$}
				\For{$v \in {D}_r^t(v(t))$}
					\State Let $v' \in {D}_1^t(v)$ be a vertex which is on a shortest path connecting $v$ and $v(t)$.
					\State  $M_{v' \rightarrow v} \gets h_{\tilde{\tau}(v')} + \sum\limits_{v'' \in {D}_1^t(v') \backslash \{v\}}F(M_{v'' \rightarrow v'})$
				\EndFor
			\EndFor
		\EndFor
		\For{$u \in V(n)$}
			\State $M_u \gets h_{\tilde{\tau}(u)} + \sum\limits_{u' \in {D}_1^n(u)}F(M_{u' \rightarrow u})$
			\State Output $-\ind\{M_u \geq 0\} + 2$ as an estimate for $\tau(u)$.
		\EndFor
	\end{algorithmic}
\end{breakablealgorithm}
We start the proof by introducing the following definition and lemmas:

\begin{definition}[Output of belief propagation]\label{def:result-of-bp}
Let $T = (V(T), E(T))$ be a tree rooted at $u$. Let $L_T$ be the set of leaves of $T$: $L_T = \{v: v \mbox{ has degree }1 \mbox{ in }T\}$. For each $v \in L_T$, assume we are given $M_v^{input} \in \mathbb{R}$ which we refer to as the \emph{input belief} into $T$ at $v$. For each $v \in V(T)$, suppose we observe a noisy label $\tilde{\tau}(v) \in \{1,2\}$, and denote the set of children of $v$ in $T$ by $\mathcal{C}(v)$. Given the model parameters $a,b$ and $\alpha$, 
\begin{enumerate}
    \item For $v \in L_T$, set $\tilde{M}_{v \rightarrow \mbox{pa}(v)} =  M_v^{input}$, where $pa(v)$ is the parent vertex of $v$ in $T$.
    \item Denote the depth of tree $T$ by $R$. For $r = R - 1, R - 2, \cdots, 1$, sequentially conduct the following updates: for any $v \in {D}_r^n(v_0)$, let 
    $$\tilde{M}_{v\rightarrow \mbox{pa}(v)} = h_{\tilde{\tau}(v)} + \sum\limits_{v' \in \mathcal{C}(v)}F(\tilde{M}_{v'\rightarrow v}).$$
    \item We define the \emph{output of belief propagation} $\Bprop(u; T, \{M_v^{input}: v \in L_T\}, \tilde{\tau}, \alpha, a, b)$ on the tree $T$ as follows:
    \begin{align*}
    	\Bprop(u; T, \{M_v^{input}: v \in L_T\}, \tilde{\tau}, \alpha, a, b) = h_{\tilde{\tau}(u)} + \sum\limits_{u' \in \mathcal{C}(u)}F(\tilde{M}_{u'\rightarrow u}).
    \end{align*}
\end{enumerate}
\end{definition}

\begin{lemma}\label{lemma-streaming-bpws-is-standard-bp-on-a-random-tree}
Let $(M_u)_{u \in V(n)}$ be the output of Algorithm \ref{alg-streaming-bp-with-side-information}, under \STSSBM$(n, 2, a, b, \alpha)$, for any $\epsilon > 0$, there exists $r_{\epsilon}, n_{\epsilon} \in \mathbb{N}^+$, such that for any $u \in V(n)$, $n \geq n_{\epsilon}$, with probability at least $1 - \epsilon$, the following holds:
\begin{align*}
    M_u = \Bprop(u; T, \{M_v^{input} = h_{\tilde{\tau}(v)}: v \in L_T\}, \tilde{\tau}, \alpha, a, b)
\end{align*}
for some (random) tree $T$ rooted at $u$, with the depth of $T$ no larger than $r_{\epsilon}$. 
\end{lemma}
\begin{proof}
By Lemma \ref{lemma:locally-tree-like} and Theorem \ref{thm:streaming-local-implies-local}, for any $\epsilon > 0$, there exists $r_{\epsilon}, n_{\epsilon}\in \mathbb{N}^+$, such that with probability no less than $1 - \epsilon$, (1) $\cG_u^n$ is a subgraph of $\Ball_{r_{\epsilon}}^n(u)$, so $M_u$ is a function of $\bar{\cG}_u^n$ and (2) $\Ball_{r_{\epsilon}}^n(u)$ is a tree. The result then follows by observing the iterating formulas of Algorithm \ref{alg-streaming-bp-with-side-information}. Furthermore, if the event just described occurs, then the depth of $T$ is no larger than $r_{\epsilon}$.
\end{proof}

\begin{lemma}\label{lemma-form-of-BP}
Consider Algorithm \ref{alg-streaming-bp-with-side-information}, denote the value of $M_{v\rightarrow v'}$ before vertex $v(t + 1)$ arrives by $M_{v\rightarrow v'}^t$. 
Then the following equation holds for some random time indices $\left\{t(v) \in [n]: \ v \in {D}_R^n(u)\right\}$:
\begin{align*}
    M_u = \Bprop(u; \Ball_R^n(u), \{M_v^{input} = M_{v\rightarrow \mbox{pa}(v)}^{t(v)}: v \in L_{\Ball_R^n(u)}\}, \tilde{\tau}, \alpha, a, b).
\end{align*} 
\end{lemma}
\begin{proof}
We conduct this proof by induction. Actually, we will show a stronger result. Instead of focusing only on $R$-local streaming belief propagation, we will show a more general result for $N$-local streaming belief propagation with any $N \geq R$. Specifically, we will consider the following algorithm:
\begin{enumerate}
    \item At each time ${t}$, a new vertex $v(t)$ is revealed. For $v \in  D_1^t(v(t))$, define $M_{v \rightarrow v(t)}(N) = h_{\tilde{\tau}(v)} + \sum\limits_{v' \in D_1^t(v) \backslash \{v(t)\} } F(M_{v' \rightarrow v}(N))$.
    \item For $r = 1,2,\cdots, N$, sequentially conduct the following updates: for each $v \in {D}_r^t(v)$, let $v' \in D_1^t(v)$ be a vertex which is on one of the shortest path connecting $v(t)$ and $v$, update or initialize the value of $M_{v' \rightarrow v}(N) = h_{\tilde{\tau}(v')} + \sum\limits_{v'' \in D_1^t(v')\backslash \{v\}} F(M_{v'' \rightarrow v'}(N))$. Since the graph is with high probability locally tree-like, such $v'$ is with high probability unique.
    \item Repeat step 1 and step 2 for all $1 \leq t \leq n$.
\end{enumerate}
For $(v, v') \in E(t)$, let $M_{v\rightarrow v'}^t(N)$ denote the value of $M_{v\rightarrow v'}(N)$ after t-th iteration obtained by the $N$-local streaming belief propagation algorithm. Let $t_R(u)$ be the first time such that all vertices in $V_R^n(u)$ have been revealed:
\begin{align*}
    t_R(u) := \inf\{t \in [n]: V_R^n(u) = V_R^t(u)  \}.
\end{align*}
Let $M_u^t(N) = h_{\tilde{\tau}(u)} + \sum\limits_{u' \in D_1^t(u)}F(M_{u' \rightarrow u}^t(N))$. Now instead of proving Lemma \ref{lemma-form-of-BP}, we show that for any $N \geq R$, and any $t \geq t_R(u)$, if $\Ball_R^n(u)$ is a tree, then for some random time indices $\left\{t(v) \in [n]: \ v \in {D}_R^t(u)\right\}$, we have
\begin{align}\label{eq:124}
	M_u^t(N) = \Bprop(u; \Ball_R^n(u), \{M_v^{input} = {M}^{t(v)}_{v\rightarrow \mbox{pa}(v)}(N): v \in L_{\Ball_R^n(u)} \}, \ttau, \alpha, a, b).
\end{align}
Note that for all $t \geq t_R(u)$, $M_u^t(R) = M_u^n(R) = M_u$, therefore the result we have just described is indeed a stronger version of Lemma \ref{lemma-form-of-BP}. Then we will prove this stronger result by performing induction on $R$. Equation \eqref{eq:124} obviously holds for $R = 1$. Now suppose equation \eqref{eq:124} holds for $R = r$, we will show it holds for $R = r + 1$ by induction. If $\Ball_{r + 1}^n(u)$ is a tree rooted at $u$, for any $N \geq r + 1$, $u' \in D_1^n(u)$, let $\bar{T}_{u'}^r$ be the depth $r$ subtree consisting of $u'$ and its descendants in $\Ball_{r + 1}^n(u)$. Let
\begin{align*}
	\bar{t}_r(u') = \inf\{t: \bar{T}_{u'}^r \mbox{ is a subgraph of } \Ball_{r+1}^t(u)\}. 
\end{align*}
For $u_1, u_2 \in D_1^n(u)$, $u_1 \neq u_2$, we have $\bar{t}_r(u_1) \neq \bar{t}_r(u_2)$, and $t_{\ast}(u) \neq \bar{t}_r(u_1)$, where $t_{\ast}$ is defined in Section \ref{sec:summary-stat}. For $u' \in D_1^n(u)$, by the induction hypothesis, for $t \geq t_{\ast}(u) \vee \bar{t}_r(u')$, there exists random time indices $\{t(v) \in [n]: v \in L_{\bar{T}_{u'}^r}\}$, such that
\begin{align*}
	M_{u' \rightarrow u}^t(N) = \Bprop(u'; \bar{T}_{u'}^r, \{M_v^{input} = M_{v \rightarrow pa(v)}^{t(v)}(N): v \in L_{\bar{T}_{u'}^r}\}, \ttau, \alpha, a, b).
\end{align*}
Since $t_{r+1}(u) \geq \bar{t}_r(u') \vee t_{\ast}(u)$ for all $u' \in D_1^n(u)$, then for all $t \geq t_{r + 1}(u)$, there exists random time indices $\{t(v) \in [n]: \ v \in {D}_{r + 1}^t(u)\}$, such that
\begin{align*}
    M_u^{t}(N) =&  h_{\tilde{\tau}(u)} + \sum\limits_{u' \in D_1^t(u)}F(M_{u'\rightarrow u}^{t}(N))\\
     =& \Bprop(u; \Ball_{r + 1}^n(u), \{M_v^{input} = {M}^{t(v)}_{v\rightarrow \mbox{pa}(v)}(N): v \in L_{\Ball_{r + 1}^n(u)} \}, \ttau, \alpha, a, b),
\end{align*}
which finishes the proof of this lemma.  

\subsection{Proof of Theorem \ref{thm-streaming-local-bp-better-than-local-bp}}
For $u \in V$, let $(\lambda_{r_{\epsilon}}, T_u^{r_{\epsilon}}) \sim \P^T_{r_{\epsilon},(a + b) / 2,b / (a + b)}$ where $T_u^{r_{\epsilon}}$ is the tree and $\lambda_{r_{\epsilon}}$ is the set of labels. For vertices in $T_u^{r_{\epsilon}}$, we denote the set of noisy labels generated independently with incorrect probability $\alpha$ by $\tilde{\lambda}_{r_{\epsilon}}$. For any $\epsilon > 0$, by Theorem \ref{thm:streaming-local-implies-local} and Lemma \ref{lemma:locally-tree-like}, there exists $r_{\epsilon}, n_{\epsilon} \in \mathbb{N}^+$, such that for all $n \geq n_{\epsilon}$, with probability at least $1 - \epsilon$: (1) $\cG_u^n$ is a subgraph of $\Ball_{r_{\epsilon}}^n(u)$, (2) $(\lambda_{r_{\epsilon}},\tilde{\lambda}_{r_{\epsilon}}, T_u^{r_{\epsilon}}) = (\tau({V}_{r_{\epsilon}}^n(u)), \tilde{\tau}({V}_{r_{\epsilon}}^n(u)), \Ball_{r_{\epsilon}}^n(u))$. Further there exists a coupling of $(\lambda_{r_{\epsilon}},\tilde{\lambda}_{r_{\epsilon}}, T_u^{r_{\epsilon}})$ and $(\tau({V}_{r_{\epsilon}}^n(u)), \tilde{\tau}({V}_{r_{\epsilon}}^n(u)), \Ball_{r_{\epsilon}}^n(u))$ such that defining
\begin{align*}
    S_{\epsilon} := \left\{(\lambda_{r_{\epsilon}},\tilde{\lambda}_{r_{\epsilon}}, T_u^{r_{\epsilon}}) = (\tau({V}_{r_{\epsilon}}^n(u)), \tilde{\tau}({V}_{r_{\epsilon}}^n(u)), \Ball_{r_{\epsilon}}^n(u)), \cG_u^n \mbox{ is a subgraph of }\Ball_{r_{\epsilon}}^n(u) \right\},
\end{align*}
then $\P(S_{\epsilon}) \geq 1 - \epsilon$ for all $n \geq n_{\epsilon}$. In the following parts of the analysis, we always assume $S_{\epsilon}$ occurs. Let $\P_{bt}$ denote the probability distribution of $(\lambda_{r_{\epsilon}}, \tilde{\lambda}_{r_{\epsilon}}, T_u^{r_{\epsilon}})$. According to Lemma \ref{lemma-streaming-bpws-is-standard-bp-on-a-random-tree} and \ref{lemma-form-of-BP}, conditioning on $S_{\epsilon}$, there exists $T_u$ being a tree rooted at $u$, such that $\Ball_R^n(u)$ is a subgraph of $T_u$ and $T_u$ is a subgraph of $\Ball_{r_{\epsilon}}^n(u)$. Furthermore, $M_u$ can be expressed as:
\begin{align*}
	M_u = \Bprop(u; T_u, \{M_v^{input} = h_{\tilde{\tau}(v)}: v \in L_{T_u}\}, \ttau, \alpha, a, b) = \frac{1}{2} \log \frac{\P_{bt}\left(\tau(u) = 1\ | T_u, \tilde{\tau}(T_u) \right)}{\P_{bt}\left(\tau(u) = 2\ | T_u, \tilde{\tau}(T_u) \right)},
\end{align*}
where $L_{T_u}$ is the set of leaf vertices in $T_u$, $\tilde{\tau}(T_u)$ refers to the set of noisy labels of vertices in $T_u$. Then for all $n \geq n_{\epsilon}$, we have
\begin{align*}
    \est_ac(\mathcal{A}_R) 
    \geq & \P(\{\mathcal{A}_R(u; G(n), \ttau) = 1, \tau(u) = 1\} \cap S_{\epsilon}) + \P(\{\mathcal{A}_R(u; G(n), \ttau) = 2, \tau_u = 2\} \cap S_{\epsilon}) \\
    \geq & \frac{1}{2}\P_{bt}(M_u \geq 0 \mid \tau(u) = 1) + \frac{1}{2}\P_{bt}(M_u < 0 \mid \tau(u) = 2) - \epsilon \\
    = & \frac{1}{2} + \E_{bt}\left[\left|\P_{bt}(\tau(u) = 1 \mid T_u, \tilde{\tau}(T_u) ) - \frac{1}{2}\right|\right] - \epsilon.
\end{align*}
We further have
\begin{align*}
    \E_{bt}\left[\P_{bt}(\tau(u) = 1 | T_u, \tilde{\tau}(T_u)) \mid \Ball_R^n(u), \tilde{\tau}(V_R^n(u))\right] =\P_{bt}(\tau(u) = 1 | \Ball_R^n(u), \tilde{\tau}(V_R^n(u)) ).
\end{align*}
Since $x \mapsto |x - \frac{1}{2}|$ is convex, therefore, by Jensen's inequality, for all $n \geq n_{\epsilon}$,
\begin{align*}
    & \est_ac(\mathcal{A}_R)\geq \frac{1}{2} + \E_{bt}\left[\left|\P_{bt}(\tau(u) = 1\ | \Ball_R^n(u), \tilde{\tau}(V_R^n(u)) ) - \frac{1}{2}\right|\right] - \epsilon.
\end{align*}
Since $\epsilon$ is arbitrary, we have 
\begin{align}\label{tree-and-streaming-local-BP}
    \liminf\limits_{n\rightarrow \infty}\left(\est_ac(\mathcal{A}_R) - \frac{1}{2} - \E_{bt}\left[\left|\P_{bt}(\tau(u) = 1\ | \Ball_R^n(u), \tilde{\tau}({V}_R^n(u))) - \frac{1}{2}\right|\right]\right) \geq 0.
\end{align}
According to Lemma 3.7 in \cite{mossel2016local}, we have
\begin{align}\label{tree-and-local-BP}
    \lim\limits_{n\rightarrow\infty}\left|\est_ac(\cA_R^{\mbox{\tiny\rm off}}) - \frac{1}{2} - \E_{bt}\left[\left|\P_{bt}(\tau(u) = 1 | \Ball_R^n(u), \tilde{\tau}({V}_R^n(u))) - \frac{1}{2}\right|\right]\right| = 0.
\end{align}
Combining equations \eqref{tree-and-streaming-local-BP} and \eqref{tree-and-local-BP}, we have 
\begin{align*}
    \liminf\limits_{n\rightarrow \infty}\left(\est_ac(\mathcal{A}_R) - \est_ac(\cA_R^{\mbox{\tiny\rm off}})\right) \geq 0,
\end{align*}
Thus finishes the proof of Theorem \ref{thm-streaming-local-bp-better-than-local-bp}.

\subsection{Proof of Corollary \ref{thm-streaming-local-bp-is-optimal}}

According to Theorem 2.3 in \cite{mossel2016local}, under the three regimes listed in this theorem, we have 
\begin{align}\label{mossel16-optimal-result}
    \lim\limits_{R\rightarrow \infty} \limsup\limits_{n\rightarrow \infty}\left(\opt_est_ac
 - \est_ac(\cA_R^{\mbox{\tiny\rm off}})\right) = 0. \end{align}
Combining \eqref{mossel16-optimal-result} and Theorem \ref{thm-streaming-local-bp-better-than-local-bp}, we have
\begin{align*}
    \limsup\limits_{n\rightarrow \infty}\left(\opt_est_ac - \est_ac(\mathcal{A}_R)\right) &\leq \limsup\limits_{n\rightarrow \infty}\left(\opt_est_ac - \est_ac(\cA_R^{\mbox{\tiny\rm off}})\right) + \limsup\limits_{n\rightarrow \infty}\left(\est_ac(\cA_R^{\mbox{\tiny\rm off}}) - \est_ac(\mathcal{A}_R)\right)\\
    &\leq \limsup\limits_{n\rightarrow \infty}\left(\opt_est_ac - \est_ac(\cA_R^{\mbox{\tiny\rm off}})\right),
\end{align*}
thus 
\begin{align*}
	\limsup\limits_{R\rightarrow \infty} \limsup\limits_{n\rightarrow \infty}\left(\opt_est_ac - \est_ac(\mathcal{A}_R)\right) & \leq \lim\limits_{R\rightarrow \infty} \limsup\limits_{n\rightarrow \infty}\left(\opt_est_ac - \est_ac(\cA_R^{\mbox{\tiny\rm off}})\right) = 0.
\end{align*}
Since $\opt_est_ac \geq \est_ac(\mathcal{A}_R)$, the other direction naturally holds, thus finishes the proof of Corollary \ref{thm-streaming-local-bp-is-optimal}.
\end{proof}

\end{appendices}



\end{document}